\RequirePackage{amsthm}
\documentclass[sn-mathphys-num,iicol]{sn-jnl}

\usepackage{hhline}
\usepackage{multirow} \usepackage{graphicx}\usepackage{amsmath,amssymb,amsfonts}\usepackage{amsthm}\usepackage{mathrsfs}\usepackage[title]{appendix}\usepackage{xcolor, color}\usepackage{textcomp}\usepackage{manyfoot}\usepackage{booktabs}\usepackage{algorithm}\usepackage{algorithmicx}\usepackage{algpseudocode}\usepackage{listings}

\usepackage{mathtools}
\usepackage{enumitem}
\setlist[enumerate]{leftmargin=*,topsep=0.2cm}
\setlist[itemize]{leftmargin=*,topsep=0.2cm}
\usepackage{subcaption}
\usepackage{sidecap}
\usepackage[normalem]{ulem}

\newcommand{\addrebut}[1]{{\color{blue} #1}}
\newcommand{\rmrebut}[1]{\textcolor{blue}{\sout{#1}}}

\renewcommand{\addrebut}[1]{#1}
\renewcommand{\rmrebut}[1]{}

\newcommand{\unchanged}[1]{[Unchanged] \textit{#1}}
\newcommand{\newtext}[1]{[New] #1}
\usepackage{mycommands}
\usepackage{cleveref}
\theoremstyle{thmstyleone}\newtheorem{theorem}{Theorem}[section]\newtheorem{lemma}[theorem]{Lemma}\newtheorem{proposition}[theorem]{Proposition}\newtheorem{fact}[theorem]{Fact}
\newtheorem{corollary}[theorem]{Corollary}

\theoremstyle{thmstyletwo}\newtheorem{example}[theorem]{Example}\newtheorem{remark}[theorem]{Remark}

\theoremstyle{thmstylethree}\newtheorem{definition}[theorem]{Definition}\newtheorem{assumption}[theorem]{Assumption}\Crefname{assumption}{Assumption}{Assumptions}

\usepackage{tcolorbox}
\tcbuselibrary{theorems}
\newtcolorbox{box_light}{colback=blue!5!white,            colframe=blue!5!white,                  boxrule=0pt,                     sharp corners,                   top=4pt,
    left=2pt,
    right=2pt
}

\newenvironment{question}
  {\begin{center}\itshape}
  {\end{center}}

\graphicspath{{figures/}}

\begin{document}

\title[Article Title]{Convexity in ReLU Neural Networks: beyond ICNNs?}

\author*[1]{\fnm{Anne} \sur{Gagneux}}\email{anne.gagneux@ens-lyon.fr}

\author[2]{\fnm{Mathurin} \sur{Massias}}

\author[3]{\fnm{Emmanuel} \sur{Soubies}}
\author[2]{\fnm{R{\'e}mi} \sur{Gribonval}}

\affil*[1]{\orgname{Ens de Lyon, CNRS, Universit{\'e} Claude Bernard Lyon 1, Inria, LIP, UMR 5668}, \orgaddress{\street{69342}, \city{Lyon Cedex 07}, \country{France}}}

\affil*[2]{\orgname{Inria, ENS de Lyon, CNRS, Universit{\'e} Claude Bernard Lyon 1, LIP, UMR 5668}, \orgaddress{\street{69342}, \city{Lyon Cedex 07}, \country{France}}}

\affil[3]{\orgname{Universit{\'e} de Toulouse, CNRS, IRIT, UMR 5505}, \orgaddress{\street{31000}, \city{Toulouse} \country{France}}}

\abstract{

Convex functions and their gradients play a critical role in mathematical imaging, from proximal optimization to Optimal Transport.
The successes of deep learning has led many to use \emph{learning-based methods}, where fixed functions or operators are replaced by learned neural networks.
Regardless of their empirical superiority, establishing rigorous guarantees for these methods often requires to impose structural constraints on neural architectures, in particular convexity.
The most popular way to do so is to use so-called Input Convex Neural Networks (ICNNs).
In order to explore the expressivity of ICNNs, we provide necessary and sufficient conditions for a ReLU neural network to be convex.
Such characterizations are based on product of weights and activations, and  write nicely for any architecture in the \emph{path-lifting} framework.
As particular applications, we study our characterizations in depth for 1 and 2-hidden-layer neural networks: we show that every convex function implemented by a 1-hidden-layer ReLU network can be also expressed by an ICNN with the same architecture; however this property no longer holds with more layers.
Finally, we provide a numerical procedure that allows an exact check of convexity for ReLU neural networks with a large number of affine regions.
}

\maketitle

\section{Introduction}
\label{section:introduction}

There exists a strong demand for neural networks that implement expressive convex functions.
Here are a few use cases.
In image processing, Plug-and-Play (PnP) methods \cite{venkatakrishnan2013plug} replace explicit proximal operators by learned denoisers using neural networks \cite{meinhardt2017learning,zhang2017beyond,zhang2021plug,pesquet2021learning}.
Since proximal operators are (sub)gradients of convex functions
 \cite{gribonval2020characterization}, an important ingredient in the mathematical analysis of PnP schemes is to enforce the learned denoiser to be the gradient of a convex function -- e.g. through automatic differentiation of a convex neural network \cite{hurault2023convergent}.
More generally, handling convex priors in inverse problems is often beneficial \cite{cohen2021has,fang2024whats}.
Another example is in optimal transport, where Brenier's theorem \cite{brenier1991polar} states that any Monge map which moves an absolutely continuous distribution to another distribution while minimizing a quadratic cost is the gradient of a convex function.
Hence, many applications which build upon the optimal transport framework seek to train convex neural networks to learn Monge maps \cite{korotin2019wasserstein,makkuva2020optimal,korotin2021neural,amos2022amortizing}.

Though there exist alternatives \citep{saremi2019approximating,chaudhari2023learning,hurault2023convergent}, Input Convex Neural Networks (ICNNs) \cite{amos2017input} are by far the most-used solution when it comes to implementing convex functions with neural networks.
ICNNs are a simple approach to enforce convexity which only requires slight modifications to popular architectures (e.g. ResNets, Convolutional Neural Networks, etc).
In short, clipping to zero the negative weights of all hidden layers of a $\relu$ neural network (except the first one) is sufficient to obtain a convex neural network.

Besides providing a very easy-to-implement approach to get convex neural networks and being widespread in practice, ICNNs are also grounded
on the theoretical side: \citet{chen2018optimal} show that ICNNs are rich enough to approximate any convex function over a compact domain in the sup norm.

But is the story over?
Not quite.
First, ICNNs have shown poor scalability properties \cite{korotin2023neural} and there remains a need for more expressiveness, as suggested by the introduction of several variants of
the standard ICNN architecture \cite{korotin2019wasserstein,vesseron2024neural}.
Second, the architecture that \citet{chen2018optimal} exhibit has one neuron per layer and \emph{as many layers as affine pieces in the implemented function}, which is far from what one would use in practice.
These observations lead us to the following:
\vspace{1mm}
\begin{question}
Given a $\relu$ neural network architecture with set depth and widths, do ICNNs cover all convex functions that are implementable with this architecture?
\end{question}
\vspace{1mm}

In this paper, we first give a positive answer for neural networks with one hidden-layer: we cannot hope to do better than ICNNs.
Yet, with 2 hidden layers, we can already exhibit a network
which does not abide by the non-negativity constraint of ICNNs but still implements a convex function.
Moreover, we show that \emph{no ICNN with the same number of layers and neurons per layer} can implement this function.
These results stem from our work to provide \emph{exact}
characterizations of $\relu$ neural network parameters (on general deep $\relu$ architectures that notably allow modern features such as skip connections: connections -- weighted or not -- between neurons of non-consecutive layers) which yield convex functions, beyond ICNNs.

Since the functions implemented by such networks are continuous and piecewise affine, the first stage of our analysis is to characterize the convexity of such functions.
This characterization is reminiscent of global convexity conditions for piecewise convex functions \citep{bauschke2016convexity}, but  is somehow made ``minimal'' by finely exploiting the affine nature of the pieces.

The paper is organized as follows: \Cref{sec:relu_nn} formalizes the considered neural network architectures, and provides (\Cref{prop:ICNN_counter_ex}) our simple, two-hidden-layer motivating counter-example of a convex $\relu$ network that  cannot be implemented by an ICNN with the same architecture.
Then, we characterize convexity of continuous piecewise affine functions (\Cref{section:cpwl}) and provide a first translation of these conditions to one-hidden-layer $\relu$ networks (\Cref{sec:onehiddenlayer}), showing
(\Cref{prop:ICNNallyouneed}) that
in this simple case ICNNs ``are all you need''.

The characterization of convexity conditions in the one- and two-hidden-layer case (\Cref{sec:onehiddenlayer,sec:twohiddenlayer}) serve as a gentle entrypoint to handle more general $\relu$ networks in \Cref{section:dag,section:cvx_relu_sufficient}.
The path-lifting formalism
of \cite{gonon2023approximation} for such DAG (Directed Acyclic Graph) networks is recalled in  \Cref{section:dag} where we translate the convexity conditions of \Cref{section:cpwl} to necessary conditions for DAG $\relu$ networks.
\emph{Sufficient} convexity conditions for DAG networks are expressed in \Cref{section:cvx_relu_sufficient}.
Finally, we use these theoretical results to derive in \Cref{sec:numerics} a practical algorithm that is able to check convexity on small networks.
 \section{Convex $\relu$ networks}
\label{sec:relu_nn}

Our goal throughout this paper is to study how well $\relu$ networks can express convex functions.
Two key observations will guide our study:
\begin{enumerate}[leftmargin=*,topsep=0.2cm]
      \item Functions implemented by $\relu$ networks are Continuous PieceWise Linear (CPWL) \cite{arora_understanding_2018}.
  \item {\em Convex} CWPL functions can uniformly approximate arbitrarily well any  Lipschitz-continuous convex function on a compact domain \cite{chen2018optimal}.
\end{enumerate}

In other words, \emph{the ability of a given $\relu$ neural network architecture to approximate convex functions comes down to its capability to represent convex CPWL functions}.
Therefore, we will focus on how effectively $\relu$ networks can implement convex CPWL functions.

\subsection{$\relu$ MLPs and ICNNs}

We consider a $\relu$ neural network associated to a function $f_{\theta} : \bbR^d \to \bbR$ which is a function of the input $x \in \bbR^d$ and some parameters $\theta$.

In the case of a multilayer perceptron (MLP) or its variants with weighted input skip-connections,
let $L \geq 2$ denote the number of (linear) layers.
The set of neurons is partitioned into $\Nin=N_{0}$  the input layer ($|\Nin|  = d$), $N_1, \dots N_{L-1}$ the hidden layers,  and $\Nout  = N_L$ the last (output) layer (to implement convex functions we focus on scalar-valued networks, so that $|N_{L}| =1$). We denote $H := \cup_{\ell=1}^{L-1} N_{\ell}$ the set of all hidden neurons.

The neural network function $f_{\theta}$ is implemented as follows: starting from an input $x \in \bbR^d$, we iteratively compute the \emph{pre-activation} $z_{\ell} \in \bbR^{N_{\ell}}$
of each layer $\ell$ to obtain the network output as \begin{align*}
    z_1(x, \theta) &  := \text{  affine function of
      $x$}, \\
    z_\ell(x, \theta) & :=  \text{ affine function of
     $\relu(z_{\ell-1})$ and $x$}, \\
    &  \qquad \ell  \in \{2, \dots, L \}, \\
   f_{\theta}(x) & := z_L(x,\theta),
\end{align*}
where $\theta$ represents the collection of weights and biases in the network's linear functions. We can also define the pre-activation \emph{ of a neuron}\footnote{It will be convenient to consider neurons as distinct elements of the set $N := \cup_{\ell=0}^{L} N_{\ell}$ (with $N_{0}:= N_{\mathrm{in}}$) rather than indices $1 \leq i \leq n_{\ell}:= |N_{\ell}|$. Hence the notation $\bbR^{N_{\ell}}$ instead of $\bbR^{n_{\ell}}$. Accordingly $W_{\ell} \in \bbR^{N_\ell \times N_{\ell-1}}$ and $b_\ell \in \bbR^{N_\ell}$.}  $\nu \in N_{\ell}$ as $z_\nu(x, \theta) = (z_\ell(x, \theta))_\nu$.

\begin{figure}[htp]
    \centering
    \includegraphics[width=\linewidth]{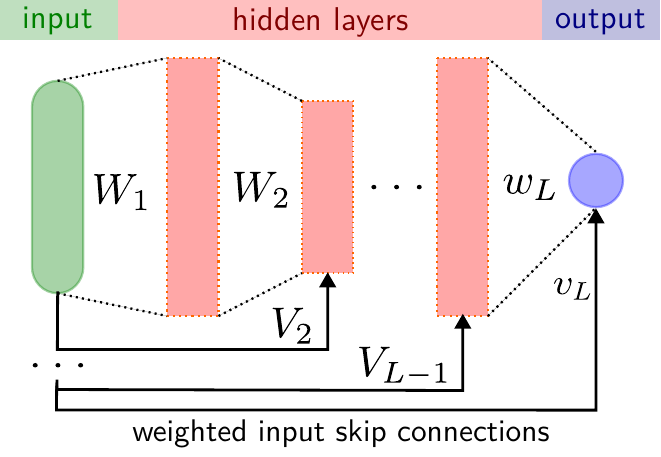}
    \caption{Architecture of a $\SMLPclass$ network.  $\ICNNclass$ imposes weights $W_2, \ldots, W_{L-1}, w_L$ to be non-negative entry-wise.}
    \label{fig:architectures}
\end{figure}
ICNNs are a subset of Multi Layers Perceptrons (MLPs) {\em including linear, weighted, input skip connections} with the additional requirement of non-negativity of certain weights.
Based on this difference, we now formalize
two classes of functions (see \Cref{fig:architectures}): the general class of MLPs with weighted input skip connections (unconstrained) and the constrained subset corresponding to ICNNs.
To specify an architecture, we denote $\bfn = (n_{\ell})_{\ell=1}^{L-1}$ the tuple of widths, \emph{excluding the width $n_{L}$ of the last layer} as we focus on scalar-valued networks ($n_{L}=1$).
Given $\bfn$:
\begin{itemize}[topsep=1mm]
    \item $\SMLPclass_d(\bfn)$ is the set of all functions  $ f : \bbR^d \to \bbR$ that can be implemented as a neural network with $L$ layers ($L-1$ hidden layers) and each hidden layer $N_{\ell}$, $\ell \in \{1 , \dots, L-1 \}$ has at most $n_\ell$ $\relu$ neurons.
    In addition are allowed \emph{weighted input skip connections}, linking the input $x$ to each layer $\ell \geq 1$.
    In this setting, each pre-activation for $\ell \geq 2$ writes as
    \begin{align*}
        z_\ell(x, \theta) & :=  W_\ell \relu(z_{\ell-1}(x, \theta)) + V_\ell x +  b_\ell,
    \end{align*}
    and $\theta = (W_1, b_1, W_2, V_2, b_2, \dots, w_L, v_L, b_L)$.
    \item $\ICNNclass_d(\bfn) \subseteq \SMLPclass_d(\bfn)$ corresponds to the subset of functions implemented by ICNNs:
    MLPs
     with weighted input skip connections {\em with the additional constraint} that the weight matrices/vectors $(W_2, \dots, W_{L-1}, w_{L})$ have non-negative entries.
    Noticeably, {\em no non-negativity constraint} is imposed on $W_{1}$ nor $V_{\ell}$, $2 \leq \ell \leq L$.
\end{itemize}
We will usually index weights matrices or bias vectors by neurons: e.g.
 given two neurons $\mu$ and $\nu$ in consecutive layers $N_{\ell - 1}$ and $N_{\ell}$ with $\ell \in \{ 1 , \dots, L \}$, the weight connecting $\mu$ to $\nu$ is $W_\ell[\nu, \mu]$.

We also refer to {\em standard} MLPs when we consider MLPs \emph{without input skip connections}: they can be retrieved from the previous classes by imposing zero weights on the input skip connections, {\em i.e.}, $V_{\ell}=0$.
Finally, in \Cref{section:dag}, we will consider a general family of $\relu$ neural networks with Directed Acyclic Graphs (DAGs) architectures.
\subsection{Convexity of ICNNs}

The convexity of ICNNs is based on rules for composition of convex functions: the composition $f \circ g$ of a convex function $g$ with a non-decreasing convex function $f$ is convex \cite[Sec. 3.2.4]{boyd2004convex}, and so is any linear combination of convex functions using non-negative coefficients.
By recursion over the layers, the entry-wise non-negative weight constraints $W_\ell \geq 0$ ($\ell \geq 2$) of ICNNs ensure convexity of each neuron's pre-activation, and of the overall function.
Indeed, the first layer implements an affine function, so each component of $z_1$ is convex.
Then, all the subsequent operations are non-decreasing and convex: using convex non-decreasing activation functions ($\relu$ in our case) and non-negative weight matrices preserves convexity.
Besides, this approach also allows for weighted input skip-connections of unconstrained sign between the input and any layer as the sum of a convex function and a linear function remains convex.

\subsection{Limits of ICNNs} \label{subsection:iccn_counter_ex}

While \citet{chen2018optimal} have shown that any convex CPWL function can be implemented with an ICNN, the corresponding implementation has a very peculiar architecture, of depth {\em equal to the number of affine pieces of the function}.
This seems potentially suboptimal, and indeed we show in this section that to implement some convex CPWL functions, ICNNs require more neurons or layers that their unconstrained  counterpart.
In other words, there are convex CPWL functions
in  $\SMLPclass_d(\bfn) \setminus \ICNNclass_d(\bfn)$.
To establish this, we simply provide a concrete example.

\begin{proposition}\label{prop:ICNN_counter_ex}
Let $f_\ex:\bbR^2 \to \bbR$ defined by
\begin{multline}
    \label{eq:Ex2DConvex}
    f_\ex \begin{pmatrix} x_1 \\ x_2 \end{pmatrix}  = \begin{pmatrix}
        1 & 1
    \end{pmatrix} \relu \left(\begin{pmatrix*}[r]
        -1 & 1 \\
        2 & 1
    \end{pmatrix*} \relu \begin{pmatrix} x_1 \\ x_2 \end{pmatrix} \right. \\
    \left. + \begin{pmatrix*}[l]
        -1 \\
        -0.5
    \end{pmatrix*} \right) \ .
\end{multline}
    This function is convex and CPWL (see \Cref{fig:counter_ex_plots}).
    By its definition, it can be implemented by a network in $\SMLPclass_2((2,2))$, {\em i.e.}, an MLP with  $2$ hidden layers, each made of $2$ neurons.
    Yet, it cannot be implemented by any ICNN network with the same architecture, i.e  belonging to $\ICNNclass_2((2,2))$.
\end{proposition}

The proof is in \Cref{app:counter_ex}.
This example was constructed from a careful study of necessary and sufficient conditions for convexity of CPWL functions, that we detail in the next section.

 \section{Convex CPWL functions}
\label{section:cpwl}
Our characterization of the convexity of functions implemented by $\relu$ networks leverages a
general machinery to characterize convexity of CPWL functions, that we establish in this section.

\subsection{CPWL functions: definition}
\label{sub:cpwl_def}

Several ways of defining CPWL functions can be found in the literature: either from their linear pieces \cite{goujon2024number}, or from the notion of polyhedral complex and cells \cite{grigsby2022transversality,masden2022algorithmic}, or also from the notion of polyhedral partition and regions \cite{gorokhovik1994piecewise}.
We follow the latter approach as it removes some terminology related to polyhedra while allowing a rather simple connection with $\relu$ networks.
A brief complementary reminder is provided in \Cref{app:background_cpwl}.

\begin{definition}[Polyhedral partition, \cite{gorokhovik1994piecewise}]\label{def:poluhedron_part}
        A finite family of convex polyhedral sets $(R_k)_{k=1}^K$ is said to form \emph{a polyhedral partition} of $\bbR^d$ if
        \begin{enumerate}[label=(\roman*)]
            \item $\bbR^d = \cup_{k=1}^K R_k$,
            \item $\intt(R_k) \neq \emptyset$ for every $k \in \{1, \dots, K\}$,
            \item $\intt(R_k) \cap \intt(R_\ell) = \emptyset$ for every $k,\ell \in \{1, \dots, K \}, k \neq \ell$.
        \end{enumerate}
        The sets $(R_k)_{k=1}^K$ are called \emph{regions} of $\bbR^d$.
\end{definition}

\begin{definition}[CPWL function]
    \label{def:cpwl}
    A function $f : \bbR^d \to \bbR$ is said to be continuous and piecewise linear (CPWL) if it is continuous and if there exists a polyhedral partition $(R_k)_{k=1}^K$ of $\bbR^d$ such that $f_{|R_k}$ is affine for each $k \in \{1, \dots, K \}$.
\end{definition}
In \Cref{def:cpwl}, the polyhedral partition $(R_k)_{k=1}^K$ is not uniquely defined.

\begin{definition}[Compatible partition]
   A \emph{compatible partition} for a CPWL function $f$ refers to any polyhedral partition of $\bbR^d$ for which $f$ is affine on each set of the partition.
\end{definition}

In what follows, we introduce a few notions that will come in handy when studying the convexity of CPWL functions.
We have not found any equivalent to these definitions in the literature.

\begin{definition}[Neighboring regions]
    \label{def:neighbors}
    Two regions $R_k$ and $R_l$ of a polyhedral partition are  \emph{neighboring regions} if $k \neq \ell$ and $R_k$ and $R_\ell$ share a {\em facet}, i.e. the affine hull of $R_k \cap R_\ell$ is an affine hyperplane.
    We denote $R_k \sim R_\ell$ or $k \sim \ell$ in short the neighboring relationship between $R_k$ and $R_\ell$.
\end{definition}

\begin{definition}[Frontiers]
    \label{def:frontiers}
    Given a polyhedral partition $(R_k)_{k=1}^K$,
    the {\em frontier} $F_{k,\ell}$ between neighboring regions $R_k$ and $R_\ell$ is  the relative interior of $R_{k} \cap R_{\ell}$. \end{definition}

A visualization of neighbouring regions and their frontier is given in \Cref{fig:proof_cvx_cpwl_case_1}.

\begin{remark}
    \addrebut{By \Cref{def:polyhedron}, the polyhedra considered in this paper are \emph{convex polyhedra}, \emph{i.e.} they are defined as the intersection of finitely many closed half-spaces.
    There exists in the literature other definitions that do not enforce convexity.
    However, note that any non-convex polyhedral partition can always be subdivided into a convex polyhedral partition.}
\end{remark}

\subsection{``Minimal'' convexity conditions}

We consider in this section a CPWL function $f : \bbR^d \to \bbR$ and investigate convexity conditions.
Given a compatible polyhedral partition $(R_k)_{k=1}^K$, we denote  $u_k \in \bbR^d$ and  $b_k \in \bbR$ respectively the slope of $f_{|R_k}$ and its intercept:
\begin{equation}\label{eq:DefAffinePiece}
    f_{|R_k} : x \mapsto \langle u_k, x \rangle + b_k \ .
\end{equation}
This implies that the gradient of $f$ is well-defined at any $x \in \intt(R_k)$, with $\nabla f(x) = u_{k}$.

Any convex function $f$ must satisfy a well-known monotonicity condition  \citep[Example 20.3]{bauschke2011convex}
   \begin{equation}
        \label{eqn:monotonicity}
        \langle \nabla f(x)-\nabla f(y), x - y \rangle \geq 0
    \end{equation}
at every points $x,y \in \bbR^{d}$ where $f$ is differentiable. To go further and obtain necessary and sufficient conditions,
\citet{bauschke2016convexity} already studied conditions under which a piecewise-defined and piecewise-convex function is globally convex. CPWL functions fall under this framework.
More precisely, \citet{bauschke2016convexity}  provide sufficient conditions that only need to be checked at boundaries between regions. They also show that a finite number of points can be ignored when checking convexity.
Yet, as CPWL functions have {\em the same slope} on each of their affine components, we are able to derive ``minimal'' conditions for this specific case: convexity can be checked {\em at only a finite number of points}.
    The following proposition (\ref{prop:cvx_cpwl_item1} $\Leftrightarrow$\ref{prop:cvx_cpwl_item4}) asserts that {\em it suffices to check \eqref{eqn:monotonicity} on a finite number of points} (one pair of points around each frontier $F_{k,\ell}$) to assess convexity.

\begin{proposition} \label{prop:cvx_cpwl}
Consider a CPWL function $f : \bbR^d \to \bbR$, and any compatible partition $(R_k)_{k=1}^K$.
Denote $ \cF := \bigcup_{k \sim \ell} F_{k,\ell}$ the set of all frontier points.
The following are equivalent:
\begin{enumerate}[label=(\roman*)]
    \item \label{prop:cvx_cpwl_item1} The function $f$ is convex on $\bbR^d$.
    \item \label{prop:cvx_cpwl_item2}
    For all frontier points $x \in \mathcal{F}$ and $v \in \bbR^d$, there is $\epsilon > 0 $ s.t. $f_{\mid [x - \epsilon v, x+ \epsilon v]}$ is convex.
    \item  \label{prop:cvx_cpwl_item3}
    For all neighboring regions $R_k \sim R_\ell$, there are $x \in F_{k,\ell}$, $v \in \bbR^d \setminus \color{blue} \Span((R_k \cap R_\ell) - x)$  and $\epsilon > 0 $ such that $f_{\mid [x - \epsilon v, x+ \epsilon v]}$ is convex.
    \item (Minimal characterization) \label{prop:cvx_cpwl_item4}
    For all neighboring regions  $R_k \sim R_\ell$, there are $x_k \in \intt( R_k)$, $x_\ell \in  \intt(R_\ell)$ such that\footnote{$[x,y]$ denotes a line segment, see \eqref{eq:DefLineSegment} in \Cref{appendix:background_cvx}.}  $[x_k, x_\ell] \cap F_{k,\ell} \neq \emptyset$ and \begin{equation}
        \langle \nabla f(x_{k})-\nabla f(x_{\ell}) ,  x_k - x_\ell \rangle \geq 0 \ . \label{ineq:cvx_slopes}
    \end{equation}
    \item (Local characterization) \label{prop:cvx_cpwl_item5}  For all frontier points $x \in \mathcal{F}$ there exists a neighborhood $\cN$ of $x$ such that for every $x_1, x_2 \in \cN \setminus \cF$,  $f$ is differentiable at $x_1$ and $x_{2}$ and
    \begin{equation}
        \label{eqn:local_cvx_ineq}
        \langle \nabla f(x_1)-\nabla f(x_2) , x_1 - x_2 \rangle \geq 0 \ .
    \end{equation}
\end{enumerate}
\end{proposition}

The proof  in \Cref{appendix:proof_cvx_cpwl} uses tools from convex analysis reminded in
 \Cref{appendix:background_cvx}.

 \section{The one-hidden-layer case}
\label{sec:onehiddenlayer}
In this section, we use \Cref{prop:cvx_cpwl} to show that, for one hidden-layer networks, any convex neural network is an ICNN.
For this we first introduce a few additional notations and definitions for $\relu$ networks that are valid {\em beyond the one-hidden-layer case} and will be re-used throughout the paper.

\subsection{Preliminaries} \label{sec:prelim}

Checking the convexity of a CPWL function using \Cref{prop:cvx_cpwl} involves checking certain properties near {\em frontiers between regions}. In the case of $\relu$ networks, such frontiers are related to {\em changes in neuron activations}.

We thus start by introducing some notations regarding the activations of neurons of a neural network.
Let $f_\theta \in \SMLPclass_d(\bfn)$.
It is a CPWL function \cite{arora_understanding_2018}, as the composition and addition of affine and continuous piecewise linear functions.
Given an input $x \in \bbR^d$, if $\cN$ is a neighborhood of $x$ where all pre-activations $z_\nu$ of all hidden neurons $\nu \in H$ have constant sign, then $f_\theta$ is affine on $\cN$.
This observation leads us to define the (binary) activation of a hidden neuron $\nu \in H$:
\begin{equation}
    a_{\nu}(x,\theta) := \mathbf{1}_{z_\nu(x, \theta) > 0}  \in \{0, 1\}.
\end{equation}
We define similarly the layer-wise activation pattern $\bfa_\ell := (a_\nu)_{\nu \in N_\ell} \in \{0,1\}^{N_\ell}$ for the $\ell-$th hidden layer of $\relu$ neurons ($1 \leq \ell \leq L-1$).
Note that the function $a_{\nu}(\cdot, \theta)$ is piecewise constant.
\addrebut{See \cite{hanin2019deep} or \cite{stock_embedding_2022} for similar definitions of activation patterns. }
To verify convexity of $f_{\theta}$, we want to look at input points where the slope of $f_\theta$ changes, i.e. points around which the activation function of at least one neuron is not constant.
To study {\em necessary} conditions for convexity, we will even focus on
points where \emph{only} the activation of one specific neuron changes.

\begin{definition}[Isolated neuron]
    \label{def:isolated_neurons}
    Given \rmrebut{a parameter} \addrebut{parameters} $\theta$, for each hidden
    neuron $\nu  \in H:=\cup_{\ell=1}^{L-1}N_{\ell}$, we define\footnote{this set depends on $\theta$, but for the sake of brevity we omit it in the notation, as $\theta$ should always be clear from context} $\fX_\nu$ as the set of input points for which in every small enough neighborhood\footnote{$B(x,\epsilon)$ denotes the open ball ({\em e.g.} with respect to the Euclidean metric) of radius $\epsilon>0$ centered at $x$} , \emph{only} the activation of neuron $\nu$ switches
    \begin{multline*}
        \fX_\nu := \Bigg\{x \in \bbR^d:  \exists \epsilon_0>0\ : \forall \, 0<\epsilon \leq \epsilon_0,
         \\
  \begin{cases}
            a_\mu(\cdot,\theta) \text{ constant on } B(x, \epsilon)  \ \forall \mu \neq \nu,  &\\
            a_\nu(\cdot,\theta) \text{ not constant on } B(x, \epsilon).&
        \end{cases} \Bigg\} \ .
    \end{multline*}
    A hidden neuron $\nu$ is said to be \emph{isolated} if $\fX_\nu \neq \emptyset$.
\end{definition}

\subsection{ICNNs are all you need}

Consider a one-hidden-layer standard MLP (thus \emph{without} input skip-connection) and
 denote $f_\theta$ the CPWL function it implements, which thus reads:
\begin{equation}\label{eq:twolayer}
    f_\theta  : x \in \bbR^{d} \mapsto w_2^\top \relu(W_1 x + b_1) + b_2 \in \bbR
\end{equation}
with $w_{2},b_{1} \in \bbR^{N_{1}}$ and $W_{1} \in \bbR^{N_{1} \times N_{0} }$.

We prove non-negativity of the last layer $w_2$ associated to isolated neurons when $f_{\theta}$ is convex.
\begin{lemma}[Non-negativity of the last layer]
    \label{lemma:last_layer}
    Assume that $f_{\theta}$ expressed as~\eqref{eq:twolayer} is convex.
    Then, for any \emph{isolated} hidden neuron of the (only) hidden layer $\nu \in
    N_{1}$, the corresponding weight in the last layer must be non-negative: $w_2 [\nu] \geq 0 $.
\end{lemma}

\begin{proof}
    Consider an isolated neuron $\nu \in N_{1}$, $x \in \fX_\nu \neq \emptyset$, and a neighborhood $B(x,\epsilon)$ where only the activation of $\nu$ changes, as given by \Cref{def:isolated_neurons}.
    Using that $\relu(t) = \mathbf{1}_{t>0} t$ it is standard to rewrite $f_{\theta}$ as
    \begin{align}
        f_\theta : x \mapsto w_2^\top \diag(\bfa_1(x,\theta)) (W_1 x + b_1) + b_2 \, .
    \end{align}
    The neuron $\nu$ being isolated, the layerwise activation pattern $\bfa_{1}(\cdot,\theta) \in \{0,1\}^{N_{1}}$ takes two different values on $B(x,\epsilon)$, denoted $\bfa^+$ and $\bfa^-$ with $\bfa^+_{\mu} = \bfa^-_{\mu}$ for all hidden neurons $\mu \in N_{1} \setminus \{\nu\}$
    (because $a_\mu(\cdot,\theta)$ is constant on $B(x, \epsilon)$).
    Wlog, we assume $\bfa^+_{\nu} = 1$ and $\bfa^-_{\nu} = 0$.
     By construction, there exists $x^+,x^- \in B(x,\epsilon)$ such that $\bfa_{1}(x^+,\theta) = \bfa^{+}$, $\bfa_{1}(x^-,\theta) = \bfa^{-}$ and that $f_{\theta}$ is differentiable at both points. Denoting  $u^{+} = \nabla f_{\theta}(x^{+})$, $u^{-} = \nabla f_{\theta}(x^{-})$,
        we have:
    \begin{align*}
        (u^+ - u^-)^\top & = w_2^\top \left( \diag(\bfa^+) - \diag(\bfa^-) \right) W_{1}  \\
        & =  w_2 [\nu] W_1 [\nu,:] \, .
    \end{align*}
     Since $\bfa^+_\nu = 1$, the pre-activation of $\nu$ at $x^{+}$ is positive: $z_\nu(x^+,\theta) > 0 $.
    Similarly, because $\bfa^-_\nu = 0$ we have $z_\nu(x^-,\theta) \leq 0$.
    This rewrites
    \begin{align*}
        (W_1 x^+ + b_1)[\nu] & > 0 \, , \qquad
        (W_1 x^- + b_1)[\nu] & \leq 0 \, ,
    \end{align*}
    from which we get
    \begin{equation}
         W_{1} [\nu, : ] (x^+ -x^-) > 0 \, .
    \end{equation}
    Combining with the above expression of $u^+-u^-$, we obtain
    \begin{align*}
      \langle u^+ - u^-, x^+ - x^- \rangle = w_2[\nu]  \underbrace{W_1 [\nu,:](x^+ - x^-)}_{>
      0 }.
    \end{align*}
   Since $f_\theta$ is convex, by monotonicity of the gradient (cf \Cref{eqn:monotonicity}) it holds $\langle u^+ - u^-, x^+ - x^- \rangle \geq 0$ from which we get $w_2 [\nu]   \geq 0$.
\end{proof}

\begin{remark}
    Even though we focus in this section on the one-hidden-layer case, \Cref{lemma:last_layer} extends to any
    standard MLP
     with $L$ layers with essentially the same proof.
	In \Cref{prop:nec_relu_cvx} we state the most general result which applies to $\relu$ networks of any depth, in their general DAG (Directed Acyclic Graph) form.
\end{remark}

We have just shown that if $f_\theta$ is convex, any weight $w_2[\nu]$  corresponding to an \emph{isolated} hidden neuron $\nu$ (i.e. a neuron which can be activated independently of the other neurons) must be non-negative.
Before trivially extending the result to one-hidden-layer $\relu$ networks {\em with input skip connections}, we characterize one-hidden-layer networks where {\em all} hidden neurons $\nu \in N_{1}$ are isolated.
\begin{assumption}
    \label{assumption:1_layer_colin}
    For \emph{all} neurons $\nu$ of the first hidden layer $N_{1}$, we have $W_1[\nu,:] \neq 0$.
    For all pairwise distinct neurons $\nu_1 \neq \nu_2$ of this
     layer, the corresponding augmented
     rows $\left( W_1[\nu_1,:]  \ \vert \   b_1[\nu_1] \right)  \in \bbR^{N_{1}+1}$ and $ \left( W_1[\nu_2,:] \ \vert \ b_1[\nu_2] \right)\in \bbR^{N_{1}+1}$ are not co-linear.
\end{assumption}

\addrebut{The next results serve as steps towards the main proposition \ref{prop:ICNNallyouneed}.
They show that under \Cref{assumption:1_layer_colin}, every hidden neuron is isolated, which allows to use \Cref{lemma:last_layer}.
This finally leads to the equivalence between convex one-hidden-layer $\relu$ networks and ICNNs.}

 \begin{lemma}\label{lem:nosisolatedcaract}
    Consider a one-hidden-layer $\relu$ network, possibly with weighted input skip connections.
    \Cref{assumption:1_layer_colin} holds if and only if every neuron $\nu \in N_1$ is isolated.
 \end{lemma}

 \begin{proof}
    We prove the forward implication \rmrebut{(the other implication is left to the reader)}. Denote $H_\nu := \{ x : W_1[\nu,:] x + b_1[\nu] =0 \}$ the $0$-level set of $z_\nu(\cdot,\theta)$ for $\nu \in N_1$.
    By~\Cref{assumption:1_layer_colin}:
    \begin{itemize}[leftmargin=*]
        \item For every neuron $\nu \in N_1$, $z_\nu(\cdot, \theta)$ cannot be of constant sign on $\bbR^d$, as otherwise we would have $W_1[\nu,:] = 0$. So, $a_\nu(\cdot,\theta)$ is not constant on $\bbR^d$.
        \item Given two distinct neurons $\nu_1 \neq \nu_2$ in the first layer $N_1$, the corresponding
        level sets $H_{\nu_1},H_{\nu_2}$ are hyperplanes, and are distinct.
 \end{itemize}
Fix $\nu \in N_1$. Since all considered hyperplanes are pairwise distinct, there exists $x \in H_\nu$ such that for every $\mu \neq \nu$, we have $x \not \in H_\mu$.
Since $x \in H_{\nu}$ we have
$z_\nu(x,\theta) = 0$ while since $x \notin H_{\mu}$ we have $z_\mu(x,\theta) \neq 0$ for every $\mu \neq \nu$. As all functions $z_\mu(\cdot,\theta)$ are continuous, it follows that there exists an open ball $B(x,\epsilon)$ on which $\mathrm{sign}(z_\mu(\cdot,\theta)) \neq0$ is constant for every $\mu \neq \nu$.
This implies that $a_{\mu}(\cdot,\theta)$ is constant on $B(x,\epsilon)$ for every $\mu \neq \nu$ and establishes $x \in \fX_\nu$, thus proving that $\fX_\nu \neq \emptyset$, {\em i.e.} by definition $\nu$ is isolated. Since this holds for every $\nu \in N_{1}$ this proves the forward implication.
\addrebut{For the reverse implication, the same arguments hold. Assume $\nu \in N_1$ is isolated (\emph{i.e.} $\fX_\nu \neq \emptyset$), then $a_\nu(\cdot, \theta)$ is not constant on $\bbR^d$, forbidding $W_1[\nu,:]=0$. Since, by definition of $\fX_\nu$, there exists a ball around $x \in \fX_\nu$ such that $a_{\mu}(\cdot,\theta)$ is constant for $\mu \neq \nu$ and $a_\nu(\cdot, \theta)$ is not constant, it prevents $H_\mu$ and $H_\nu$ from being coincident.}
\end{proof}

We are now equipped to characterize the convexity of $f_{\theta}$ of the form~\eqref{eq:twolayer} when all neurons are isolated.
\begin{proposition} \label{prop:nsc_1_layer}
    Under \Cref{assumption:1_layer_colin}, $f_\theta$ of the form~\eqref{eq:twolayer} is convex if and only if its last layer is non-negative: $\forall \nu \in N_1, w_2[\nu] \geq 0$.
\end{proposition}

\begin{remark}
    \Cref{assumption:1_layer_colin} cannot be simply omitted:
    for example, consider the 1D example $f_\theta : \bbR \to \bbR$ and take $W_1 = (1,-1)^\top, b_1 = (0,0)^\top$ (2 hidden neurons), and $w_2 = (1, -1)^\top$, then $f_\theta(x)= x$, which is convex although $w_{2}$ has a negative entry.
\end{remark}

\begin{proof}[Proof of~\Cref{prop:nsc_1_layer}]
    Assume $f_\theta$ convex.
    By \Cref{assumption:1_layer_colin} and \Cref{lem:nosisolatedcaract}, every neuron $\nu \in N_1$ is isolated.
    By~\Cref{lemma:last_layer} it follows that
    $w_2 [\nu] \geq 0 $ for every $\nu \in N_1$.
    Conversely, if $w_2 [\nu] \geq 0$ for every $\nu \in N_1$ then the network is an ICNN, hence $f_{\theta}$ is convex.
\end{proof}

The extension to one-hidden-layer $\relu$ networks with input skip connections is an immediate corollary, and the main result of this section.

\begin{box_light}
    \begin{corollary}[\textbf{Convex one-hidden-layer $\relu$ networks are ICNNs}]
        \label{cor:icnnonelayer}
    Under \Cref{assumption:1_layer_colin}, the only convex one hidden-layer $\relu$ networks with weighted input skip connections are Input Convex Neural Networks, i.e. networks which write as
        \begin{equation*}
            f_\theta : x \in \bbR^d \mapsto w_2^\top \relu(W_1 x + b_1) + b_2 + v_2^\top x \ ,
        \end{equation*}
    with $w_2 \geq 0$.
    \end{corollary}
\end{box_light}

\begin{proof}
    \addrebut{The result follows directly from \Cref{prop:nsc_1_layer} as adding or removing a linear term  $x \mapsto v_2^\top x$ (corresponding to a weighted input skip connection) to the function does not change its convexity.}
    \rmrebut{By \mbox{\Cref{prop:nsc_1_layer}}, under \mbox{\Cref{assumption:1_layer_colin}}, $x \in \bbR^d \mapsto w_2^\top \relu(W_1 x + b_1) + b_2$ is convex if and only if $w_2 \geq 0 $.}
    \rmrebut{Adding or removing the linear term  $x \mapsto v_2^\top x$ (corresponding to a weighted input skip connection) does not change the convexity of the considered function, so $x \in \bbR^d \mapsto w_2^\top \relu(W_1 x + b_1) + b_2 + v_2^\top x$ is convex if and only if $w_2 \geq 0$.}
    \rmrebut{The condition $w_2 \geq 0$ is precisely the ICNN requirement.} \end{proof}

\Cref{assumption:1_layer_colin} is not restrictive, as networks which do not satisfy it can be rewritten as networks of the same depth and fewer neurons that do satisfy it, leading to the following proposition.
\begin{box_light}

\begin{proposition}\label{prop:ICNNallyouneed}
    Denote $\CvxCPwL_d$ the class of convex CPWL functions from $\bbR^d$ to $\bbR$.
    Then, it holds for any one-hidden-layer architecture $\bfn = (n)$ with $n \in \bbN^*$ and any dimension $d \in \bbN^*$:
    \begin{equation}
    \SMLPclass_d(\bfn ) \cap \CvxCPwL_d =   \ICNNclass_d(\bfn ),
    \end{equation}
    i.e. a convex function implemented by a one-hidden-layer $\relu$ network with weighted input skip connections can always be also implemented by a one-hidden-layer ICNN with the same width.
\end{proposition}
\end{box_light}

\begin{proof}
    Consider a one-hidden-layer architecture $\bfn = (n)$ for some width $n \in \bbN^*$ and let $d \in \bbN^*$ be the input dimension.
    The inclusion $\ICNNclass_d(\bfn ) \subset \SMLPclass_d(\bfn ) \cap \CvxCPwL_d $ is trivial.
    Consider a {\em convex} CPWL function $f$ implemented by a network $\SMLPclass_d(\bfn )$.
    If \Cref{assumption:1_layer_colin} is satisfied, then by \Cref{cor:icnnonelayer}, $f$ belongs to $\ICNNclass_d(\bfn )$.
    Otherwise, we show that $f$ can be implemented by some one-hidden-layer Skip-MLP with width $n' \leq n$ which satisfies \Cref{assumption:1_layer_colin}:
    \begin{itemize}
        \item if $W_{1}[\nu,:]=0$ for some $\nu \in N_{1}$, this neuron can be removed from the network architecture, with no impact on the implemented function $f$, up to adjusting the bias of the last layer.
        \item if two neurons have co-linear augmented rows (or equivalently have the same hyperplanes
        $H_{\nu_1}=H_{\nu_2}$), they can be merged (see the discussion on {\em twin neurons} in \citet{stock_embedding_2022}). \qedhere
    \end{itemize}
\end{proof}

To summarize, once we exclude trivial degeneracies of network parameterizations corresponding to \Cref{assumption:1_layer_colin}, the \emph{only} convex one-hidden-layer $\relu$ networks (with or without input skip connection) are Input Convex Neural Networks.
Does this extend to deeper $\relu$ networks? Clearly not: \Cref{prop:ICNN_counter_ex} yields a non-ICNN two-hidden-layer $\relu$ network that \emph{does} implement a convex CPWL function.  Moreover this  function \emph{cannot} be implemented with \emph{any}
 ICNN with the same depth and width.
Next we investigate extensions and adaptation beyond the one-hidden-layer case of the necessary non-negativity condition of \Cref{lemma:last_layer}.

\section{The two-hidden-layer case}
\label{sec:twohiddenlayer}

\addrebut{This section is here as an intermediate step before addressing general architectures.
While readers can skip it and directly go to the more general results in \Cref{section:dag}, the two-hidden-layer case illustrates the main ideas of the reasoning and proofs without relying on the more abstract framework of the next section.}
\rmrebut{Before proving our most general result for any DAG $\relu$ network, we study a simple two-hidden-layer \emph{standard} MLP.}
\addrebut{We thus study a two-hidden-layer standard MLP implementing} the CPWL function
\begin{equation}\label{eq:2HL_feedfwd}
    f_\theta  : x \mapsto w_3^\top \relu(W_2 \relu(W_1 x + b_1) + b_2) + b_3.
\end{equation}
The obtained necessary conditions involve, as in the one-hidden-layer case, non-negativity of the last layer, but also of certain {\em products of weights}.

Regarding necessary convexity conditions on the weights of the last layer,
it is not difficult to adapt the proof of  \Cref{lemma:last_layer} (by replacing $\nu \in N_{1}$ with $\nu \in N_{2}$, and identifying proper slope vectors $u^{+},u^{-}$) to show that if $f_\theta$ convex, then all weights of the last layer associated to isolated hidden neurons $\nu \in N_{2}$ are again non-negative: $w_{3}[\nu] \geq 0$.
This necessary condition will be proved formally in the general case of DAG $\relu$ networks in the next section (\Cref{lemma:last_layer_general}).

Concerning other weights, we establish an additional necessary condition associated to isolated neurons of the {\em first} hidden layer, $\nu \in N_{1}$, and illustrate on the example of \Cref{prop:ICNN_counter_ex}, as a sanity check, that this condition indeed holds.

\subsection{Additional necessary condition}
The additional condition involves the reachable activation patterns of the second hidden-layer, i.e. the image of $\fX_{\nu}$ via $\bfa_2(\cdot,\theta)$, {\em cf.} the notations of \Cref{sec:prelim}.

\begin{lemma}
    \label{lemma:pos_cond_2_layer}
    Consider a \emph{convex two-hidden-layer} $\relu$ network $f_\theta : \bbR^d \to \bbR$ of the form~\eqref{eq:2HL_feedfwd}.
    For each \emph{isolated neuron} of the first hidden layer $\nu \in N_{1}$, it holds
    \begin{equation}\label{eq:pos_cond_2_layer}
        \min_{\bfa \in \bfa_2(\fX_\nu)}
             \langle \bfa, (w_3 \odot W_2 [:,\nu])\rangle \geq 0 \ ,
    \end{equation}
    where $\odot$ is the entry-wise (Hadamard) product, and $\bfa_2(\fX_\nu) = \{ \bfa_2(x, \theta) : x \in \fX_\nu \}$.
\end{lemma}

The proof follows the same approach as the one for non-negativity of the last layer in the one-hidden-layer case (\Cref{lemma:last_layer}).
More precisely, for each isolated $\nu \in N_1$, it compares the two possible slopes of $f_\theta$ at points within a neighborhood of a point where the activation of $\nu$ is non constant while other neurons have constant activations.

\begin{proof}[Proof of \Cref{lemma:pos_cond_2_layer}.]
     Since $\nu \in N_1$ is isolated, we can pick $x \in \fX_\nu \neq \emptyset$ together with a  ball $B(x,\epsilon)$ given by \Cref{def:isolated_neurons}.
    The activation pattern of the first layer $\bfa_1(\cdot,\theta) \in \{0,1\}^{N_{1}}$ takes two distinct values on $B(x,\epsilon)$, denoted $\bfa^+$ and $\bfa^-$, with $(\bfa^+)_\nu = 1$ and $(\bfa^-)_\nu = 0$.
    The patterns $\bfa^+$ and $\bfa^-$ are identical elsewhere.
     Regarding $\bfa_2(\cdot,\theta) \in \{0,1\}^{N_{2}}$, it takes a constant value $\bfa_2$ on $B(x,\epsilon)$ by the very definition of this ball for the considered isolated neuron $\nu$.
    By definition of $\bfa^+, \bfa^{-}$, there are $x^+,x^{-} \in B(x,\epsilon)$ such that $\bfa(x^+,\theta) = \bfa^+$ and $\bfa(x^-,\theta) = \bfa^-$ and $f$ is differentiable at $x^+$ and $x^-$. Denoting $u^+ = \nabla f_{\theta}(x^+)$ and $u^- = \nabla f_{\theta}(x^-)$
we have
    \begin{align*}
        (u^+ - u^-)^\top & = w_3^\top \diag(\bfa_2) W_2 \diag(\bfa^+ - \bfa^-) W_{1} \\
        & =   (w_3^\top \diag(\bfa_2) W_2) [\nu] W_1 [\nu,:] \ .
    \end{align*}
    Since $(\bfa^+)_\nu = 1$  (resp. $(\bfa^-)_\nu = 0$), we have $z_\nu(x^+,\theta) > 0 $ (resp. $z_\nu(x^-,\theta) \leq 0 $), i.e.
    \begin{align*}
        (W_1 x^+  + b_1)[\nu]
        > 0\ , \qquad
        (W_1 x^-  + b_1)[\nu]
      \leq
        0 \ ,
    \end{align*}
    from which we get $W_1[\nu, : ] (x^+ - x^-) > 0$ and
    \begin{align*}
        &\langle u^+ - u^-, x^+ - x^- \rangle   =  \\
        & \quad \quad \left( w_3^\top \diag(\bfa_2) W_2 \right) [\nu] \underbrace{ W_1[\nu, : ] (x^+ -x^-) }_{> 0} \ .
    \end{align*}
    Since $f_\theta$ is convex, by monotonicity of its gradient (cf \Cref{eqn:monotonicity}) it holds $ \langle u^+ - u^-, x^+ - x^- \rangle \geq 0$, from which we get $(w_3^\top \diag(\bfa_2) W_2) [\nu] \geq 0$, i.e.
    \begin{equation}
        \label{ineq:2_layers_path_pos}
       (\bfa_2)^\top ( w_3 \odot W_2 [:,\nu] )\geq 0 \ .
    \end{equation}
    Since \eqref{ineq:2_layers_path_pos} must hold for every $x \in \fX_\nu$, i.e. for any $\bfa \in \bfa_2(\fX_\nu)$, we recover \Cref{eq:pos_cond_2_layer}.
\end{proof}

\subsection{Comparison to ICNN condition}

\label{subsec:2HL_ICNN}
We have already seen in \Cref{prop:ICNN_counter_ex}
a convex CPWL implemented by a non-ICNN two-hidden-layer
 $\relu$ network.
As a sanity check, and a hands-on exercise to manipulate~\eqref{eq:pos_cond_2_layer}, let us verify that this
network  {\em does satisfy} the non-negativity condition of \Cref{lemma:pos_cond_2_layer}.

\begin{example}
    Recall the parametrization of $f_\ex$, which is a convex CPWL function:
    \begin{align*}
        w_3 = \begin{pmatrix}
            1
            \\ 1
        \end{pmatrix}, \quad  W_2 = \begin{pmatrix*}[r]
            -1 & 1 \\
            2 & 1
        \end{pmatrix*}, \quad W_1 & = \mathrm{Id}_2,\\
       b_3 = 0, \quad b_2 = \begin{pmatrix*}[l]
        -1 \\
        -0.5
    \end{pmatrix*}, \quad b_1 & = \mathbf{0}_2 .
    \end{align*}
    First, all neurons are isolated, see \Cref{app:counter_ex_isolated} for a proof.
    Second, $w_3 \geq 0$ satisfies the non-negativity constraint on the last layer.
    We denote $\mu_1$ and $\mu_2$ the two neurons of the first layer $N_1$ with their respective pre-activations given by $z_{\mu_1} (x) = \relu(x_{1}) $ and $z_{\mu_2} (x) = \relu(x_{2}) $ where $x = (x_{1},x_{2})^{\top}$.
    Checking  \eqref{eq:pos_cond_2_layer} for $\nu=\mu_{2}$ is trivial since $w_{3} \odot W_2[:, \mu_2] = (1 \  1)^\top$ is non-negative as well as {\em any} activation $\bfa \in \{0,1\}^{N_{2}}$.
    Since $w_{3} \odot W_2[:, \mu_1]  = (-1 \  2)^\top$,  \eqref{eq:pos_cond_2_layer}  holds for $\nu=\mu_{1}$ if, and only if,
    \begin{equation}\label{eq:ExampleConditionOnActivation}
        \min_{\bfa \in \bfa_2(\fX_{\mu_1})} -\bfa[1]+2\bfa[2] \geq 0\ .
    \end{equation}
    As $\bfa_2(\fX_{\mu_1}) = \{ (0,0), (0,1), (1,1) \}$ (see proof in~\Cref{app:counter_ex_isolated}), condition \eqref{eq:ExampleConditionOnActivation} indeed holds.

It is natural to wonder whether the necessary convexity condition of \Cref{lemma:pos_cond_2_layer} for two-hidden-layer standard MLPs extends to two-hidden-layer SkipMLPs, and whether it is also sufficient. Such questions are directly investigated using the more general path-lifting framework with DAG $\relu$ networks in the next section.

\end{example}

 \section{DAG $\relu$ networks}
\label{section:dag}

The two-hidden-layer case has highlighted the appearance of {\em products of layer weights} and their {\em inner product with activation patterns} in the exposed necessary conditions for convexity.
Such structures also appear naturally in the so-called path-lifting framework introduced by \citet{gonon2023path} which encompasses very general deep $\relu$ architectures including skip connections, max-pooling, or convolutional layers.
After recalling the needed definitions, we show that they enable a general characterization of necessary convexity conditions in this framework.

For clarity we restrict the exposition in the main text to DAG ReLU networks \emph{without max-pooling}.
We discuss in \Cref{app:maxpool} adaptations of the definitions (of path activations, and of isolated neurons) that yield the full extension to the general framework of \citet{gonon2023path}.

\subsection{The DAG $\relu$ model}
Consider a $\relu$ network with parameters $\theta$ which can be described as a Directed Acyclic Graph (DAG): each vertex represents a neuron with weight $\theta_\nu$ (the bias), each oriented edge connects neurons $\nu$ and $\mu$ with a parameter $\theta^{\nu \to  \mu}$.
The DAG, denoted $G$, can be described by the tuple $(N,E)$ made of respectively neurons $\nu \in N$ (i.e. vertices of the graph) and directed edges $(\mu,\nu) \in E \subseteq N \times N$.

To each neuron $\nu \in N$ we associate, with a slight abuse of notation, a function $\nu: \bbR^{d} \to \bbR$ called its {\em post-activation}.
For input $\nu \in \Nin$ we have $\nu(x) := x_{\nu}$ (there is no nonlinearity); the {\em pre-activation} of hidden neurons and output neurons  $\nu \in N \setminus \Nin$ is recursively given by an affine combination of the post-activation of its antecedents in the DAG, $z_\nu(x) = \sum_{\mu : (\mu, \nu ) \in E} \theta^{\mu \to \nu} \mu(x) + \theta_\nu$. The post-activation of hidden neurons is given by applying the $\relu$ to the corresponding pre-activation $\nu(x) = \relu(z_{\nu}(x))$. For output neurons $\nu \in \Nout$, we usually have $\nu(x) = z_{\nu}(x)$ (as for input neuron, there is no nonlinearity).
The formalism below however also covers the case where a $\relu$ is applied at the output.

We are now ready to define the path-lifting $\Phi$ and path-activations~$A$.

\begin{definition}[Path-lifting and path activations, \cite{gonon2023path}]
    \label{def:path-lifting}
    Consider a $\relu$ network described by a DAG $G$ and parametrized by $\theta$, with all hidden neurons equipped by a $\relu$ activation.
\begin{itemize}[leftmargin=*]
    \item \textbf{Paths} A \emph{path} $p$ is a sequence of neurons $(p_{0},\ldots,p_{m})$ (of arbitrary path length $m \geq 0$) connected in the graph $G$.
    We denote it as $p = p_0 \to \dots \to p_m$.
    The set $\cP$ denotes the collection of all paths {\em ending at the output neuron}\footnote{Given our interest in scalar-valued function, we consider only one output neuron for simplicity, but all the definitions generalize easily (see \cite{gonon2023path}) to several output neurons.} of $G$ (note that it includes paths which start at hidden or output neurons).
    \item \textbf{Path-lifting} For each path $p$ we define $\phi_p(\theta)$ as a product of parameters, depending on the nature of the first neuron $p_0 \in p$:
    \begin{equation}\label{eq:path_lift}
        \phi_p(\theta) =
        \begin{cases}
            \prod_{i=1}^m \theta^{p_{i-1} \to p_i} &\text{ if } p_0 \in \Nin \\
            b_{p_0}\prod_{i=1}^m \theta^{p_{i-1} \to p_i} &\text{ if } p_0 \in H \\
            b_{p_0} &\text{ if } p_0 \in \Nout
        \end{cases}
    \end{equation}
    The path-lifting is $\Phi(\theta) := ( \phi_p(\theta) )_{p \in \cP}  \in \bbR^{\cP}$.
\item \textbf{Path-activations} The activation of a path $p$ is
    \begin{equation}\label{eq:DefActPath}
        a_p(x,\theta) = \prod_{i=0
        }^m
         a_{p_i} (x,\theta) \in \{ 0, 1 \}
    \end{equation}
where by convention $a_{\nu}(x,\theta)\equiv 1$ for every linear neuron $\nu$ (all input neurons, and usually also the output neuron) In other words, the activation of the path is $1$ if all hidden neurons on it are active, and $0$ otherwise.
    We define a path-activation vector $a(x,\theta) := (a_{p}(x,\theta))_{p \in \cP}  \in \{0,1\}^{\cP}$ and the \emph{path-activations} matrix $A(x, \theta) \in \bbR^{\cP \times (\Nin +1)} $ as the matrix with entries
    \begin{equation}\label{eq:defpathactivation}
            A_{p, \nu}(x,\theta) = \begin{cases} a_p(x, \theta) \mathbf 1_{p_0 = \nu}  \quad \text{if } \nu \in \Nin, \\
            a_p(x, \theta) \mathbf 1_{p_0 \notin \Nin} \quad \text{if }\ \nu \notin \Nin.
            \end{cases}
    \end{equation}
   In the path-activations matrix, the columns indexed by $\Nin$ collect the activation of paths in $\cP$ that start from input neurons, while the last column gathers the activations starting from {\em hidden} neurons (equipped with biases).
   Just as the neuron-wise functions $a_\nu(\cdot,\theta)$, the matrix-valued function $A(\cdot,\theta)$ is piecewise constant.
    \end{itemize}
\end{definition}

Regarding our work, the main interest of these objects is that they enable to write the network output as a scalar product between $\Phi(\theta)$ and $A(x,\theta)$ \cite[Lemma A.1]{gonon2023path}:
\begin{equation}\label{eqn:f_from_path}
    f_\theta(x) = \left\langle \Phi(\theta), A(x,\theta) \begin{pmatrix}
        x \\ 1
    \end{pmatrix} \right\rangle.
\end{equation}
From this, the piecewise affine nature of $f_\theta$ is clear: at each point $x$ around which the activation patterns are locally constant, the network output is an affine function whose slope
depends on the path-lifting $\Phi(\theta)$ and path activation $A(x,\theta)$.

\subsection{Path-lifting factorization}

\begin{figure}
    \centering
    \includegraphics[width=\linewidth]{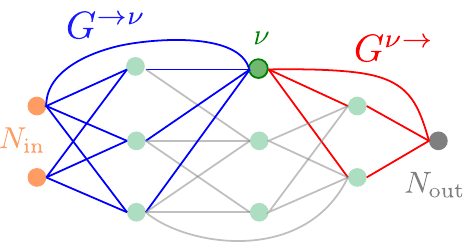}
    \caption{DAG $G$ and extracted subgraphs $G^{\to\nu}$ and $G^{\nu\to}$.}
    \label{fig:nn_dag_path}
\end{figure}
It will also be convenient to manipulate variants of the path-lifting and the path-activation matrix associated to certain sub-graphs of $G$ (and to the associated subset of parameters).

{\bf Subgraphs and their paths.}  Given a hidden neuron $\nu$, as illustrated on \Cref{fig:nn_dag_path}, we denote $G^{\to \nu}$
\addrebut{the (DAG) sub-graph deduced from $G$ by keeping input neurons, $\nu$ as a single output neuron, and all vertices and edges belonging to paths ending at $\nu$}.
\rmrebut{the largest (DAG) sub-graph of $G$ with the same input neurons $\Nin$ and with $\nu$ as the single output neuron.}

Likewise, $G^{\nu \to}$ is the  largest  (DAG) sub-graph with $\nu$ as its single input neuron and the same output neuron as $G$.

{\bf Path-lifting and activations on $G^{\to \nu}$.}
The pre-activation $z_{\nu}(x,\theta)$ is implemented by the restriction of the initial DAG $\relu$ network to $G^{\to \nu}$.
Instantiating the above notions yields $\cP^{\to \nu}$, the set of all paths of $G^{\to \nu}$ ending at \addrebut{the} \rmrebut{(the only)} output neuron $\nu$  \rmrebut{(by convention we equip it with a linear activation instead of its original $\relu$)}.
\addrebut{By convention, the output neuron is equipped with a linear activation instead of its original $\relu$.} \rmrebut{and with a slight abuse of notation (these functions depend on {\em a subset} of entries of $\theta$):}
Likewise, we define the path-lifting $\Phi^{\to \nu}(\theta)$; the path-activation vector $a^{\to \nu}(x,\theta)$; and the path-activation matrix $A^{\to \nu}(x,\theta)$. By~\eqref{eqn:f_from_path} we have
\begin{equation}\label{eqn:preact_from_path}
    z_{\nu}(x,\theta) = \left\langle \Phi^{\to \nu}(\theta),A^{\to \nu}(x,\theta)\begin{pmatrix}
    x \\ 1
    \end{pmatrix}\right\rangle.
\end{equation}

{\bf Modified definitions on $G^{\nu \to}$.}
By analogy with the above notions, and for later notational convenience, we denote $\cP^{\nu \to}$ the set of all paths on $G^{\nu\to}$ that {\em start at $\nu$ and end at the output neuron}.
The definition of the path-lifting $\Phi^{\nu\to}$ is:
\begin{align}
\Phi^{\nu \to}(\theta) & := \left(  \prod_{i=1}^m \theta^{p_{i-1} \to p_i} \right)_{p \in \cP^{\nu \to}} \in \bbR^{\cP^{\nu\to}}\ . \label{eq:DefPhiOut}
\end{align}
In contrast with the definitions of $\Phi$ and $\Phi^{\to \nu}$, $\nu$ plays here a role analog to that of an \emph{input} neuron in $G^{\nu \to}$, and we {\em do not include its bias} $\theta_\nu$ in the path-lifting.
The activation of $r =r_{0} \to \ldots \to r_{m} \in  \cP^{\nu \to}$ is defined with {\em a variant} of~\eqref{eq:DefActPath}
\begin{align}\label{eq:DefActPathOutEntry}
    a^{\nu \to}_{r}(x,\theta) & := \prod_{i=1}^m  a_{r_i} (x,\theta) \ ,
\end{align}
(contrary to $a$ and $a^{\to \nu}$, the input neuron $r_{0}=\nu$ is {\em not} part of the product) and we define, similarly to above, the path-activation vector
\begin{align}
a^{\nu \to}(x,\theta) & := (a^{\nu \to}_{r}(x,\theta))_{r \in \cP^{\nu \to}} \in \{0,1\}^{\cP^{\nu\to}},\label{eq:DefActPathOut}
\end{align}

We insist that the definition of $\cP^{\nu \to}$ and definitions \eqref{eq:DefPhiOut}-\eqref{eq:DefActPathOutEntry} introduced on
$G^{\nu \to}$ \emph{differ from the standard path-lifting} not only because they are restricted to paths that start at $\nu$, but also because the function $a^{\nu \to}(x,\theta)$ depends on $x \in \bbR^d$ which is the input vector feeding the \emph{initial network defined on $G$} but not the scalar entry $z_{\nu}$ feeding the input neuron $\nu$ of $G^{\nu \to}$.

When $x$ and $\theta$ are clear from context, we omit them for brevity.
We will exploit the following useful lemma.

\begin{lemma}\label{lemma:kronecker}
With the above notations, consider an arbitrary hidden neuron $\nu$ and any parameter $\theta$.
The restriction of $\Phi = \Phi(\theta) \in \bbR^{\cP}$ (resp. of $a = a(x,\theta) \in \{0,1\}^{\cP}$) to paths starting from an input neuron and ``containing'' $\nu$
 ({\em i.e.} such that $p_{i}=\nu$ for some $i$; this is denoted $p \ni \nu$) satisfies
    \begin{align}
        \label{eqn:kron_path}
        \Phi_{\left[ p: \substack{p \ni \nu \\ p_0 \in \Nin} \right]}
        =
        \Phi^{\to \nu }_{[ q:q_0 \in \Nin]}
        \otimes
        \Phi^{ \nu \to} \ ,
        \\
        \label{eqn:kron_pathactvec}
        a_{\left[ p:\substack{p \ni \nu \\ p_0 \in \Nin} \right]}
        =
        a_{\nu} \cdot \left(a^{\to \nu }_{[q:q_0 \in \Nin]}
        \otimes
        a^{ \nu \to}\right).
    \end{align}Finally,
    \begin{align}\label{eqn:kronmatrix}
        A_{\left[p:\substack{p \ni \nu \\ p_0 \in \Nin}, \Nin  \right]}
        &=
         a_{\nu} \cdot \left( A_{[q: q_0 \in \Nin,\Nin]}^{\to \nu}
         \otimes a^{\nu \to}
\right).
    \end{align}
\end{lemma}

The proof in~\Cref{app:proof_dag} uses that $\{ p \in \cP :p \ni \nu, p_0 \in \Nin \} $ is in bijection with $ \{ q \in \cP^{\to \nu} :q_0 \in \Nin \} \times \cP^{\nu \to}$.
While we choose to focus in this paper on {\em scalar} function with a single output neuron as
we study convexity, \Cref{lemma:kronecker} still holds for networks with several output neurons.

\subsection{Necessary convexity condition}
Thanks to the path-lifting and activation formalism, we can express a necessary convexity condition for generic DAG networks.
\begin{box_light}
    \begin{proposition}[\textbf{Necessary condition for convex DAG ReLU networks}]
        \label{prop:nec_relu_cvx}
        Consider a $\relu$ network described by a DAG and parametrized by $\theta$, which implements a \emph{convex} CPWL function $f_\theta : \bbR^d \to \bbR$.
        For every \emph{isolated} hidden neuron $\nu \in H$,
        it holds
        \begin{equation}
            \label{eqn:nec_relu}
            \min_{x \in \fX_\nu} \langle
            a^{\nu \to}(x,\theta), \Phi^{\nu \rightarrow} (\theta)
            \rangle \geq 0 \ .
        \end{equation}
    \end{proposition}
   \end{box_light}
\begin{remark}
To simplify the exposition we only introduced in the main text the path-lifting and path-activations for DAG ReLU networks \emph{in the special case where all hidden neurons are equipped with the $\relu$ activation}. As discussed in \Cref{app:maxpool}, these definitions can be extended to also cover max-pooling neurons, again by leveraging the framework of \citet{gonon2023path}. With a proper adaptation of the notion of isolated neuron and of the corresponding set $\fX_{\nu}$ (\Cref{def:variant_IsolatedMaxPool}),
\Cref{prop:nec_relu_cvx} remains valid.
 \end{remark}
    Introducing the {\em finite} set of possible activations of the paths $\cP^{\nu\to}$ from $\nu$ to the output neuron,
    \[
        \bfa^{\nu \to}(\fX_\nu):= \{a^{\nu \to}(x,\theta): x \in \fX_\nu\} \subseteq \{0,1\}^{\cP^{\nu\to}},
    \]
    we get an equivalent expression of~\eqref{eqn:nec_relu} which emphasizes its analogy with the special case of~\eqref{eq:pos_cond_2_layer}:
    \[
    \min_{\bfa \in \bfa^{\nu \to}(\fX_\nu)} \langle \bfa,\Phi^{\nu\to}(\theta)\rangle\ ,
    \]
where for two-hidden-layer feedforward networks, $\Phi^{\nu \to}(\theta)$ is exactly $(w_3 \odot W_2[:,\nu])$.

Before proceeding to the proof of~\Cref{prop:nec_relu_cvx}, let us highlight a particular consequence: a generalization of \Cref{lemma:last_layer} on the non-negativity of the last layer. The equivalent of the last layer in a DAG $G = (N,E)$ is the set of hidden neurons $\nu \in H$ that are {\em antecedents} of the output neurons. Here, since we focus on convex functions, there is a single  output neuron ($\Nout= \{\mu_\mathrm{out}\}$) and $\ant(\mu_\mathrm{out}) := \{\nu \in N: (\nu, \mu_\mathrm{out}) \in E \}$.

\begin{box_light}
    \begin{corollary}[\textbf{Necessary condition: non-negativity of the ``last layer''}]
        \label{lemma:last_layer_general}
    Isolated hidden neurons in the ``last layer'', i.e. $\nu \in H \cap \ant(\mu_\mathrm{out})$ such that $\fX_\nu \neq \emptyset$, must have non-negative outgoing weight $ \theta^{\nu \to \mu_\mathrm{out}} \geq 0$.
    \end{corollary}
\end{box_light}
\begin{remark}
The reader may notice that the above is only expressed for neurons in $\ant(\mu_{\mathrm{out}})$ that are {\em hidden} neurons: in the presence of input skip connections, there can also be {\em input neurons} in $\ant(\mu_{\mathrm{out}})$, but there is no non-negativity constraint on the corresponding weight.
\end{remark}

\begin{proof}
    When $\nu$ is in the ``last layer'' $G^{\nu \rightarrow}$ is reduced to a simple graph with two nodes: $\nu$, $\mu_\mathrm{out}$ connected via a single edge $\nu \to \mu_\mathrm{out}$ with weight $\theta^{\nu \to \mu_\mathrm{out}}$.
    The only path $p \in \cP^{\nu \to}$ starting from $\nu$  is $p=\nu \to \mu_\mathrm{out}$, hence $\Phi^{\nu \rightarrow} (\theta) =
    \phi_{p}(\theta) = \theta^{\nu \to \mu_\mathrm{out}}$ and $\bfa^{\nu \to}(x,\theta) = a_{p}(x,\theta) = 1$.
\end{proof}

\begin{remark}
    Informally, the non-negativity constraint~\eqref{eqn:nec_relu} reads as follows.
    For every isolated hidden neuron $\nu$ of the network, one has to check sign of scalar products between
    \begin{itemize}[label=--]
        \item the path-lifting for paths which begin at $\nu$ and end at the output neuron, and
        \item the activations of these paths at input points $x \in \bbR^{\Nin}$ at which $a_\nu$ is the only activation which switches, i.e. input points $x$ such that $z_\nu(x,\theta) = 0$ and $z_\mu(x,\theta) \neq 0$ for $\mu \neq \nu$.
    \end{itemize}
\end{remark}

The proof of \Cref{prop:nec_relu_cvx} relies on the following lemma.

\begin{lemma}[Local convexity criterion for $\relu$ networks]
    \label{lemma:convexity_criterion}
    Consider a $\relu$ network described by a DAG and parametrized by $\theta$ which implements a CPWL function $f_\theta: \bbR^d \to \bbR$.
    Consider an isolated neuron $\nu \in H$ and $x \in \fX_\nu \neq \emptyset$.
    There exists $\epsilon > 0$ such that
    \begin{enumerate}
    \item there are $x^+, x^- \in B(x, \epsilon)$ such that $f_\theta$
    is differentiable at $x^{+},x^{-}$ and $a_{\nu}(x^+,\theta)=1$,  $a_{\nu}(x^-,\theta)=0$;
    \item for any such pair $x^{+},x^{-}$ it holds
    \begin{align}
        \label{eq:monotoneconditiondag}
        & \langle \nabla f_{\theta}(x^+) - \nabla f_{\theta}(x^-),x^+ - x^-\rangle \geq 0  \nonumber \\
        \iff & \langle a^{\nu\to},\Phi^{\nu \to}(\theta)\rangle \geq 0 \ .
    \end{align}
    \end{enumerate}
\end{lemma}
The proof of this lemma in \Cref{app:maxpool}
should sound familiar as it is reminiscent of the proof of \Cref{lemma:last_layer,lemma:pos_cond_2_layer}.
Let us highlight that  \Cref{app:maxpool} actually proves an \emph{extension} of the lemma that  covers the case of DAG $\relu$ networks that may include max-pooling neurons.

\begin{proof}[Proof of \Cref{prop:nec_relu_cvx}]
By convexity of $f_\theta$, it holds for any points $x^+, x^-$ where $f_\theta$ is differentiable that $\langle u^+ - u^-, x^+ - x^- \rangle \geq 0 $ where $u^+ =\nabla f_\theta(x^+),u^-=\nabla f_\theta(x^-)$ denote the slopes of $f_\theta$ at $x^+, x^-$.
Combined with \Cref{lemma:convexity_criterion} this implies that for every isolated neuron we have $\langle a^{\nu\to},\Phi^{\nu \to}(\theta)\rangle \geq 0$.
    As this non-negativity condition must hold for every $x \in \fX_\nu$, we end up with the claimed
     result.
\end{proof}

 \section{Convex $\relu$ Networks: sufficient conditions}
\label{section:cvx_relu_sufficient}

As in the case of simple feedforward $\relu$ architectures with few layers, it is natural to wonder whether the necessary convexity condition of \Cref{prop:nec_relu_cvx} is in some sense also sufficient under mild ``non-degeneracy'' conditions.
This is the object of this section.

\subsection{Non-degeneracy assumption}

Consider a CPWL function $f_\theta : \bbR^d \to \bbR $ implemented by a DAG network, parametrized by $\theta$ and fix $(R_k)_{k=1}^K$ any compatible partition.
Recall \Cref{prop:cvx_cpwl}\ref{prop:cvx_cpwl_item5} on convexity of CPWL functions: it gives sufficient conditions that need to be checked along the set $\cF := \bigcup_{k \sim \ell} F_{k,\ell}$ of all frontiers of the CPWL function -- where the slope of the CPWL function \emph{may} change.
To study convexity, it is in fact sufficient to restrict ourselves to frontiers where the slope of the function $f_\theta$ {\em actually} changes, so we introduce
\begin{multline}
    \label{eq:NoDiffFront}
    \cF_{\text{no-diff}} :=  \left\{ x \in \cF: f_\theta \text{ is not affine on }  B(x,\epsilon), \right. \\ \left.  \forall \epsilon>0 \right\}
\end{multline}

On the other hand, the necessary conditions of \Cref{prop:nec_relu_cvx} involve the set $\bigcup_{\nu \in H} \fX_\nu$, the set of input points where exactly one hidden neuron  switches.
Note that if the slope of $f_\theta$ changes at input point $x$, then at least one neuron must switch at $x$.
The connection we need to easily establish sufficient conditions is to assume that for any input point on a frontier where the slope changes, \emph{one and only one neuron ``switches''}.

\begin{assumption}
    \label{assumption:relu_nn}
    Consider a $\relu$ network described by a DAG with parameters $\theta$ which implements a CPWL function $f_\theta : \bbR^d \to \bbR$ and consider a compatible polyhedral partition $(R_k)_{k=1}^K$ associated with this CPWL function.
    The assumption states
    \begin{equation}
        \label{eqn:assumption_1}
        \fX_\nu \neq \emptyset  \quad \forall \nu \in H
     \end{equation}
    and
    \begin{equation}
        \label{eqn:assumption_2}
       \cF_{\text{no-diff}} \subset \bigcup_{\nu \in H} \fX_\nu.
    \end{equation}
\end{assumption}

\begin{remark}
In words, the assumption requires that every hidden neuron in the network is isolated, and that every change of slope of $f_\theta$ at points belonging to a $d-1$ dimensional face between regions corresponds to a change of activation of a {\em single} neuron. \end{remark}

For one-hidden-layer networks \Cref{assumption:relu_nn} simply coincides with \Cref{assumption:1_layer_colin}.

\begin{lemma}
    \label{lemma:equiv_assumptions}
    For a one-hidden-layer $\relu$ network architecture, \Cref{assumption:relu_nn,assumption:1_layer_colin} are equivalent.
\end{lemma}

\begin{proof}
    Consider a function $f : \bbR^d \to \bbR$ in $\SMLPclass_d(\bfn)$ with $\bfn = (n)$ for some $n \in \bbN^*$.
    In this case, the set of hidden neurons is $H = N_1$.
    By \Cref{lem:nosisolatedcaract}, \Cref{assumption:1_layer_colin} is equivalent to the fact that each neuron $\nu \in N_{1}$ is isolated, which by \Cref{def:isolated_neurons} is equivalent to \eqref{eqn:assumption_1}.
    Thus,  \Cref{assumption:relu_nn} implies \Cref{assumption:1_layer_colin}.
    Conversely, assume that
    \Cref{assumption:1_layer_colin} holds. To recover  \Cref{assumption:relu_nn}, it suffices to show \eqref{eqn:assumption_2}.
    First, define the $0$-level sets of each neuron $\nu \in N_1$ as $H_{\nu} := \{x \in \bbR^d: \ z_\nu(x) = W_1[\nu,:]x + b_1[\nu] = 0 \}$.
    By the non-degeneracy assumption ($W_{1}[\nu,:] \neq 0$ for each $\nu \in N_{1}$) of \Cref{assumption:1_layer_colin}, these sets $H_\nu$ are affine hyperplanes, and by the non-colinearity assumption of \Cref{assumption:1_layer_colin}, they cannot coincide nor be parallel.
    Considering $x \in \cF_{\text{no-diff}}$ we wish to prove that  $x \in \cup_{\nu \in H} \fX_\nu$.
    By contradiction, assume that $x \not \in \cup_{\nu \in H} \fX_\nu$.
    Because of the non-differentiability of $f$ at $x$, there exists $\nu \in N_1$ such that $x \in H_{\nu}$.
    Because $x \not \in \fX_\nu$, another neuron switches: there exists $\mu \neq \nu \in N_1$ such that $z_\mu(x) = 0$.
    Hence $x \in H_\nu \cap H_\mu$, so $x$ belongs at least to four different regions which contradicts $x \in \cF$.
\end{proof}

\subsection{Characterization of convex $\relu$ DAG networks}

\begin{theorem}[NSC for convex $\relu$ networks]
    \label{thm:NCS_convex_DAG}
    Consider a $\relu$ network described by a DAG and parametrized by $\theta$ which implements a CPWL function $f_\theta : \bbR^d \to \bbR$.

    Under \Cref{assumption:relu_nn}, the function $f_\theta$ is convex if and only if for every hidden neuron $\nu$ of the network, it holds
    \begin{equation}
        \label{eqn:NCS_RELU}
        \min_{x \in \fX_\nu}
                    \langle a^{\nu \to}(x,\theta), \Phi^{\nu \to} (\theta)\rangle \geq 0.
    \end{equation}
This holds even with max-pooling neurons, with $\fX_{\nu}$, $a^{\nu \to}$, $\Phi^{\nu \to}$ as defined in \Cref{app:maxpool}.
\end{theorem}
\begin{proof}
    \textbf{Necessity.} Assume $f_\theta$ convex. By \Cref{assumption:relu_nn}, every hidden neuron $\nu$ is isolated (as $\fX_\nu \neq \emptyset$). We conclude with \Cref{prop:nec_relu_cvx}. \\
    \textbf{Sufficiency.} Assume that \eqref{eqn:NCS_RELU} holds for every hidden neuron (notice that, by  \Cref{assumption:relu_nn}, $\fX_\nu \neq \emptyset$).
    We will exploit the characterization of \Cref{prop:cvx_cpwl} \ref{prop:cvx_cpwl_item5}.
    For this, consider an arbitrary $x \in \cF$.
    If $f_\theta$ is affine for every small enough ball centered at $x$ then we directly get the expected inequality \eqref{eqn:local_cvx_ineq}.
    Otherwise we must have $x \in \cF_{\text{no-diff}}$.
    By \Cref{assumption:relu_nn}, there exists a hidden neuron $\nu$ such that $x \in \fX_{\nu}$. By \eqref{eqn:NCS_RELU}  we have $\langle a^{\nu \rightarrow} (x,\theta), \Phi^{\nu \rightarrow} (\theta) \rangle \geq 0$.
    By \Cref{lemma:convexity_criterion}, there exists $\epsilon > 0$ such that for all $x^+, x^- \in B(x,\epsilon) \setminus \cF_{\text{no-diff}}$, $f_\theta$ is differentiable, and with $u^{+}= \nabla f_{\theta}(x^{+})$,  $u^{-}= \nabla f_{\theta}(x^{-})$ we have
     $\langle u^+ - u^-, x^+ - x^- \rangle \geq 0 $.
  This yields again \eqref{eqn:local_cvx_ineq}, which therefore holds for any $x \in \cF$. We conclude using \Cref{prop:cvx_cpwl} \ref{prop:cvx_cpwl_item5}.
\end{proof}

\subsection{Genericity of \Cref{assumption:relu_nn}}

In this section, we comment on why \Cref{assumption:relu_nn} is not too restrictive.

Back to the one-hidden-layer case, \Cref{lemma:equiv_assumptions} shows that \Cref{assumption:1_layer_colin,assumption:relu_nn} are equivalent.
As already discussed in \Cref{sec:onehiddenlayer}, \Cref{assumption:1_layer_colin} is not restrictive in the sense that every function implemented by a network that does not satisfy it can also be  implemented by a network of the same depth and with fewer neurons which \emph{does} satisfy it (resulting in \Cref{prop:ICNNallyouneed}).

Still in the one-hidden-layer case, a stronger assumption, which implies \Cref{assumption:1_layer_colin} and yet is generic, is to assume that for each $\nu \in N_{1}$ the $0$-level set of $z_\nu$ is a hyperplane, and that the resulting collection of hyperplanes is in {\em general position}\footnote{Let $\nu_1, \dots, \nu_m \in N_1$ and $\cH := \{ H_{\nu_1}, \dots, H_{\nu_m} \}$ where $H_{\nu_i} := \{x \in \bbR^d : W_1[\nu_i, :] x + b_1[\nu_i] = 0\}$.
If $W_1[\nu_i,:]$ is a non-zero vector for all neurons $\nu_i$, then $\cH$ is a \emph{finite (affine) hyperplane arrangement}.
Such an arrangement is said to be in \emph{general position} \cite{stanley2004introduction} if:
\begin{enumerate}[label=(\roman*)]
    \item for $k \leq d$, every $k$-fold  intersection of hyperplanes has dimension $d-k$,
    \item for $k > d$, every $k$-fold  intersection of hyperplanes is empty.
\end{enumerate}}.
This is more restrictive than \Cref{assumption:1_layer_colin}: in $\bbR^2$, it prevents $3$ hyperplanes to intersect at the same point.
In particular, it does not allow null biases.

When looking at deeper \emph{standard MLPs}\footnote{{\em i.e.,} with no skip connection} $\relu$ networks, there exists  an extension of this general position property that applies to so-called \emph{bent} hyperplanes, i.e. to the $0$-level set of pre-activation of any hidden neuron $\nu \in H$.

\citet{grigsby2022transversality} introduce the {\em transversality} property, corresponding to every $0$-level set of the pre-activation of any neuron $\nu \in N$ either having dimension $d-1$ or being empty.
 \Citet{masden2022algorithmic} then introduced the \textit{supertransversality property} ensuring that intersection of $k$ bent hyperplanes has dimension $d-k$. These properties are shown to hold on almost every standard $\relu$ MLPs.
Because of the possibility of having hidden neurons whose activation never changes, these assumptions do no exactly match \Cref{assumption:relu_nn}.
Yet, since such neurons do not introduce new regions, we believe that the characterization of \Cref{thm:NCS_convex_DAG} would still hold under a slight modification of \Cref{assumption:relu_nn}
that would be a direct consequence of transversality and supertransversality.
Such a variant of \Cref{assumption:relu_nn} would hence be likely to hold for almost every $\relu$ network.

\subsection{Back to the 2D example}
We can leverage the sufficient conditions given by \Cref{thm:NCS_convex_DAG} to prove convexity of the 2D function $f_\ex$ of~\Cref{prop:ICNN_counter_ex}.
We have already mentioned (see proof in \Cref{app:counter_ex}) that all its hidden neurons are isolated and that at
each frontier point where $f_\ex$ is non-differentiable, there is a single neuron whose activation changes.
Hence, \Cref{assumption:relu_nn} holds.
To prove convexity of $f_\ex$, it is sufficient to check that for every neuron $\nu$ in $N_1 \cup N_2$, \Cref{eqn:NCS_RELU} holds.
Denoting $\nu_{\mathrm{out}}$ the output neuron,
for each neuron $\nu \in N_2$ of the second hidden layer the set of paths $\cP^{\nu \to} = \{ \nu \to \nu_{\mathrm{out}} \}$ is a singleton and for all $x \in \fX_{\nu}$,  $a^{\nu \to} (x,\theta):= (a_{\nu_{\mathrm{out}} }(x,\theta)) = (1)$ and $\phi^{\nu \to} = (\theta^{\nu \to \nu_{\mathrm{out}}}) = w_3[\nu]$.
As $w_3 = (1,1)^\top$, \Cref{eqn:NCS_RELU} (which simply reads $w_{3}[\nu] \geq 0$) is satisfied for all $\nu \in N_2$.
Now, consider $\mu \in N_1$ and denote $N_2 := \{ \nu_1 , \nu_2 \}$.
We have $\cP^{\mu \to} = \{ \mu \to \nu_1 \to \nu_{\mathrm{out}}, \mu \to \nu_2 \to \nu_{\mathrm{out}}  \}$, and for all $x \in \fX_{\mu}$,
\begin{equation*}
    a^{\mu \to}(x, \theta) = \begin{pmatrix}
        a_{\nu_1}(x,\theta) \underbrace{a_{\nu_{\mathrm{out}} }(x,\theta)}_{:=1} \\ a_{\nu_2}(x,\theta) \underbrace{a_{\nu_{\mathrm{out}} }(x,\theta)}_{:=1}
    \end{pmatrix} = \bfa_2 (x, \theta),
\end{equation*}
\begin{equation*}
    \phi^{\mu \to}(\theta) = \begin{pmatrix}
       \theta^{\mu \to \nu_1} \theta^{\nu_1 \to \nu_{\mathrm{out}}}  \\
       \theta^{\mu \to \nu_2} \theta^{\nu_2 \to \nu_{\mathrm{out}}}
    \end{pmatrix} = (w_3 \odot W_2[:,\mu]),
\end{equation*}
so \Cref{eqn:NCS_RELU} boils
down --as expected-- to the same necessary condition given in the two-hidden-layer case
\begin{equation}
    \min_{x \in \fX_\mu} \langle \bfa_2(x,\theta), w_3 \odot W_2[:,\mu] \rangle \geq 0 ,
\end{equation}
which we have already verified in \Cref{subsec:2HL_ICNN}.

 \section{Numerical check of convexity}
\label{sec:numerics}

Given an architecture $G = (N,E)$, the two main computational bottlenecks to check the convexity conditions of \Cref{thm:NCS_convex_DAG} for a $\relu$ network with architecture $G$ and parameters $\theta$ are:
\begin{enumerate}
    \item \textbf{Handling the dimension of the vectors $\phi^{\nu \to}$ and $a^{\nu \to}$} involved in the scalar product. Considering {\em feedforward} networks, the number of paths grows exponentially with the number of layers, and in general with the depth of the network, so this is {\em a priori} a daunting task.
    \item \textbf{Computing the set of activation patterns $a^{\nu \to} (\fX_\nu)$} for every hidden neuron $\nu \in H$.
    This requires identifying all input points $x \in \fX_\nu$ where only the activation of $\nu$ changes.
    In practice, it amounts to identify all frontiers and regions of the CPWL function, whose number grows exponentially with the number of neurons.
\end{enumerate}

The first bottleneck can be addressed as follows.
Given one hidden neuron $\nu \in H$ and one path-activation vector $\bfa \in \bfa^{\nu \to}(\fX_\nu)$, the constraint to be checked $\langle \bfa, \Phi^{\nu \to}(\theta) \rangle \geq 0$ is reminiscent of \Cref{eqn:f_from_path} which gives the network output as a scalar product between the path-activation matrix and the path-lifting.
More precisely, $\langle \bfa, \Phi^{\nu \to}(\theta) \rangle \geq 0$ is \emph{exactly} the output of a modified network with architecture $G^{\nu \to}$ ($\nu$ is its input neuron), {\em with a scalar input} equal to $1$,  {\em with the biases of all its hidden neurons set to zero}, and with activation values given by $\bfa$ ({i.e.}, unlike in standard $\relu$ networks, the activation here is {\em not} dependent on the modified network's input; this is further detailed below).
Given a hidden neuron $\nu \in H$, suppose we have already identified a finite subset $ \fX_\nu^{\mathrm{finite}} \subset \fX_\nu$ such that $\bfa^{\nu \to}(\fX_\nu^{\mathrm{finite}} )=\bfa^{\nu \to}(\fX_\nu)$
(remember that there is a finite number of activation patterns in $\bfa^{\nu \to}(\fX_\nu)$).
Then, one can circumvent the explicit computation of the vector $\Phi^{\nu \to}$ (of combinatorially high dimension) and the enumeration of all paths by proceeding as follows:
\begin{itemize}
    \item extract the architecture corresponding to the sub-graph $G^{\nu \to} := (N^{\nu \to}, E^{\nu \to})$;
    \item set all biases of the sub-network to zero, i.e. parametrize the architecture $G^{\nu \to}$ with $\tilde \theta$ such that
    \begin{align}
        \tilde \theta_{\mu} &:= 0 & \forall \mu \in N^{\nu \to} \\
        \tilde \theta^{\mu \to \mu'} & := \theta^{\mu \to \mu'} &  \forall (\mu, \mu') \in E^{\nu \to}
    \end{align}
    \item for each identified input point $x \in  \fX_\nu^{\mathrm{finite}}$, perform a forward pass on the {\em initial network} and store the activations neuron-wise, i.e. store $a_{\mu}:= a_{\mu}(\theta,x)$ for every $\mu \in N$;
    \item replace the $\relu$ activation\footnote{The adaptation to max-pooling neurons in the spirit of \Cref{app:maxpool} is straighforward.} of each hidden and output neuron $\mu \in N^{\nu \to}$ by a function that computes a simple product $z  \mapsto \mu(x):= a_\mu(x) z$,
    \item do a forward pass with scalar input equal to $1$ to obtain $\langle a^{\nu \to}(x, \theta), \Phi^{\nu \to}(\theta)\rangle$.
\end{itemize}

Regarding the second bottleneck, it requires identifying points in $\fX_\nu$: notice that any input point $x \in \fX_\nu$  is such that $z_\nu(x) = 0$ and $z_\mu(x) \neq 0 $ for $\mu \neq \nu$.
The problem of identifying bent hyperplanes (i.e. $0$-level sets of $z_\nu$) has already been previously studied \citep{rolnick2020reverse,humayun2023splinecam,berzins23relu}.
Notably, \citet{berzins23relu} provides an efficient algorithm which --given a {\em compact} input domain--  {\em exactly} extracts all intersections between bent hyperplanes.
This algorithm is able to extract $0-$faces and $1$-faces of the polyhedral partition of a \addrebut{feedforward fully-connected} $\relu$ network in a few seconds for architectures with up to $10$ layers and hundreds of neurons per layer, defining millions of regions.

\begin{remark}
While all the convexity conditions expressed in this paper are for convex functions $f: \bbR^{d} \to \bbR$, the framework can easily be adapted to study convexity on any non-empty convex domain $\Omega \subseteq \bbR^{d}$. This essentially requires replacing $\fX_{\nu}$ from \Cref{def:isolated_neurons} (and its variant for max-pooling from \Cref{app:maxpool}) and $\cF$ by their intersection with $\Omega$ everywhere.
\end{remark}

We build upon the framework of \cite{berzins23relu} to provide an algorithm which numerically checks convexity of a given $\relu$ network on a compact convex domain $\Omega \subseteq \bbR^{d}$: it checks if the necessary conditions from \Cref{thm:NCS_convex_DAG} are satisfied with $\fX_{\nu}$ replaced by its intersection with $\Omega$.
We detail in \Cref{app:algo} how to proceed to obtain the sets $a^{\nu \to}(\fX_\nu)$ for each hidden neuron $\nu \in H$.
Then, we apply the method described above to compute the scalar product $ \langle \bfa, \Phi^{\nu \to} \rangle$ and finally we test whether all these inner products are non-negative.

\subsection{Experiments}

In this section, we look at the probability of randomly sampling a convex $\relu$ network, for a given architecture, when drawing parameters $\theta$ from a Gaussian distribution. Convexity of the implemented function $f_\theta$ is numerically checked with the method described previously.
We consider standard MLPs with two-hidden-layer networks and input dimension equal to 2.
The width of each hidden layer varies from $2$ to $7$.

For each  of these architectures $\bfn = (n_1, n_2)$ with $n_1,n_2 \in \{2,\dots, 7\}$, we draw $10^4$ parameters $\theta$ (weights and biases) at random according to a standard Gaussian.
We compare the number of sampled convex functions  to the number of sampled ICNNs.
Note, the probability of drawing a random ICNN can be easily analytically computed: each weight in $W_2$ and $w_3$ having probability $1/2$ to be non-negative under standard Gaussian initialization, it follows
\begin{equation}
        \bbP \left(f_\theta \in \ICNNclass_2(\bfn) \right) = \frac{1}{2^{n_2(n_1 +1)}} .
\end{equation}
For a small architecture $\bfn = (2,2)$, $\bbP(f_\theta \in \ICNNclass_2(\bfn)) \approx 0.016$ which is to be compared with the experimental frequency of convex $\relu$ networks (including the ICNNs) which is $4$ times greater (\Cref{fig:nums_cvx}).
Increasing the number of neurons in the architecture further increase the ratio between the number of obtained convex $\relu$ networks and the number of obtained ICNNs: adding a single neuron more than double this ratio.
This suggests that for larger architectures, ICNNs are a very small fraction of all the convex $\relu$ networks implementable.
\addrebut{However, note that among these convex $\relu$ networks, some implemented functions may be also implementable by ICNNs with the same architecture.}

While our condition for convexity could allow one to analytically derive the probability of having a convex $\relu$ network (not necessarily ICNN), the combinatorial aspects of the exploitation of these conditions remain currently somewhat challenging. Yet, we anticipate that this should be possible in the $\bfn = (2,2)$ case for instance.

\begin{figure*}[tpb]
    \centering
    \includegraphics[width=\textwidth]{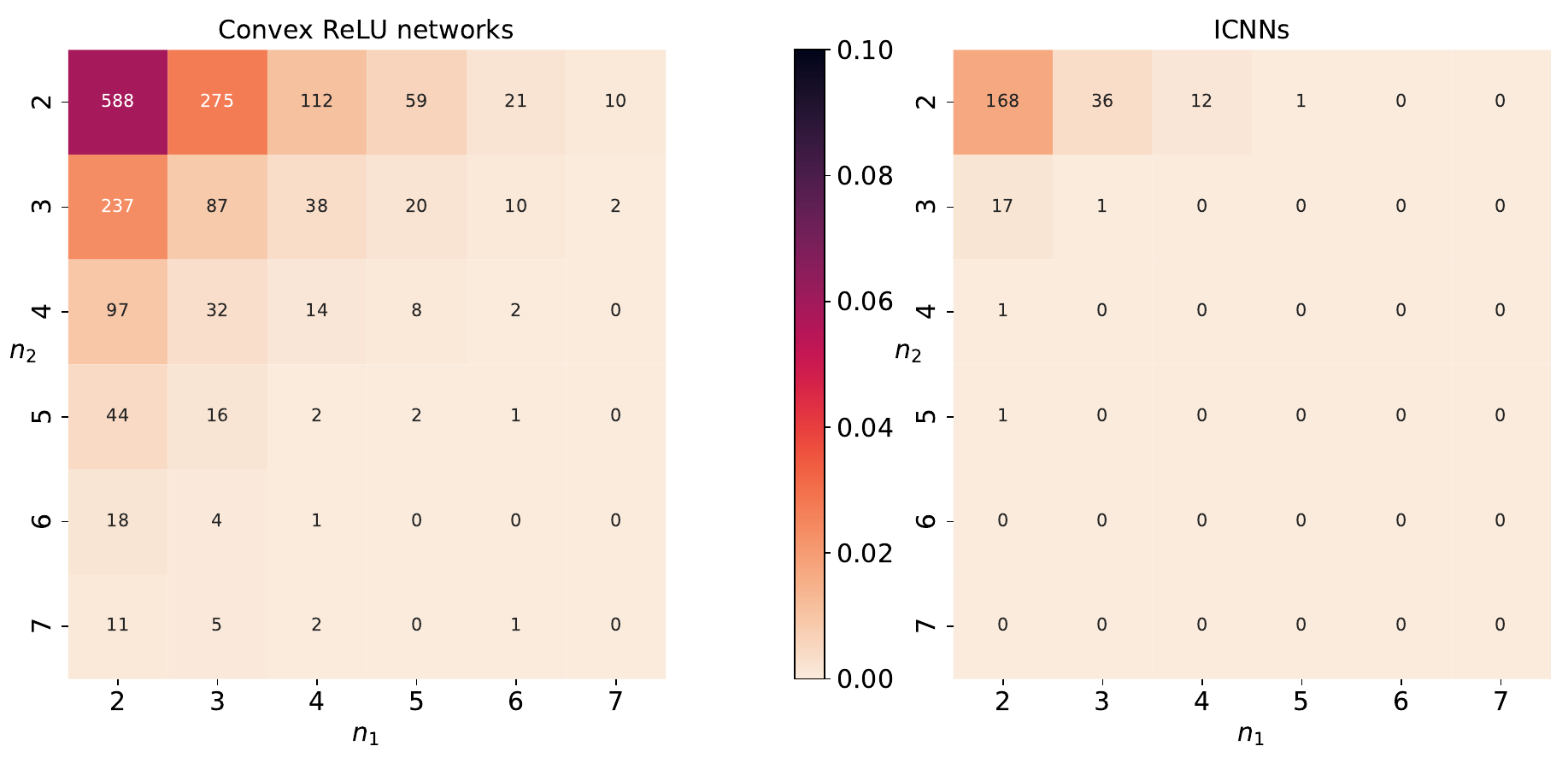}
    \caption{Number of convex $\relu$ networks among
    $10^{4}$ draws. (Left): convex $\relu$ networks. (Right): ICNNs.
     Architecture $\bfn = (n_1, n_2)$. \addrebut{The probability of sampling a convex function is higher than that of obtaining an ICNN, with a gap which becomes larger as the width of the layers increases.  It confirms that ICNN constraints do restrict the set of implementable convex functions at a given architecture.}}
    \label{fig:nums_cvx}
\end{figure*}

\section{Conclusion}

Through our exploration of convexity conditions for ReLU networks, we confirmed the intuition that there is life beyond ICNNs. Despite being the current standard to implement convex functions, ICNNs have intrinsic expressivity limitations starting from simple two-hidden-layer networks.
Thanks to the analysis of necessary and sufficient convexity conditions for the general class of DAG $\relu$ networks, we have highlighted that there are many parameterizations of such neural networks that yield convex functions, far beyond the restricted scope of ICNNs.

The expressed conditions allow to concretely {\em test} the convexity of a given network of moderate size, and allowed to concretely observe on small random two-hidden-layer networks the increased probability of drawing a convex function when using Gaussian parameters.
Being based on the non-negativity of a finite number of computable criteria, these conditions also naturally give rise to possible regularizers to promote convexity while training a network. One of the main challenges lying ahead is computational, with a particular focus on reducing as much as possible the number of convexity conditions to be tested. A possible avenue is to explore possible redundancies between these constraints as well as the graph structure of the network.

\section{Acknowledgements}
R. Gribonval is grateful to Antonin Chambolle for enlightening discussions on related topics during the SMAI-MODE 2024 conference. This project was supported in part by the AllegroAssai ANR project ANR-19-CHIA0009 and by the SHARP ANR project ANR-23-PEIA-0008 funded in the context of the France 2030 program.
A. Gagneux thanks the RT IASIS for supporting her work through the PROSSIMO grant.

\newpage
\bibliography{references.bib}

\onecolumn
\begin{appendices}
\crefalias{section}{appendix}
\section{\addrebut{Notations}}

\addrebut{We list below recurrent symbols in this paper along with their description.
\begin{description}[align=right,labelwidth=2cm]
    \item[$\mathbf n$]  \quad  tuple of widths of hidden layers
    \item[$\ICNNclass_d(\mathbf n)$]  \quad  set of convex functions $\bbR^d \to \bbR$ implemented by an ICNN architecture $\mathbf n$
    \item[$\SMLPclass_d(\mathbf n)$]  \quad set of functions $\bbR^d \to \bbR$ implemented by an MLP architecture $\mathbf n$ with
    \item[] \quad weighted input skip connections allowed
    \item[$\CvxCPwL_d$] \quad  set of convex functions $\bbR^d \to \bbR$
    \item[$N$] \quad set of all neurons of a network
    \item[$\Nin$] \quad  set of \emph{input} neurons of a network
    \item[$\nu \in N $] \quad a neuron
    \item[$\cF$] \quad  set of all frontiers points of the polyhedral partition of a CPWL function
    \item[$\cF_{\text{no-diff}}$] \quad  set of all points $x \in \cF$ such that the function is not differentiable at $x$
    \item[$\fX_\nu$] \quad input points at which only the activation of neuron $\nu$ switches
    \item[$G$] \quad a DAG describing the architecture of a neural network
    \item[$\cP$] \quad set of all paths ending at the output neuron of the DAG
    \item[$\Phi(\theta)$] \quad  path-lifting for DAG with parameters $\theta$, vector of dimension $|{\cP}|$
    \item[$A(x, \theta)$] \quad  path-activations for DAG with parameters $\theta$ at input point $x$, matrix of dimension
    \item[] \quad $|\cP| \times (|\Nin|+1)$
\end{description}
}
 \section{Background on CPWL functions}

\label{app:background_cpwl}
\begin{definition}[Affine function]
    A function $f : \Omega \to \bbR$ defined on a convex set $\Omega \subset \bbR^d$
    is \emph{affine} if and only if for every  $x , y \in \Omega$, $t \in [0,1]$, $f \left( (1- t) x + ty \right) = (1-t) f(x) + t f(y)$.
\end{definition}

\begin{proposition}[{\cite[Section 3.1.1]{boyd2004convex}}]
    A function $f : \Omega \to \bbR$ defined on a convex set $\Omega \subset \bbR^d$ is affine if and only if there exists $u \in \bbR^d$, $b \in \bbR$ such that, for every $x \in \Omega$,
    \begin{equation}
        f(x) = \langle u, x \rangle + b \ .
    \end{equation}
\end{proposition}

\begin{definition}[Polyhedron]
    \label{def:polyhedron}
    A polyhedron (also called a polyhedral set) $\Omega \subset \bbR^d$ is a non-empty intersection of finitely many closed half-spaces.
    A polytope is a polyhedron that is bounded (in $\bbR^d$).
\end{definition}

We recall below some useful terminology for polyhedra, see also \cite{Ziegler1995}.

\begin{definition}
    Consider a polyhedral set $P$.
    \begin{itemize}
        \item A linear inequality $a^\top x \leq b $ is \emph{valid} for $P$ if it is satisfied for all points $x \in P$.
        \item A set $F$ is called a \emph{face} of $P$ if $F = P \cap \{ x \in \bbR^d : a^\top x = b \} $ where $a^\top x \leq b $ is a valid inequality for $P$.
        A face of a polyhedron is also a polyhedron.
        \item The dimension of a polyhedron $P$ is the dimension of its affine hull $\aff(P)$.
        Likewise, the dimension of a face is the dimension of its affine hull: $\dim(F) = \dim(\aff(F))$.
        \item A face whose dimension  is $d-1$ is called a \emph{facet}.
        \item $F = \emptyset$ and $F = P$ are \emph{improper} faces.
        All other faces are \emph{proper} faces.
    \end{itemize}
\end{definition}

 \section{Background on convexity}
\label{appendix:background_cvx}

\begin{definition}[Line segment]
    The line segment between $x \in \bbR^d$ and $y \in \bbR^d$ is
    \begin{equation}\label{eq:DefLineSegment}
        [x,y] := \{ (1-t) x + t y : t \in [0,1] \} \ .
    \end{equation}
\end{definition}

\begin{definition}[Convex set]
    A set $\Omega \subset \bbR^d$ is convex if and only if it contains all its line segments, i.e. for every $x,y \in \Omega$, $[x,y] \subset \Omega$.
\end{definition}

\begin{definition}
    Let $f : \bbR^d \to (-\infty, +\infty]$ be a function.
    The \emph{domain} of $f$ is $\dom f := \{x \in \bbR^d : f(x) <  +\infty \}$.
    The function $f$ is \emph{proper} if $\dom f \neq \emptyset$.
\end{definition}

\begin{definition}[Subdifferential]
The subdifferential of a function $f : \bbR^d \to \bbR$ at $x$ is the (possibly empty) set of slopes of all affine minorants of $f$ that are exact at $x$:
\begin{equation}
    \partial f(x) := \{ u \in \bbR^d :\forall y \in \bbR^d,f(y) \geq f(x)+ \langle u,y- x \rangle \} \ .
\end{equation}
\end{definition}

\begin{fact}\label{fact:non_empti_subdif}
    A proper function $f : \bbR^d \to (- \infty, + \infty]$ is convex if and only if $\partial f(x) \neq \emptyset$ for all $x \in  \mathrm{ri}(\dom(f))$,
    where $\mathrm{ri}$ denotes the relative interior \citep[Def. 6.9]{bauschke2011convex}.
\end{fact}
\begin{proof}
	The implication $\Rightarrow$ can be found in~\cite[Thm. 23.4]{rockafellar2015convex}.
    The converse property $\Leftarrow$ can be found in \cite[Sec. 1]{moreau1963fonctionnelles}.
\end{proof}

\begin{fact}
    Let $f : \cI := [t_1, t_2] \to \bbR$ be a convex function where $t_1, t_2 \in \bbR,\ t_1< t_2$.
    Then, $f'_+(t_1)$ and $f'_-(t_2)$ exist (but may be equal to $- \infty$ or $+\infty$, respectively).
\end{fact}

\begin{proof}
    Recall that the  right derivative $f'_+(t_1)$ is defined as:
    \begin{equation}
        \lim_{\substack{t \to t_1^+ \\ t \in \cI}} \frac{f(t)- f(t_1)}{t-t_1} \ .
    \end{equation}
    From convexity of $f$, the function $t \mapsto \frac{f(t)- f(t_1)}{t-t_1}$ is non-decreasing, so it admits a limit from the right at $t_1$ (potentially equal to $-\infty$).
    The proof for $t_2$ is similar.
\end{proof}

\begin{fact}[{\cite[Section 2.C]{rockafellar2009variational}}]
    \label{fact:convex_line}
     $f : \bbR^d \to \bbR $ is convex if and only if it is convex along every line segment: for all $x, y \in \bbR^d$, $t \mapsto f(x + t(y - x))$ is convex on $[0,1]$.
\end{fact}

\begin{lemma}
    \label{lemma:piecewise_convexity_1D}
    Let $f : \cI = [t_{\min},t_{\max}] \to \bbR$ be continuous and piecewise convex, i.e. there exist an integer $K$ and $- \infty \leq t_{\min} = t_0 < t_1 < \ldots < t_K = t_{\max} \leq + \infty$ such that
      $f_{\mid [t_{i-1},t_{i}]}$
       is convex, for every $i \in \{1 , \dots, K
        \}$.
    The function $f$ is convex on $\cI$ if and only if, for each breakpoint $t_i, i \in \{1 , \dots, K-1 \}$, one has
    \begin{equation}
        f'_-(t_{i}) \leq f'_+(t_{i}) \ .
    \end{equation}
\end{lemma}

\begin{proof}
We prove the result when $K=2$. If $f$ has more than two pieces, we can iteratively gather groups of 2 consecutive pieces of $2$ and conclude by induction.

  Considering $f$ withonly two pieces, we denote $I_{1} = [t_{\min},t]$, $I_{2}=[t,t_{\max}]$ those pieces, where $t_{\min}<t<t_{\max}$.
    Define $f_{1},f_{2}: \bbR \to \left]-\infty, + \infty\right]$ as: \begin{align}
        f_1  : s \mapsto \begin{cases}
            f(s) & s \in I_1 \\
            + \infty & \text{otherwise }
        \end{cases} \qquad
        f_2 : s \mapsto \begin{cases}
            f(s) &  s \in I_2 \\
            + \infty & \text{otherwise }
        \end{cases}
    \end{align}
    The functions $f_1$ and $f_2$ are proper and convex on $\bbR$  because their epigraphs are convex sets.
    Hence, from \Cref{fact:non_empti_subdif}, we get that $\partial f(s) = \partial f_1(s)  \neq  \emptyset$ for all
    $s\in ]t_{\min},t[$ and $\partial f(s) = \partial f_2(s)  \neq  \emptyset$ for all $s \in ]t,t_{\max}[$.
    Then, using again \Cref{fact:non_empti_subdif}, we get that $f$ is convex on $[t_{\min},t_{\max}]$ if and only if
    $\partial f(s) \neq \emptyset$ for all $s\in ]t_{\min},t_{\max}[$. Since $\partial f(s) \neq \emptyset$ for every $s \in ]t_{\min},t[ \cup ]t,t_{\max} [$,
    the convexity condition is equivalent to $\partial f(t) \neq \emptyset$.
   Moreover we have that
    \begin{align}
        \partial f(t) & =  \bigcap_{t' \in I_1 \cup I_2} \{ u : f(t') \geq f(t) + u(t' -t) \} \\
        & = \left( \bigcap_{t' \in I_1} \{ u : f(t') \geq f(t) + u(t' -t) \} \right) \cap \left( \bigcap_{t' \in I_2} \{ u : f(t') \geq f(t) + u(t' -t) \}  \right) \\
        & = \partial f_1(t) \cap \partial f_2(t) \\
        & = \left[f'_-(t) , +\infty \right[ \cap \left]- \infty , f'_+(t)\right]
    \end{align}
    This is non-empty
 if and only if $f'_-(t) \leq f'_+(t)$.

\end{proof}

\begin{fact}[\cite{bauschke2011convex}, Proposition 8.16]
    \label{fact:pointwise_limit_cvx}
    Let $(f_n)_{n \in \bbN}$ be a sequence of convex functions from $\bbR^d$ to $\bbR$ such that $(f_n)_{n \in \bbN}$ is pointwise convergent. Then $\lim_n f_n$ is convex.
\end{fact}

 \section{Proof of \Cref{prop:cvx_cpwl}}
\label{appendix:proof_cvx_cpwl}

\begin{lemma}
    \label{lemma:ball_frontier}
    \addrebut{Consider a CPWL function $f :\bbR^d \to \bbR$, and any compatible partition $(R_k)_{k=1}^K$. Denote $\cF := \bigcup_{k \sim l} F_{k,l}$ the set of all frontiers points. Then, for every point on a frontier $x \in F_{k,l}$, there exists a ball $B(x,\epsilon)$ with $\epsilon > 0$ such that $B(x,\epsilon) \subset R_k \cup R_l$.}
\end{lemma}

\begin{proof}
    \addrebut{First, we describe these two polyhedra $R_k$ and $R_l$ (\Cref{def:polyhedron}) as the intersection of  closed halfspaces (with irredundant inequalities):
    \begin{align}
        R_k &:= \{ z \in \bbR^d : A^k z \leq b^k \}, \\
        R_l &:= \{ z \in \bbR^d : A^l z \leq b^l \}.
    \end{align}
    From the one-to-one correspondence between the inequalities representing the polyhedra and their facets, it implies that $\aff(R_k \cap R_l)$ can be described by $\{ z  \in \bbR^d : a^\top z = b \}$ where $a^\top z \leq b$ is both in $A^k z \leq b^k$ and $A^l z \leq b^l$.
    As $\intt(R_k) \cap \intt(R_l) = \emptyset$ from the definition of a polyhedral partition (\Cref{def:poluhedron_part}), we can rewrite $R_k$ and $R_l$ as
    \begin{align}
        R_k &= \{ z \in \bbR^d : \tilde A^k z \leq \tilde b^k, a^\top z \leq b \}, \\
        R_l &= \{ z \in \bbR^d : \tilde A^l z \leq \tilde b^l, a^\top z \geq b \}.
    \end{align}

    Now let $x \in F_{k,l}$ and recall that $F_{k,l} = \mathrm{relint}  (R_k \cap R_l)$ which is of dimension $d-1$. By definition of the relative interior, there exists $\epsilon' > 0$ such that $B_{k,l}:=B(x, \epsilon') \cap \aff(R_k \cap R_l) \subset R_k \cap R_l$.  It implies that for every direction $v \in \aff((R_k \cap R_l)-x)$, the segment $[x + \epsilon v, x - \epsilon v] \subset R_k \cup R_l$.

    We then consider an orthogonal direction $v \perp \aff((R_k \cap R_l)-x)$ that points towards $R_l$, i.e., $v = \alpha  a $ with $\alpha > 0$. Our objective is to show that there exists $\epsilon >0$ such that $x + \epsilon v \in \intt(R_l)$, i.e., $a^\top (x+ \epsilon v) > b$ and $\tilde A^l (x+ \epsilon v) < \tilde b^l$.
    The first inequality is direct as for all $\epsilon > 0$, $a^\top (x+ \epsilon v) = b + \epsilon \alpha \Vert a\Vert^2 > b$. It remains to check the existence of an $\epsilon''$ for which the second inequality  can be made strict. First, we have that  $\tilde A^k x < \tilde b^k$ and $\tilde A^l x < \tilde b^l$.  Indeed, from \cite[Lemma 2.9]{Ziegler1995}, if $x$ saturates another inequality, then all the points of $R_k \cap R_l$ saturate this inequality, which is either contradicting the irredundancy assumption or the dimension of $\aff (R_k \cap R_l)$. As such, we get that $\tilde A^l (x + \epsilon'' v) < \tilde b^l$ is true for any $\epsilon'' > 0$ such that
    \begin{equation}
        \epsilon'' < \min_{i : (\tilde a^l_i )^\top v > 0 } \frac{\tilde b^l_i -  (\tilde a^l_i )^\top x}{(\tilde a^l_i )^\top v },
    \end{equation}

    It follows that  $\mathrm{conv} (B_{k,l} \cup \{ x + \epsilon'' v \} )$ (convex hull) is a subset of $R_l$.
    With the same arguments, there exists $\epsilon''' > 0$ such that $\mathrm{conv} (B_{k,l} \cup \{ x - \epsilon''' v \} )$ is a subset of $R_k$. Combining these two results we get
    \begin{equation}
        C := \mathrm{conv} ( B_{k,l} ,\{ x + \epsilon'' v \}, \{ x - \epsilon''' v \} ) \subset R_k \cup R_l
    \end{equation}

    Finally, given that $\mathrm{dim}(B_{k,l})=d-1$ and that $v \perp \aff ((R_k \cap R_l)-x)$, we deduce that $\mathrm{dim}(C)=d$ and $\intt(C) \neq \emptyset$.
    Moreover,  since $x \in \intt(C)$, there exists $\epsilon > 0$ such that
    \begin{equation}
        B(x, \epsilon) \subset C \subset R_k \cup R_l.
    \end{equation}
    which concludes the proof.}
\end{proof}

\textbf{Scheme of the proof}
\begin{itemize}[label=-]
    \item \ref{prop:cvx_cpwl_item1} $\implies$ \ref{prop:cvx_cpwl_item2} $\implies$ \ref{prop:cvx_cpwl_item3}
     $\implies$ \ref{prop:cvx_cpwl_item4}
     $\implies$ \ref{prop:cvx_cpwl_item2}
     $\implies$ \ref{prop:cvx_cpwl_item1}
     \item \ref{prop:cvx_cpwl_item1} $\implies$ \ref{prop:cvx_cpwl_item5}
     \item \ref{prop:cvx_cpwl_item5} $\implies$ \ref{prop:cvx_cpwl_item3}
\end{itemize}

\begin{proof}
    \noindent \ref{prop:cvx_cpwl_item1} $\implies$ \ref{prop:cvx_cpwl_item2} Assume $f$ convex on $\bbR^d$.
    From \Cref{fact:convex_line}, $f$ is convex on every line, so \ref{prop:cvx_cpwl_item2} follows.

    \noindent \ref{prop:cvx_cpwl_item2} $\implies$ \ref{prop:cvx_cpwl_item3} is trivial.

    \noindent \ref{prop:cvx_cpwl_item3} $\implies$ \ref{prop:cvx_cpwl_item4}
    Consider two neighboring regions $R_k \sim R_\ell$.
    Choose $x \in F_{k,\ell}$, $v \in \bbR^d \setminus \Span(R_k \cap R_\ell - x)$ and $\epsilon >0$ from \ref{prop:cvx_cpwl_item3}. \addrebut{As it has been done in the proof of}  \Cref{lemma:ball_frontier}\addrebut{, and} since $v \notin \Span(R_k \cap R_\ell - x)$,
    we have that $[x - \epsilon v, x+ \epsilon v] \cap \intt(R_k) \neq \emptyset$ and $[x - \epsilon v, x+ \epsilon v] \cap \intt(R_\ell) \neq \emptyset$ so we can pick $x_k \in [x - \epsilon v, x+ \epsilon v] \cap \intt(R_k)$ and $x_\ell \in [x - \epsilon v, x+ \epsilon v] \cap \intt(R_\ell)$.
    Then, we rewrite $f(x_k)$ and $f(x_\ell)$ with their respective affine expression
    \begin{align}\label{eq:AffineExpressionsRegionsProof}
        f(x_k) & = \langle \nabla f(x_k), x_k \rangle + b_k \\
        f(x_\ell) & = \langle \nabla f(x_\ell), x_\ell \rangle + b_\ell
    \end{align}
    We define $\bar t \in (0,1)$ the scalar for which $x = (1-\bar t) x_k + \bar t x_\ell$.
    Then, by convexity of $f_{|[x_k, x_\ell]}$, we have
    \begin{itemize}
        \item for $t \in (0, \bar t]$,
        \begin{equation}
            \langle \nabla f(x_k), (1-t)  x_k + t x_\ell \rangle + b_k \leq (1-t) \left( \langle \nabla f(x_k), x_k \rangle + b_k \right) + t  \left( \langle \nabla f(x_\ell), x_\ell \rangle + b_\ell \right) \ ,
        \end{equation}
        hence,
        \begin{equation}
         b_k - b_\ell  \leq   \langle \nabla f(x_\ell)-\nabla f(x_k), x_\ell \rangle \ . \label{proof:cvx_statements_ineq_1}
        \end{equation}
        \item for $t \in [ \bar t, 1)$,
        \begin{equation}
            \langle \nabla f(x_\ell), (1-t)  x_k + t x_\ell \rangle + b_\ell \leq (1-t) \left( \langle \nabla f(x_k), x_k \rangle + b_k \right) + t  \left( \langle \nabla f(x_\ell), x_\ell \rangle + b_\ell \right)  \ .
        \end{equation}
        hence,
        \begin{equation}
            b_\ell - b_k \leq  - \langle \nabla f(x_\ell)-\nabla f(x_k), x_k \rangle  \ . \label{proof:cvx_statements_ineq_2}
        \end{equation}
    \end{itemize}
    By summing \eqref{proof:cvx_statements_ineq_1} and \eqref{proof:cvx_statements_ineq_2}, we recover $\langle \nabla f(x_\ell) - \nabla f(x_k), x_\ell-x_k\rangle \geq 0 $.

    \noindent \ref{prop:cvx_cpwl_item4} $\implies$ \ref{prop:cvx_cpwl_item2}:
    Take $x \in \mathcal{F}$ and $v \in \bbR^{d}$.
    By the definition of $\mathcal{F}$ there exist neighborhing regions $R_{k} \sim R_{\ell}$ such that $x \in F_{k,\ell}$.
     Moreoever, \addrebut{from} \Cref{lemma:ball_frontier}, there is $\epsilon > 0 $ such that $B(x, \epsilon) \subset R_k \cup R_\ell$. The existence of such a ball \rmrebut{comes from the definition of a frontier and} is illustrated in \Cref{fig:proof_cvx_cpwl_case_1}.
We now distinguish two cases.
In the case where $v \in \Span(R_k \cap R_\ell -x )$, then, for $\epsilon$ small enough,  $[x-\epsilon v, x+\epsilon  v] \subset F_{k,\ell}$ so $f_{|[x-\epsilon v, x+\epsilon  v] }$ is affine, thus convex, as claimed. We now treat the case where
     $v \in \bbR^d \setminus \Span(R_k \cap R_\ell - x)$.
    This implies that the tuple $(y_{k},y_{\ell}) := (x-\epsilon v, x+ \epsilon v)$ either belongs to $\intt(R_k) \times \intt(R_\ell)$ or to $\intt(R_\ell) \times \intt(R_k)$.
    Assume, without loss of generality, that $y_k \in \intt(R_k)$ and $y_\ell  \in \intt(R_\ell)$.
    From \ref{prop:cvx_cpwl_item4}, there exists $x_k \in \intt(R_k)$ and $x_\ell \in \intt(R_\ell)$ such that $ \langle \nabla f(x_k) - \nabla f(x_\ell), x_k - x_\ell \rangle \geq 0$. See \Cref{fig:proof_cvx_cpwl_case_1}.
    Let $\eta$ be a normal vector to $R_k \cap R_\ell$ which points from $R_k$ to $R_\ell$.
    As $\dim(\Span(R_{k} \cap R_{\ell}-x)) =  \dim(\aff(R_k \cap R_\ell)) = d-1$, we can complete $\eta$ with $\zeta_1, \dots, \zeta_{d-1} \in \Span(R_{k} \cap R_{\ell}-x)) $ to form a basis of $\bbR^d$.
    \begin{SCfigure}[1.0]
        \centering
        \includegraphics[width=0.4\textwidth]{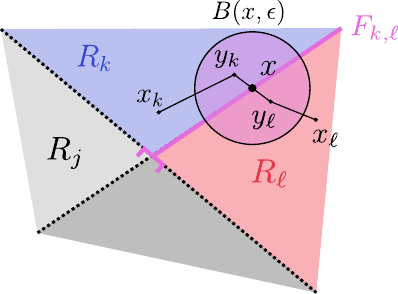}
        \caption{The regions $R_k$ and $R_\ell$ are neighbouring, but not $R_j$ and $R_\ell$.
        A frontier $F_{k,\ell}$ is the interior of the facet of dimension $d-1$ which separates two distinct neighbouring regions.
        Convexity is a local property that has to be studied on a ball around $x$.
       }
        \label{fig:proof_cvx_cpwl_case_1}
    \end{SCfigure}
    We write $x_k - x_\ell$ and $y_k - y_\ell$ in this basis:
    \begin{align}
        x_k - x_\ell &= |\alpha^x| \eta + \sum_{i=1}^{d-1} \beta^x_i \zeta_i \ , \\
        y_k - y_\ell &= |\alpha^y| \eta + \sum_{i=1}^{d-1} \beta^y_i \zeta_i \ .
    \end{align}
    where $|\alpha^x| > 0$, $|\alpha^y| > 0$ since $\eta$ points from $R_{k}$ to
    $R_{\ell}$ and the considered points are interior to the regions. By the affine expressions~\eqref{eq:AffineExpressionsRegionsProof} of $f$ on the regions $R_{k}, R_{\ell}$ and the continuity of $f$, we have
    \[
    x \in R_{k} \cap R_{\ell} \subseteq
    \{y : \langle \nabla f(x_k), y \rangle + b_k = \langle \nabla f(x_\ell), y \rangle + b_\ell \}
    =
    \{y : \langle \nabla f(x_k)-\nabla f(x_{\ell}), y \rangle = b_\ell-b_{k} \} \]
    hence $\Span(R_k \cap R_\ell-x) \subseteq \{y : \langle \nabla f(x_k)-\nabla f(x_{\ell}),
    y-x \rangle = 0 \} $. As a result
    $\langle \nabla f(x_k) - \nabla f(x_\ell), \zeta_i \rangle = 0$ for every $i \in \{1, \dots, d-1\}$ (the same holds for $\nabla f(y_k) - \nabla f(y_\ell)$ since $y_{k} \in \intt(R_{k})$ and $y_{\ell} \in \intt(R_{\ell})$).
    Then,
    \begin{equation}
        \langle \nabla f(x_k) - \nabla f(x_\ell) , y_k - y_\ell \rangle = \frac{|\alpha^y|}{|\alpha^x|} \langle \nabla f(x_k) - \nabla f(x_\ell), x_k - x_\ell \rangle \stackrel{ \ref{prop:cvx_cpwl_item4}}{\geq} 0 \ .
    \end{equation}
    Consider $t \in [0,1] \mapsto g(t) := f((1-t)y_{k}+ty_{\ell})$. Showing convexity of $f_{|[y_k, y_\ell]}$
    amounts to showing that $g$ is convex. We conclude using straightforward calculus and
    \Cref{lemma:piecewise_convexity_1D} (with $K=2$ pieces).

    \noindent \ref{prop:cvx_cpwl_item2} $\implies$ \ref{prop:cvx_cpwl_item1}:
    Consider any line segment in $\bbR^d$ parametrized as $[x, y]$ for some $x, y \in \bbR^d$.
    To show \ref{prop:cvx_cpwl_item1}, we need to prove that $f_{|[x,y]}$ is convex (\Cref{fact:convex_line}).
    We distinguish the case where this line segment $[x,y]$ only contains non-pathological points, i.e. which are either in the interior of a region (where $f$ is affine hence convex) or on a frontier (where \ref{prop:cvx_cpwl_item2} ensures
    convexity of $f_{\mid [z-\epsilon (y-x),z+\epsilon(y-x)]}$), from
    the case where it contains pathological points that belong to  more than two neighboring regions (see \Cref{fig:proof_cvx_cpwl_pathological}). These pathological points belong to affine subspaces which are faces of the convex regions with dimension at most $d-2$ (otherwise, $d-1$ faces of the convex regions are the so-called frontiers).
    \begin{SCfigure}[0.6]
        \centering
        \includegraphics[width=0.3\textwidth]{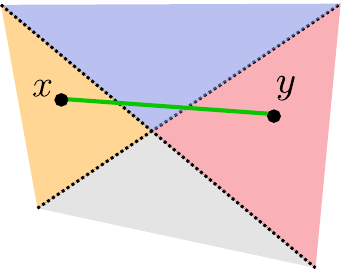}
        \hfill
        \includegraphics[width=0.3\textwidth]{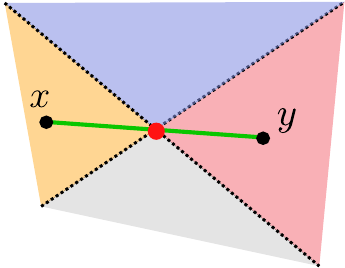}
        \caption{Left: the segment does not cross any pathological points. Right: the red point is pathological: it belongs to more than two neighboring regions.}
        \label{fig:proof_cvx_cpwl_pathological}
    \end{SCfigure}
    \begin{itemize}
        \item \textit{Case 1: $[x,y]$ only has non-pathological points.}
        More formally, we assume here that $[x,y]$ is such that, for every $z \in [x,y]$, either $z \in \intt(R)$ for
        some region $R$ from the considered compatible partition, either $z \in F$ for some frontier $F$ between two neighboring regions.
        According to \Cref{lemma:piecewise_convexity_1D}, it suffices to study convexity at breakpoints $(z_i)_{i \in I}$ belonging to $(F_i)_{i \in I}$,  frontiers between neighboring regions, using left and right derivatives.
        Then, \ref{prop:cvx_cpwl_item2}  gives convexity of $f_{|[z_i \pm \epsilon (y-x)]}$ for $\epsilon > 0 $ small enough.
        So, $f_{|[x,y]}$ is locally convex around each breakpoint implying that it is convex overall.

        \item \textit{Case 2: $[x,y]$ has pathological points.}
        First, pathological points belong to faces of the convex regions with dimension at most $d-2$.
        The number of such subspaces is finite due to the finite number of convex regions.
        We assume there exist pathological points $(z_i)_{i \in I}$ in $[x,y]$ which belong to affine subspaces $(A_i)_{i \in I}$ of dimension at most $d-2$.
        Note that either $[x,y]$ belongs to some face, making $I = [x,y]$ (with a slight abuse of notation), either it crosses a finite number of faces, with $I$ being finite.
        We rewrite each affine subspace $A_i = z_i + V_i$ where $V_i$ is the corresponding vectorial subspace and $z_i \in [x,y]$.
        The strategy is to find a direction $\delta$ so as to perturb the segment $[x,y]$ along $\delta$ and avoid pathological subspaces $A_i$.
        To do so, we consider $\tilde V_i = V_i + \mathrm{Vect}(y-x)$ which is of dimension at most $d-1$.
        Then, we take $\delta \in \bbR^d \setminus \cup_{i \in I} \tilde V_i$ ($\delta \neq 0$).
        Now, we built a new segment $s_t := [x + t\delta, y + t\delta]$ for $t>0$.
        We show that $s_t$ does not contain any of the initial pathological points of $[x,y]$.
        Consider $z \in s_t$.
        The following equivalences hold: $z \in A_i \iff \exists v_i \in V_i: z = z_i + v_i \iff \exists \epsilon, \epsilon' > 0 :  (1-\epsilon)x + \epsilon y + t \delta = (1-\epsilon')x + \epsilon' y + v_i \iff t\delta = v_i + (\epsilon' - \epsilon)(y-x)$.
        As $t\delta \not \in \tilde V_i$ for $t > 0$ and $ v_i + (\epsilon' - \epsilon) (y-x) \in \tilde V_i$, if $z \in s_t$ with $t >0$, then by contradiction $z \not \in A_i$.
        So, for every $t >0$, for every $i \in I$, $s_t \cap A_i = \emptyset$.
        Since the number of pathological regions is finite, there exists $\tau >0$ such that for every $0 <t < \tau$, $f_{|[x + t \delta, y + t \delta]}$ does not contain any pathological points: according to Case 1, $f_{|[x + t \delta, y + t \delta]}$  is convex.
        Then $f_{|[x,y]}$ is the pointwise limit, when $t \to 0$ of the convex functions $f_{|[x + t\delta, y+ t\delta]}$ so, by \Cref{fact:pointwise_limit_cvx} it is convex.
    \end{itemize}

    \noindent \ref{prop:cvx_cpwl_item1} $\implies$ \ref{prop:cvx_cpwl_item5}.
    Assume the function $f$ is convex and consider $x \in \mathcal{F}$, that is $x \in F_{k,\ell}$ for some neighboring regions $R_k \sim R_\ell$.
    By definition of a frontier, there exists
    a neighborhood $\mathcal{N} \subset R_k \cup R_\ell$. Take $x_1, x_2 \in \mathcal{N} \setminus \cF$.
    Then, $x_1 \in \intt(R_k) \cup \intt(R_\ell)$ (resp. $x_2$) so $f$ is differentiable at $x_1$ (resp. $x_2$): the inequality follows from the characterization of convexity~\citep[Example 20.3]{bauschke2011convex} (see also \Cref{eqn:monotonicity}).

    \noindent \ref{prop:cvx_cpwl_item5} $\implies$ \ref{prop:cvx_cpwl_item3}  Let $R_k \sim R_\ell$, $x \in F_{k,\ell}$, and  $v \in \bbR^d \setminus \mathrm{span} (R_k \cap R_\ell -x)$. Consider the neighborhood $\mathcal{N}$ given by \ref{prop:cvx_cpwl_item5}.
    Then there exists $\epsilon>0$ such that points $x \pm \epsilon v$ are in $\mathcal{N}$.
    Besides, $F_{k,\ell} \cap \mathcal{N} \subset x + \mathrm{span}(R_k \cap R_\ell - x )$ so $x \pm \epsilon v \notin F_{k,\ell}$. Then, the inequality in \ref{prop:cvx_cpwl_item5} implies increasing partial derivatives of $f_{|[x-\epsilon v, x+\epsilon  v] }$ which proves its convexity (\Cref{lemma:piecewise_convexity_1D}).
\end{proof}
 \section{Proof of path-lifting factorization}
\label{app:proof_dag}

\begin{proof}[Proof of \Cref{lemma:kronecker}]
    Consider a path $p = p_0 \to \cdots \to p_{m} \in \cP$  {\em on the initial graph $G$} such that
    $p_{0} \in \Nin$ and $p_{i} = \nu$ for some $i$.
    Considering the paths $q:= p_{0} \to \cdots \to p_{i}$ and
    $r := p_{i} \to \cdots \to p_{m}$, one can easily check that $q \in \cP^{\to \nu}$, $r  \in \cP^{\nu \to}$. Moreover,
     by definition of the path-lifting $\phi_p(\theta)$, we have
    \[
    \Phi_{p}
    = \phi_p(\theta)
    = \phi_{q}(\theta) \phi_{r}(\theta)
    = \Phi^{\to \nu}_{q} \Phi^{\nu \to}_{r}.
    \]
    To prove a similar property for the activations, first observe  that since $\nu$ is an output neuron of $G^{\to \nu}$, its activation {\em on this DAG $\relu$ network} (but not necessarily on the original one with $G$) is $a^{\to \nu}_{\nu}(x,\theta)=1$ for every $x$, $\theta$, by convention, while for every other neuron $\mu$ in $G^{\to \nu}$, $a^{\to \nu}_{\mu}(x,\theta) = a_{\mu}(x,\theta)$.
    As a consequence
      \begin{align}
         a^{\to \nu}_{q} &
         \stackrel{\eqref{eq:DefActPath}}{=} \prod_{j=0}^{i} a_{q_j}^{\to \nu}
          = \prod_{j=0}^{i-1} a_{q_j}\label{eq:PropActPathIn}\\
        \text{so that}\quad   a_p &\stackrel{\eqref{eq:DefActPath}}{=}\prod_{j=0}^{m} a_{p_j} \notag\\
            & =
             \left(\prod_{j=0}^{i-1} a_{p_j} \right)
             a_{\nu}
             \left(\prod_{j=i+1}^{m} a_{p_j} \right) \notag\\
             &  \stackrel{\eqref{eq:DefActPathOutEntry} \& \eqref{eq:PropActPathIn}}{=}  a_q^{\to \nu} a_{\nu}a^{\nu\to}_r
    \end{align}
    Since $\{p \in \cP: p \ni \nu, p_{0} \in \Nin\}$ is in bijection with $\{q \in \cP^{\to \nu}: q_{0} \in \Nin\} \times  \cP^{\nu \to}$,  \eqref{eqn:kron_path}--\eqref{eqn:kron_pathactvec} follow using the definition of a Kronecker product (adapted to arbitrary cartesian product of index sets).

    Finally, to establish~\eqref{eqn:kronmatrix}, explicit each entry of $A_{\left[p:\substack{p \ni \nu \\ p_0 \in \Nin}, \Nin \right]}$ from \eqref{eq:defpathactivation}
    \begin{align}
        A_{p, \mu} & = \begin{cases} 0 &  \text{ if } p_0 \neq \mu \\ a_p = a_q^{\to \nu} a_{\nu} a_r^{\nu \to}  &  \text{ otherwise } \end{cases} \label{eq:DefActMatFullEntry}
    \end{align}
    and similarly for $A_{ \{q: q_0 \in \Nin \},\Nin}^{\to \nu}$, again from \eqref{eq:defpathactivation}
    \begin{align}
        A^{\to \nu}_{q, \mu} = \begin{cases} 0 &  \text{ if } q_0 \neq \mu \\ a_q^{\to \nu} &  \text{ otherwise } \end{cases} \label{eq:DefActMatPathInEntry}
    \end{align}
    Then, denoting the input neurons $\Nin = \{ \mu_1 , \dots, \mu_d \}$ and enumerating $\{ q : q_0 \in \Nin \} = \{q^1, \dots, q^Q \}$,$\{ r :r \in \cP^{\nu \to} \} = \{r^1, \dots, r^R \}$ and $\{p: p \ni \nu, p_0 \in \Nin \} = \{p^1, \dots, p^{QR} \}$:
    \begin{align}
         A_{\left[p:\substack{p \ni \nu \\ p_0 \in \Nin}, \Nin \right]} & \stackrel{\eqref{eq:DefActMatFullEntry}}{=} \begin{pmatrix}
            \delta_{p^1_0 = \mu_1} a_{q^1}^{\to \nu} a_\nu a^{\nu \to}_{r^1} & \dots & \delta_{p^1_0 = \mu_d} a_{q^1}^{\to \nu} a_\nu a^{\nu \to}_{r^1} \\
            \vdots & \ddots & \vdots \\
            \delta_{p^R_0 = \mu_1} a_{q^1}^{\to \nu} a_\nu a^{\nu \to}_{r^R} & \dots & \delta_{p^R_0 = \mu_d} a_{q^1}^{\to \nu} a_\nu a^{\nu \to}_{r^R} \\
            \delta_{p^{R+1}_0 = \mu_1} a_{q^2}^{\to \nu} a_\nu a^{\nu \to}_{r^1} & \dots & \delta_{p^{R+1}_0 = \mu_d} a_{q^2}^{\to \nu} a_\nu a^{\nu \to}_{r^1} \\
            \vdots & \ddots & \vdots \\
            \delta_{p^{QR}_0 = \mu_1} a_{q^Q}^{\to \nu} a_\nu a^{\nu \to}_{r^R} & \dots & \delta_{p^{QR}_0 = \mu_d} a_{q^Q}^{\to \nu} a_\nu a^{\nu \to}_{r^R}
        \end{pmatrix} \\
        & =   \begin{pmatrix}
            \delta_{q^1_0 = \mu_1} a_{q^1}^{\to \nu} a_\nu a^{\nu \to}_{r^1} & \dots & \delta_{q^1_0 = \mu_d} a_{q^1}^{\to \nu} a_\nu a^{\nu \to}_{r^1} \\
            \vdots & \ddots & \vdots \\
            \delta_{q^1_0 = \mu_1} a_{q^1}^{\to \nu} a_\nu a^{\nu \to}_{r^R} & \dots & \delta_{q^1_0 = \mu_d} a_{q^1}^{\to \nu} a_\nu a^{\nu \to}_{r^R} \\
            \delta_{q^2_0 = \mu_1} a_{q^2}^{\to \nu} a_\nu a^{\nu \to}_{r^1} & \dots & \delta_{q^2_0 = \mu_d} a_{q^2}^{\to \nu} a_\nu a^{\nu \to}_{r^1} \\
            \vdots & \ddots & \vdots \\
            \delta_{q^Q_0 = \mu_1} a_{q^Q}^{\to \nu} a_\nu a^{\nu \to}_{r^R} & \dots & \delta_{q^{Q}_0 = \mu_d} a_{q^Q}^{\to \nu} a_\nu a^{\nu \to}_{r^R}
        \end{pmatrix} \\
        & =    \begin{pmatrix}
            \delta_{q^1_0 = \mu_1} a_{q^1}^{\to \nu} a_\nu a^{\nu \to} & \dots &  \delta_{q^1_0 = \mu_d} a_{q^1}^{\to \nu} a_\nu  a^{\nu \to} \\
            \vdots & \ddots & \vdots \\
            \delta_{q^Q_0 = \mu_1}a_{q^Q}^{\to \nu} a_\nu   a^{\nu \to} & \dots & \delta_{q^Q_0 = \mu_d} a_{q^Q}^{\to \nu} a_\nu  a^{\nu \to}
        \end{pmatrix} \\
        & \stackrel{\eqref{eq:DefActMatPathInEntry}}{=} a_\nu  \begin{pmatrix}
            A^{\to \nu}_{q^1, \mu_1} a^{\nu \to} & \dots &  A^{\to \nu}_{q^1, \mu_d} a^{\nu \to} \\
            \vdots & \ddots & \vdots \\
            A^{\to \nu}_{q^Q, \mu_1} a^{\nu \to} & \dots &  A^{\to \nu}_{q^Q, \mu_d} a^{\nu \to}
        \end{pmatrix} \\
        &= a_\nu \left( A^{\to \nu}_{[q : q_0 \in \Nin, \Nin]} \otimes a^{\nu \to} \right)
    \end{align}

\end{proof}

\section{Proof of  \Cref{lemma:convexity_criterion}}
\label{app:maxpool}

We begin with the proof for DAG $\relu$ networks \emph{where all hidden neurons are equipped with the $\relu$ activation function}. Then we  prove that \Cref{lemma:convexity_criterion} (and therefore \Cref{prop:nec_relu_cvx}, \Cref{thm:NCS_convex_DAG} and their consequences) extend to the full DAG framework of \citet{gonon2023path} which also contains pooling neurons in addition to $\relu$ neurons.
We however restrict such pooling to \emph{max-pooling} neurons, instead of the general ``$k$-max-pooling'' considered by \citet{gonon2023path}, since those also include \emph{min-pooling} neurons, whose relation to convexity is somewhat opposite to max-pooling neurons.

\subsection{The case of $\relu$ hidden neurons}
    Since $\nu \in H$ is isolated we can pick $x \in \fX_{\nu} \neq \emptyset$ and
    a ball $B(x, \epsilon)$ as in \Cref{def:isolated_neurons}.
   By definition of   $B(x, \epsilon)$, the functions $A^{\to \nu}(\cdot,\theta)$ and $a^{\nu\to}(\cdot,\theta)$ are constant on this ball (they only involve products with neuron activations $a_{\mu}(x,\theta)$, $\mu \neq \nu$), and we denote $A^{\to \nu}$, $a^{\nu\to}$ their respective values.
   For the same reason, $z_{\nu}(\cdot,\theta)$ is affine on $B(x,\epsilon)$ and $a_{\nu}(\cdot,\theta)$ takes two
   values.
   As a consequence, on this ball, the path-activation matrix $A(\cdot,\theta)$ takes exactly two values.
   This implies the existence of $x^{+},x^{-}$ in $B(x,\epsilon)$ such that $a_{\nu}(x^+,\theta)=1$,
   $a_{\nu}(x^-,\theta)=0$, and $A(\cdot,\theta)$ is locally constant around $x^{+}$ and $x^{-}$.
   As a consequence, $f_{\theta}$ is differentiable
   at $x^{+}$ and $x^{-}$ establishing the first point.
   Now consider \emph{any} two such points $x^+,x^- \in B(x, \epsilon)$ and denote
     $A^+ := A(x^+,\theta)$, $A^{-}:=A(x^-, \theta)$.
     Since $a_{\nu}(x^{+},\theta)=1$ and $a_{\nu}(x^{-},\theta)=0$ we have
    $z_\nu(x^+,\theta) > 0 $ and $z_{\nu}(x^-,\theta) \leq 0$.
 By \eqref{eqn:preact_from_path} we have (we omit the dependency of $\Phi^{\to \nu}(\theta)$ in $\theta$ for brevity)
    \begin{align}
        \left\langle \Phi^{\to \nu}, A^{\to \nu} \begin{pmatrix}
            x^+ \\1
        \end{pmatrix} \right\rangle =z_{\nu}(x^{+},\theta) > 0  \, \label{eqn:proof_cvx_relu_nn_pre_act_pos}\\
        \left\langle \Phi^{\to \nu}, A^{\to \nu} \begin{pmatrix}
            x^- \\1
        \end{pmatrix} \right\rangle =z_{\nu}(x^{-},\theta)  \leq
         0 \ . \label{eqn:proof_cvx_relu_nn_pre_act_neg}
    \end{align}

    Denote $u^+ =\nabla f_{\theta}(x^+) $ (resp. $u^- =\nabla f_{\theta}(x^-)$) the two slopes of $f_\theta$ given by $A^+$ (resp. $A^-$).  To get the slope of $f_\theta$ from the expression $\langle \Phi, A \ (x \ 1)^\top \rangle$ in \eqref{eqn:f_from_path}, we need to extract the block matrix corresponding to the paths which start with inputs neurons $N_{\mathrm{in}}$ (recall that the remaining rows/column
     of $A$ collect the biases of the network)
    \begin{equation*}
        (u^+ - u^-)^\top =  \Phi^{\top}_{[p:p_0 \in \Nin]} (A^+ - A^-)_{[p: p_0 \in \Nin,N_{\mathrm{in}}]}.
    \end{equation*}
    Because the only entries that change between $A^{+}$ and $A^{-}$ are the ones corresponding to paths $p \in \cP$ which contain $\nu$, it holds
    \begin{align*}
        (u^+ - u^-)^\top
        = \Phi^{\top}_{ \left[p: \substack{p \ni \nu \\ p_0 \in \Nin} \right]} (A^+ - A^-)_{ \left[ \substack{p \ni \nu \\ p_0 \in \Nin} , \Nin\right]} .
    \end{align*}
    Moreover, $a_p(x',\theta) = 0$ for all paths $p$ which contain $\nu$ and all $x'$ such that
    $z_{\nu}(x',\theta) \leq 0$, so $A^-_{[p:p \ni \nu,N_{\mathrm{in}}]} = \mathbf{0}$, leading to
    \begin{align}\label{eq:DAGslopediff}
        (u^+ - u^-)^\top
        = \Phi^{\top}_{ \left[p: \substack{p \ni \nu \\ p_0 \in \Nin} \right]} A^+_{\left[p: \substack{p \ni \nu \\ p_0 \in \Nin}, \Nin \right]} .
    \end{align}
    Subtracting \eqref{eqn:proof_cvx_relu_nn_pre_act_neg}
    to \eqref{eqn:proof_cvx_relu_nn_pre_act_pos} similarly yields: \begin{equation}\label{eqn:activ_sign}
       [\Phi^{\to \nu}_{[p: p_0 \in \Nin]}]^\top A^{\to \nu}_{[p: p_0 \in \Nin,N_{\mathrm{in}}]}
       (x^+ - x^-) >
       0.
    \end{equation}
By \Cref{lemma:kronecker}, using that $A^{\to \nu}(\cdot,\theta)=A^{\to \nu}$ and $a^{\nu\to}(\cdot,\theta)=a^{\nu\to}$ on $B(x,\epsilon)$, and $a_{\nu}(x^+,\theta)=1$ (cf~\eqref{eqn:proof_cvx_relu_nn_pre_act_pos}) the quantity $\Delta u  := (u^+ - u^-)^\top$
    from~\eqref{eq:DAGslopediff} rewrites as
    \begin{equation*}
        \begin{aligned}
            \Delta u
            & =
            \left( \Phi^{\to \nu}_{[q:q_0 \in \Nin]} \otimes \Phi^{\nu \to} \right)^\top \\
            & \qquad \times  \left( A^{\to \nu}_{[q:q_0 \in \Nin, \Nin]} \otimes
            a^{\nu \to}
            \right) \\
            & = \left( [\Phi^{\to \nu}_{[q:q_0 \in \Nin]}]^\top A^{\to \nu}_{[q:q_0 \in \Nin, \Nin]} \right) \\
            & \qquad \otimes \underbrace{\left( [\Phi^{\nu \to}]^\top
            a^{\nu \to}
            \right)}_{\text{scalar}} \\
            & = \left( [\Phi^{\nu \to}]^\top
            a^{\nu \to}\right) \\
            & \qquad \times  \left( [\Phi^{\to \nu}_{[q:q_0 \in \Nin]}]^\top A^{\to \nu}_{[q:q_0 \in \Nin, \Nin]}\right).
        \end{aligned}
    \end{equation*}
    It follows
    \begin{equation*}
        \begin{aligned}
       \langle u^+-u^-,x^+ - x^-\rangle  =
        \left( [\Phi^{\nu \to}]^{\top}
         a^{\nu \to}
        \right) \\
        \times \underbrace{\left( [\Phi^{\to \nu}_{[q:q_0 \in \Nin]}]^{\top} A^{\to \nu}_{[q:q_0 \in \Nin, \Nin]} \right)
         (x^+ - x^-)}_{> 0 \, \text{by \eqref{eqn:activ_sign}}}.
    \end{aligned}
\end{equation*}
As a result
\begin{align}\label{eq:monotoneconditiondag}
 \langle u^+-u^-,x^+ - x^-\rangle \geq 0 \nonumber \\
 \qquad \Longleftrightarrow \langle a^{\nu\to},\Phi^{\nu \to}(\theta)\rangle \geq 0.
\end{align}

\subsection{DAG $\relu$ networks with max-pooling activation function}

The max-pooling function
\begin{equation}
\maxpool(x) := \max_{i \in \{1, \dots, d\}} x_i
\end{equation}
returns the largest coordinate of $x \in \bbR^d$.
This section explains how to incorporate max-pooling activations in the DAG neural network model considered in this paper and recover the necessary conditions for convexity of the network given in \Cref{prop:nec_relu_cvx}.
Consider a network with architecture $G$ described by the tuple $(N,E)$ and parameters $\theta$.
The key is to judiciously modify the definition of neuron activation and preactivation, and to introduce \emph{edge} activations  \cite{gonon2023path}.
Consider a neuron $\nu$ with max-pooling activation and denote $\ant(\nu) := \{\mu \in N: (\mu, \nu) \in E \}$.
The post-activation of the neuron $\nu$ for a given input point $x \in \bbR^d$ writes as (there is no bias for such a neuron)
\begin{align}
    \nu(x) & :=  \max_{\mu \in \ant(\nu)} \theta^{\mu \to \nu} \mu(x) \\
    & = \langle (\theta^{\mu \to \nu} \mu(x))_{\mu \in \ant(\nu)}, e_{\mu^*} \rangle \qquad \text{with } \mu^* := \argmax_{\mu \in \ant(\nu)} \theta^{\mu \to \nu} \mu(x) \ , \label{eqn:max_pool_dirac}
\end{align}
where $e_{\mu^*}$ in \Cref{eqn:max_pool_dirac} is the binary vector with value $1$ for index
 $\mu = \mu^*$ and $0$ otherwise\footnote{we arbitrarily number the neurons in $\ant(\nu)$ so that
if the argmax contains more than one neuron, we systematically pick the one indexed by the smallest number}
.
It is then natural to define the activation of a max-pooling neuron $\nu$ as the binary
 \emph{vector} $\bfa_\nu(x, \theta) \in \{0,1\}^{\ant(\nu)}$ with
\begin{equation}
    [\bfa_\nu(x, \theta)]_{\mu} :=
    \begin{cases}
        1 &  \text{if } \mu = \argmax_{\mu' \in \ant(\nu)} \theta^{\mu' \to \nu} \mu'(x) \ , \\
        0 &  \text{otherwise,}
    \end{cases}
\end{equation}
and the pre-activation of $\nu$ as the {\em vector} $z_\nu(x) := (\theta^{\mu \to \nu} \mu(x))_{\mu \in \ant(\nu)} \in \bbR^{\ant(\nu)}$.
The post-activation of $\nu$ rewrites
\begin{equation}
    \nu(x) = \langle z_{\nu}(x), \bfa_\nu(x, \theta) \rangle \in \bbR \ .
\end{equation}
The same holds true with $\relu$ neurons with the convention $\bfa_{\nu}(\theta,x) := a_{\nu}(\theta,x) \in \{0,1\}$ for such neurons.

Following \citet{gonon2023path}, we also define the activation {\em of an edge} $(\mu, \nu) \in E$ as
\begin{equation}
    a_{\mu \to \nu}(x,\theta) := \begin{cases}
        a_\nu(x, \theta) & \text{if $\nu$ is a $\relu$ neuron} \\
        [a_\nu(x, \theta)]_\mu & \text{if $\nu$ is a $\maxpool$ neuron}\\
        1 & \text{if $\nu$ is a linear neuron}
    \end{cases} \in \{0,1 \} \ .
\end{equation}
These extended definitions enable us to re-define the activation of a path $p := p_0 \to p_1 \to \dots \to p_m$ as
$a_p(x, \theta) := a_{p_0}(x, \theta)
\prod_{i = 1}^m a_{p_{i-1} \to p_i}(x, \theta)$ (with the convention that $a_{p_{0}}(\theta,x)=1$ when $
p_{0}$ is a max-pooling neuron, an input neuron, or a linear neuron, see  \cite[Definition A.3]{gonon2023path}) : we describe the activation of a path using the edge-wise activations instead of the neuron-wise activations.
On the extracted sub-graph $G^{\nu \to}$, the path-activation of $r := r_0 \to \dots \to r_m \in \cP^{\nu \to}$ (where $r_{0}=\nu$ by definition of $\cP^{\nu \to}$) is defined as $a^{\nu \to}_r(x, \theta) := \prod_{i=1}^m a_{r_{i-1}\to r_i}(x,\theta)$ (\emph{without} multiplication by $a_{r_{0}}
(\theta,x)$).
Note that when only considering $\relu$ neurons, the modified definitions provided here yield exactly the same path-activations as the ones given in the core of the paper.
\begin{remark}
Remember that when $\nu$ is equipped with the $\relu$, considering the subnetwork with architecture
$G^{\to \nu}$ we replace the activation of $\nu$ by a linear activation since it is the output of this subnetwork. When $\nu$ is a $\maxpool$ neuron, we preserve its activation function. Whether the proof for the $\relu$ case can be adapted to fit this framework is left to future work.
\end{remark}
Besides, since we follow the framework from \cite{gonon2023path} which already includes max-pooling activations, all relevant results from their work -- {\em e.g.} \Cref{eqn:f_from_path} -- still hold in the framework of this section.

We provide below a (slightly) modified path-activations factorization which aligns with the edge-wise description of path-activations instead of the neuron-wise one.
\begin{lemma}[Path-activations factorization -- variant of \Cref{lemma:kronecker}]
    \label{lemma:variant_kron}
    Within the framework of this section,
    consider an arbitrary hidden neuron $\nu$ and any parameter $\theta$.
    The restrictions of $\Phi = \Phi(\theta) \in \bbR^\cP$ and $a = a(x,\theta) \in \{0,1\}^{\cP}$ to paths starting from an input neuron and ``containing'' an edge $\mu \to \nu$
     ({\em i.e.} such that $p_{i} = \mu$ and $p_{i+1} = \nu$ for some $i$; this is denoted $p \ni (\mu,\nu)$) satisfy
        \begin{align}
            \label{eqn:variant_kron_path}
            \Phi_{\left[ p: \ \substack{p \ni (\mu, \nu) \\ p_0 \in \Nin} \right]}
            = \theta^{\mu \to \nu} \cdot \left(
            \Phi^{\to \mu }_{[ q:q_0 \in \Nin]}
            \otimes
            \Phi^{ \nu \to} \right) \,
        \end{align}
        \begin{align}
            \label{eqn:variant_kron_pathactvec}
            a_{\left[ p:\substack{p \ni (\mu,\nu) \\ p_0 \in \Nin} \right]}
            = a_{\mu \to \nu} \cdot \left(a^{\to \mu }_{[q:q_0 \in \Nin]}
            \otimes
            a^{ \nu \to} \right) .
        \end{align}
    For the path-activation matrix, one has
        \begin{align}\label{eqn:variant_kronmatrix}
            A_{\left[p:\substack{p \ni (\mu, \nu) \\ p_0 \in \Nin} , \Nin\right] }
            &= a_{\mu \to \nu} \cdot \left(
             A_{[q: q_0 \in \Nin, \Nin]}^{\to \mu}
             \otimes a^{\nu \to} \right).
        \end{align}
\end{lemma}
\begin{proof}
    The proof is exactly the same as the one done in \Cref{app:proof_dag}.
    We only detail below the factorization of $a_p$ for a path $p = p_0 \to \dots \to p_m \in \cP$ such that  $p_0 \in \Nin$ and $p_i = \mu$, $p_{i+1} = \nu$ for some $i$.
    We denote $q:= p_0 \to \dots \to p_i \in \cP^{\to \mu}$  and $r:= p_{i+1} \to \dots \to p_m \in \cP^{\nu \to}$.
 \begin{align}
        a_q^{\to \mu} &
        = \underbrace{a^{\to \mu}_{p_{0}}}_{=1 \text{ by convention since } p_{0}\in \Nin}
         \prod_{j=1}^{i} a^{\to \mu}_{p_{j-1}\to a_{p_{j}}}
         =  \prod_{j=1}^{i} a_{p_{j-1}\to a_{p_{j}}}
        \ , \\
        a_r^{\nu \to} & =  a_{p_{i+1} \to p_{i+2}}  \cdots a_{p_{m-1} \to p_m}  \ ,
    \end{align}
    from which we have $a_p = a_{q}^{\to \mu} a_{\mu \to \nu} a_r^{\nu \to}$.
\end{proof}

Now, we are almost equipped to prove a variant of \Cref{prop:nec_relu_cvx} for networks which include max-pooling activations.
First, we need to redefine what an isolated neuron  when considering max-pooling activations.
Recall that for a $\relu$ activation, the set $\fX_\nu \subset \bbR^d$ is defined as input points
for which there exists a neighborhood $\mathcal{N}$ over which only the activation of $\nu$ changes.
In other words, this means that over $\cN$, $a_\nu$ takes {\em exactly two values}: $0$ and $1$, while $a_\mu$ for $\mu \neq \nu$ is constant.

This motivates the following definition of isolated neurons, which handles the maxpool case while being equivalent to \Cref{def:isolated_neurons} in the case of $\relu$ neurons.

\begin{definition}[Isolated neurons -- variant of \Cref{def:isolated_neurons}]
\label{def:variant_IsolatedMaxPool}
Given a parameter $\theta$, for each hidden   neuron $\nu  \in H:=\cup_{\ell=1}^{L-1}N_{\ell}$, we define $\fX_\nu \subset \bbR^d$ as the set of input points for which in every small enough neighborhood, \emph{only} the activation of neuron $\nu$
changes, and takes exactly two distinct values:
    \begin{equation}
        \fX_\nu := \Bigg\{x \in \bbR^d:  \exists \epsilon_0>0\ : \forall \, 0<\epsilon \leq \epsilon_0,
  \begin{cases}
            \bfa_\mu(\cdot,\theta) \text{ is constant on } B(x, \epsilon) & \forall \mu \neq \nu \\
            \bfa_\nu(\cdot,\theta) \text{ takes exactly $2$ values on } B(x, \epsilon)  &
        \end{cases} \Bigg\} \ .
    \end{equation}
    A hidden neuron $\nu$ is said to be \emph{isolated} if $\fX_\nu \neq \emptyset$.
\end{definition}

\subsection{Extension \Cref{lemma:convexity_criterion} to cover maxpool activations.}

We now prove an extension of \Cref{lemma:convexity_criterion} that also covers DAG $\relu$ networks including $\maxpool$ activations.
We highlight in blue the subtle difference with the statement \Cref{lemma:convexity_criterion}.
\begin{lemma}[Local convexity criterion for $\relu$ networks]
    \label{lemma:convexity_criterionmaxpool}
    Consider a $\relu$ network described by a DAG and parametrized by $\theta$ which implements a CPWL function $f_\theta: \bbR^d \to \bbR$.
    Consider an isolated neuron $\nu \in H$ and $x \in \fX_\nu \neq \emptyset$.
    Denote $\bfa^+,\bfa^-$ the two distinct values of $\bfa_{\nu}(\cdot,\theta)$ from \Cref{def:variant_IsolatedMaxPool}.
    There exists $\epsilon > 0$ such that
    \begin{enumerate}
    \item there are $x^+, x^- \in B(x, \epsilon)$ such that $f_\theta$
    is differentiable at $x^{+},x^{-}$ and
    \textcolor{blue}{$\bfa_{\nu}(x^+,\theta)=\bfa^{+}$,  $\bfa_{\nu}(x^-,\theta)=\bfa^-$};
    \item for any such pair $x^{+},x^{-}$ it holds
    \begin{align}
        \label{eq:monotoneconditiondagmaxpool}
        & \langle \nabla f_{\theta}(x^+) - \nabla f_{\theta}(x^-),x^+ - x^-\rangle \geq 0  \nonumber \\
        \iff & \langle a^{\nu\to},\Phi^{\nu \to}(\theta)\rangle \geq 0 \ .
    \end{align}
    \end{enumerate}
\end{lemma}

\begin{proof}
    Consider $\nu \in H$ an isolated  neuron.
    In case of a $\relu$ neuron, $\bfa_{\nu}(\cdot,\theta) = a_{\nu}(\cdot,\theta)$, and without loss of generality $\bfa^+=1$, $\bfa^{-}$, so the result is a consequence of  \Cref{lemma:convexity_criterion}.
Assume now that $\nu$ an isolated neuron with $\maxpool$ activation.
	We highlight  the parts of the proof identical in spirit to the corresponding proof of  \Cref{lemma:convexity_criterion}.

    \unchanged{Pick $x \in \fX_\nu \neq \emptyset$ and a ball $B(x, \epsilon)$ as in \Cref{def:variant_IsolatedMaxPool}.
    By definition of $B(x, \epsilon)$, $a^{\nu \to}(\cdot, \theta)$ is constant on this ball
    (it only involves products with edge activations $a^{\mu \to \mu'}$, $\mu' \neq \nu
    $).}

    \newtext{By definition of $\fX_\nu$, $\bfa_\nu$ takes exactly two different values on $B(x,\epsilon)$; denote $\mu^+$ and $\mu^-$ the two antecedents of $\nu$ which realize the max-pool on this ball. The only path-activations which change are for paths containing the edges $\mu^+ \to \nu$ and $\mu^- \to \nu$.
        Without loss of generality we assume that the arbitrary numbering of the antecedents of $\nu$ chosen to uniquely
        break ties in the $\arg\max$ is such that $\mu^-$ is associated to the smallest number. This implies that as soon as
        $\bfa_{\nu}(x,\theta) = e_{\mu^{+}}$ the strict inequality
        \begin{equation}\label{eq:strictineqmaxpool}
        [z_{\nu}(x)]_{\mu^{+}}>[z_{\nu}(x)]_{\mu^{-}}
        \end{equation} holds, while when   $\bfa_{\nu}(x,\theta) = e_{\mu^{-}}$ we have
            \begin{equation}\label{eq:largeineqmaxpool}
        [z_{\nu}(x)]_{\mu^{-}}\geq[z_{\nu}(x)]_{\mu^{+}}.
        \end{equation}
    }

    \unchanged{The path-activation matrix $A(\cdot, \theta)$ takes exactly two values on $B(x,\epsilon)$.}

    \newtext{This implies the existence of $x^+$ (resp. $x^-$) in $B(x,\epsilon)$ such that
    $[\bfa_\nu(x^+, \theta)]_{\mu^+} = 1, [\bfa_\nu(x^+, \theta)]_{\mu^-} = 0$
    (resp. $[\bfa_\nu(x^-, \theta)]_{\mu^+} = 0, [\bfa_\nu(x^-, \theta)]_{\mu^-} = 1$).
    }

    \unchanged{As $A(\cdot, \theta)$ is locally constant around $x^+$ and $x^-$, $f_\theta$ is differentiable at $x^+$ and $x^-$ establishing the first point.
    Now, consider {\em any} two such points $x^+, x^- \in B(x,\epsilon)$ and denote $A^+ := A(x^+, \theta)$, $A^- := A(x^-, \theta)$.
    Denote $u^+ = \nabla f_\theta(x^+)$ (resp. $u^- = \nabla f_\theta(x^-)$) the two slopes of $f_\theta$ given by $A^+$ (resp. $A^-$).
    To get the slope of $f_\theta$ from the expression $\langle \Phi, A \begin{pmatrix} x \\ 1\end{pmatrix} \rangle$ in \Cref{eqn:f_from_path}, we need to extract the block matrix corresponding to the paths which start with input neurons $\Nin$ (recall that the remaining rows/columns of $A$ collect the biases of the network)
    \begin{equation}
        (u^+ - u^-)^\top = \Phi^T_{[p:p_0 \in \Nin]}(A^+ - A^-)_{[p: p_0 \in \Nin, \Nin]}.
    \end{equation}}
    \newtext{Because the only entries that change between $A^+$ and $A^-$ are the ones corresponding to paths $p \in \cP$ which contain the edge $\mu^+ \to \nu$ or $\mu^- \to \nu$, it holds
    \begin{equation}
        \Delta u := (u^+ - u^-)^\top = \begin{pmatrix}
            \Phi_{\left[p: \substack{p \ni (\mu^+, \nu) \\ p_0 \in \Nin}\right]}^\top  & \Phi_{\left[p: \substack{p \ni (\mu^-, \nu)  \\ p_0 \in \Nin} \right]}^\top
        \end{pmatrix}  \begin{pmatrix}
            (A^+ - A^-)_{\left[p: \substack{p \ni (\mu^+, \nu) \\ p_0 \in \Nin} , \Nin \right]} \\
            (A^+ - A^-)_{\left[p: \substack{p \ni (\mu^-, \nu) \\ p_0 \in \Nin} , \Nin \right]}
        \end{pmatrix}   .
    \end{equation}
    Using the factorization of \Cref{lemma:variant_kron}
    \begin{align}
        A^+_{p : \substack{p \ni (\mu^+, \nu) \\ p_0 \in \Nin}} & = \underbrace{a_{\mu^+ \to \nu}(x^+, \theta)}_{=1} A^{\to \mu^+}_{[q : q_0 \in \Nin, \Nin]} \otimes a^{\nu \to} , \\
        A^+_{p : \substack{p \ni (\mu^-, \nu) \\ p_0 \in \Nin}} & = \underbrace{a_{\mu^- \to \nu}(x^+, \theta)}_{=0} A^{\to \mu^-}_{[q : q_0 \in \Nin, \Nin]} \otimes a^{\nu \to} = \mathbf 0 , \\
        A^-_{p : \substack{p \ni (\mu^+, \nu) \\ p_0 \in \Nin}} & = \underbrace{a_{\mu^+ \to \nu}(x^-, \theta)}_{=0} A^{\to \mu^+}_{[q : q_0 \in \Nin, \Nin]} \otimes a^{\nu \to}  = \mathbf 0, \\
        A^-_{p : \substack{p \ni (\mu^-, \nu) \\ p_0 \in \Nin}} & = \underbrace{a_{\mu^- \to \nu}(x^-, \theta)}_{=1} A^{\to \mu^-}_{[q : q_0 \in \Nin, \Nin]} \otimes a^{\nu \to}  ,
    \end{align}
    }
    from which $\Delta u$ rewrites as
    \begin{align}
        \Delta u &  = \begin{pmatrix}
            \Phi_{\left[p: \substack{p \ni (\mu^+, \nu) \\ p_0 \in \Nin}\right]}^\top  & \Phi_{\left[p: \substack{p \ni (\mu^-, \nu)  \\ p_0 \in \Nin} \right]}^\top
        \end{pmatrix}  \begin{pmatrix}
            A^{\to \mu^+}_{[q : q_0 \in \Nin, \Nin]} \otimes a^{\nu \to} \\
            -A^{\to \mu^-}_{[q : q_0 \in \Nin, \Nin]} \otimes a^{\nu \to}
        \end{pmatrix}   \\
        & = \left( \Phi_{\left[p: \substack{p \ni (\mu^+, \nu) \\ p_0 \in \Nin}\right]}^\top (A^{\to \mu^+}_{[q : q_0 \in \Nin, \Nin]} \otimes a^{\nu \to}) \right) - \left(\Phi_{\left[p: \substack{p \ni (\mu^-, \nu)  \\ p_0 \in \Nin} \right]}^\top (A^{\to \mu^-}_{[q : q_0 \in \Nin, \Nin]} \otimes a^{\nu \to})\right)
    \end{align}
    The first term factorizes as
    \begin{align}
        \Phi_{\left[p: \substack{p \ni (\mu^+, \nu) \\ p_0 \in \Nin}\right]}^\top (A^{\to \mu^+}_{[q : q_0 \in \Nin, \Nin]} \otimes a^{\nu \to}) & =(\theta^{\mu^+ \to \nu} \Phi^{\to \mu^+}_{[q:q_0 \in \Nin]}\otimes \Phi^{\nu \to} )^\top (A^{\to \mu^+}_{[q : q_0 \in \Nin, \Nin]} \otimes a^{\nu \to}) \\
        & = \theta^{\mu^+ \to \nu} ((\Phi^{\to \mu^+}_{[q:q_0 \in \Nin]})^\top A^{\to \mu^+}_{[q : q_0 \in \Nin, \Nin]} ) \otimes \underbrace{((\Phi^{\nu \to})^\top  a^{\nu \to})}_{\text{scalar}} \\
        & = ((\Phi^{\nu \to})^\top  a^{\nu \to}) \left( \theta^{\mu^+ \to \nu} ((\Phi^{\to \mu^+}_{[q:q_0 \in \Nin]})^\top A^{\to \mu^+}_{[q : q_0 \in \Nin, \Nin]} )  \right) ,
    \end{align}
    while similarly the second term rewrites
    \begin{align}
        \Phi_{\left[p: \substack{p \ni (\mu^-, \nu)  \\ p_0 \in \Nin} \right]}^\top (A^{\to \mu^-}_{[q : q_0 \in \Nin, \Nin]} \otimes a^{\nu \to}) & = (\theta^{\mu^- \to \nu} \Phi^{\to \mu^-}_{[q:q_0 \in \Nin]} \otimes \Phi^{\nu \to} )^\top (A^{\to \mu^-}_{[q : q_0 \in \Nin, \Nin]} \otimes a^{\nu \to}) \\
        & = \theta^{\mu^- \to \nu} ( (\Phi^{\to \mu^-}_{[q:q_0 \in \Nin]})^\top A^{\to \mu^-}_{[q : q_0 \in \Nin, \Nin]} ) \otimes \underbrace{((\Phi^{\nu \to})^\top  a^{\nu \to})}_{\text{scalar}}  \\
        & = ((\Phi^{\nu \to})^\top  a^{\nu \to}) \left( \theta^{\mu^- \to \nu} ( (\Phi^{\to \mu^-}_{[q:q_0 \in \Nin]})^\top A^{\to \mu^-}_{[q : q_0 \in \Nin, \Nin]} )\right) ,
    \end{align}
    leading to
    \begin{equation}
        \Delta u =
        \langle a^{\nu\to},\Phi^{\nu \to}(\theta)\rangle
        \left[ \left( \theta^{\mu^+ \to \nu} (\Phi^{\to \mu^+}_{[q:q_0 \in \Nin]})^\top A^{\to \mu^+}_{[q : q_0 \in \Nin, \Nin]} \right) - \left(\theta^{\mu^- \to \nu} (\Phi^{\to \mu^-}_{[q:q_0 \in \Nin]})^\top A^{\to \mu^-}_{[q : q_0 \in \Nin, \Nin]}  \right) \right]
    \end{equation}
    Then we have \begin{multline}
        \langle u^+ - u^-, x^+-x^- \rangle =
        \langle a^{\nu\to},\Phi^{\nu \to}(\theta)\rangle
        \bigg[ \left( \theta^{\mu^+ \to \nu} (\Phi^{\to \mu^+}_{[q:q_0 \in \Nin]})^\top A^{\to \mu^+}_{[q : q_0 \in \Nin, \Nin]} ( x^+ - x^-) \right) \\- \left( \theta^{\mu^- \to \nu} (\Phi^{\to \mu^-}_{[q:q_0 \in \Nin]})^\top A^{\to \mu^-}_{[q : q_0 \in \Nin, \Nin]} ( x^+ - x^-) \right) \bigg] \label{eqn:variant_scalar_prod_factorized}
    \end{multline}
    To prove the equivalence~\eqref{eq:monotoneconditiondagmaxpool} it is thus enough to show that
    \begin{equation}\label{eqn:variant_scalar_prod_factorizedbis}
    \bigg[ \left( \theta^{\mu^+ \to \nu} (\Phi^{\to \mu^+}_{[q:q_0 \in \Nin]})^\top A^{\to \mu^+}_{[q : q_0 \in \Nin, \Nin]} ( x^+ - x^-) \right) \\- \left( \theta^{\mu^- \to \nu} (\Phi^{\to \mu^-}_{[q:q_0 \in \Nin]})^\top A^{\to \mu^-}_{[q : q_0 \in \Nin, \Nin]} ( x^+ - x^-) \right) \bigg]>0.
    \end{equation}

    By definition of $x^{+}$, $x^{-}$ and \eqref{eq:strictineqmaxpool}-\eqref{eq:largeineqmaxpool}, we have
    \[
    [z_{\nu}(x^{+},\theta)]_{\mu^{+}}>[z_{\nu}(x^{+},\theta)]_{\mu^{-}},\qquad \text{and}\qquad
    [z_{\nu}(x^{-},\theta)]_{\mu^{-}}\geq [z_{\nu}(x^{-},\theta)]_{\mu^{-}}
    \]
	In other words,
   \begin{align}
        \theta^{\mu^+ \to \nu} \mu^+(x^+)  > \theta^{\mu^- \to \nu} \mu^-(x^+)
\qquad \text{and}\qquad
        \theta^{\mu^- \to \nu} \mu^-(x^-)  \geq  \theta^{\mu^+ \to \nu} \mu^+(x^-)
    \end{align}
    which implies
     \begin{equation}\label{eqn:max_pool_compar-0}
      \theta^{\mu^+ \to \nu} ( \mu^+(x^+) - \mu^+(x^-) ) - \theta^{\mu^- \to \nu} ( \mu^-(x^+) - \mu^-(x^-) ) > 0.
    \end{equation}
    By \Cref{eqn:f_from_path}, here used on two neural networks which output neuron $\mu^{+}$ (resp. $\mu^{-}$) is
    {\em not} linear, the post-activation $\mu^+$ and $\mu^-$ can be explicited for any input $x$ using the path-lifting as
    \begin{align}
        \mu^+(x) & = \left\langle \Phi^{\to \mu^+}, A^{\to \mu^+}_{[q : q_0 \in \Nin, \Nin]}(x, \theta) \begin{pmatrix}
            x \\ 1
        \end{pmatrix} \right\rangle ,\\
        \mu^-(x) & = \left\langle \Phi^{\to \mu^-}, A^{\to \mu^-}_{[q : q_0 \in \Nin, \Nin]}(x, \theta) \begin{pmatrix}
            x \\ 1
        \end{pmatrix} \right\rangle ,
    \end{align}
    so that, as $A^{\to \mu^+}_{[q : q_0 \in \Nin, \Nin]}(x^+, \theta) = A^{\to \mu^+}_{[q : q_0 \in \Nin, \Nin]}(x^-, \theta) = A^{\to \mu^+}_{[q : q_0 \in \Nin, \Nin]}$
    \begin{align}
        \mu^+(x^+) - \mu^+(x^-) & = \left\langle \Phi^{\to \mu^+}, A^{\to \mu^+}_{[q : q_0 \in \Nin, \Nin]}\begin{pmatrix}
            x^+ - x^- \\ 0
        \end{pmatrix} \right\rangle \\
        & = (\Phi^{\to \mu^+}_{[q: q_0 \in \Nin]})^\top A^{\to \mu^+}_{[q : q_0 \in \Nin, \Nin]} (x^+ - x^-).
    \end{align}
    Doing the same computation for $\mu^-$, \Cref{eqn:variant_scalar_prod_factorizedbis} is equivalent to
\[
      \theta^{\mu^+ \to \nu} ( \mu^+(x^+) - \mu^+(x^-) ) - \theta^{\mu^- \to \nu} ( \mu^-(x^+) - \mu^-(x^-) ) > 0.
\]
which indeed holds, cf \eqref{eqn:max_pool_compar-0}.
\end{proof}
 \section{Extensive study of the 2D counter example and proof of \Cref{prop:ICNN_counter_ex}}\label{app:counter_ex}

\subsection{Proof of \Cref{prop:ICNN_counter_ex}}
Recall the CPWL function $f_{\ex}: \bbR^2 \to \bbR$ is implemented by a $\relu$ network with $2$ hidden layers and $2$ neurons per layer (i.e. belonging to $\SMLPclass((2,2))$):
\begin{align}
    \label{eq:Ex2DConvex}
    f_{\ex} \begin{pmatrix} x_1 \\ x_2 \end{pmatrix} & = \begin{pmatrix}
        1 & 1
    \end{pmatrix} \relu \left(\begin{pmatrix}
        -1 & 1 \\
        2 & 1
    \end{pmatrix} \relu \begin{pmatrix} x_1 \\ x_2 \end{pmatrix}  + \begin{pmatrix}
        -1 \\
        -0.5
    \end{pmatrix} \right) \ .
\end{align}
The network written above is not an ICNN because of the negative weight $-1$ in $W_2$ of the second hidden layer.
Still, $f_\theta$ is convex: it can be easily checked using the sufficient conditions of \Cref{prop:cvx_cpwl}.
Plots of $f_{\ex}$ are displayed in \Cref{fig:counter_ex_plots}.

\begin{figure}[htp]
    \begin{center}
        \includegraphics[width=0.35\textwidth]{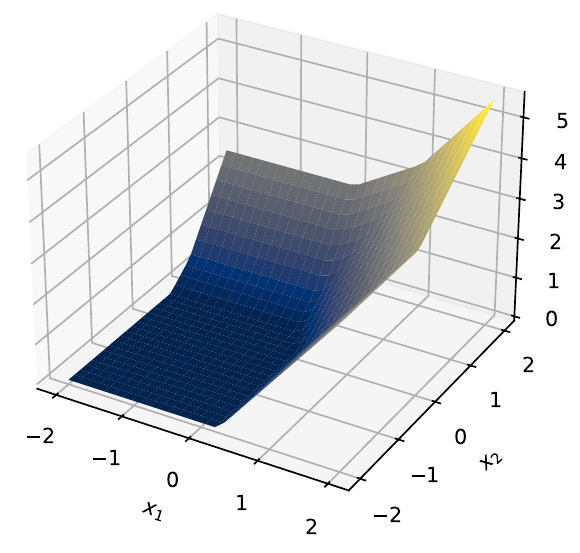}
        \includegraphics[width=0.35\textwidth]{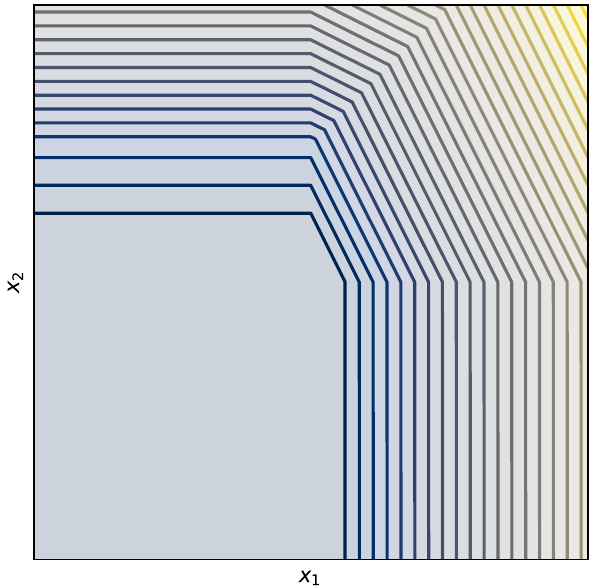}
    \end{center}
    \caption{Visual representations of $f_{\ex}$. Left: 3D plot of $f_{\ex}$. Right: Level lines of $f_{\ex}$.}
    \label{fig:counter_ex_plots}
\end{figure}

\begin{figure}[!htbp]
    \centering
    \includegraphics[width=0.4\textwidth]{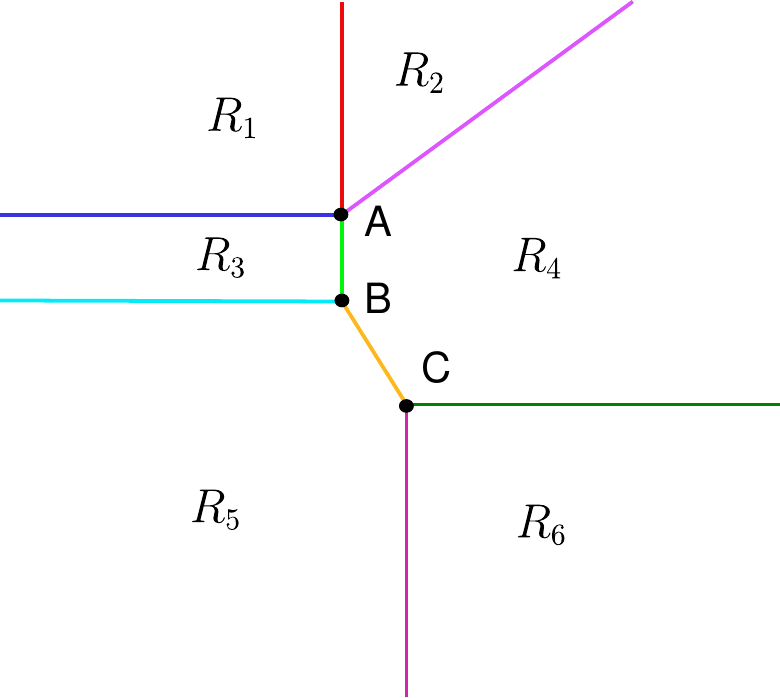}
    \caption{Frontiers and regions of the convex 2D CPWL function}
    \label{fig:ICNN_counter_ex_frontiers}
\end{figure}

To prove \Cref{prop:ICNN_counter_ex}, we adopt a reverse-engineering approach: we only use the knowledge of the landscape of the convex CPWL function $f_{\ex} \in \CvxCPwL$.
The landscape diplays 8 pieces of hyperplanes (represented in color in \Cref{fig:ICNN_counter_ex_frontiers}), which we name frontiers, and 6 corresponding regions $R_k, k \in \{1 \dots, 6 \}$.
Frontiers are denoted $F_{k,l}$ with $k,l \in \{1, \dots, 6 \}$ and correspond to points at which $f_{\ex}$ is non-differentiable.
By contradiction, we assume that $f_\ex$ can be implemented as an ICNN:
\begin{equation}
    f_{\ex}(x) = w_3^\top \relu \left( W_2 \relu (W_1 x + b_1) + V_2 x + b_2 \right) + v_3^\top x + b_3,
\end{equation}
with $w_3, W_2$ having non-negative entries.

Now, we aim to identify the parameters
$(W_1,W_2,w_3,V_2,v_3,b_1,b_2,b_3)$ of this ICNN.
Frontiers, i.e. points of non-differentiability of $f_{\ex}$, arise from the switch of -at least- one neuron (otherwise, the function is locally affine, hence differentiable).
A priori, each frontier (i.e. each colored segment in \Cref{fig:ICNN_counter_ex_frontiers}) could be matched to one or more neurons which switch, so there would be an exponential number of cases to consider.
Fortunately, not all arrangements of pieces of hyperplanes generated by neurons of $\relu$ networks are possible.
The sketch of the proof is as follows: we first state general considerations on $\relu$ networks which enable us to restrict the possible matching between frontiers and neurons to a few cases.
Then, writing down the equations which describe the frontiers, we can identify the parameters of the first layer and second. We end up showing there must be a negative weight in $W_2$, which contradicts the ICNN hypothesis.

\paragraph*{Notation}
We denote $\mu_1, \mu_2 \in N_1$ the neurons of the first layer and $\nu_1, \nu_2 \in N_2$ the neurons of the second layer.
We use the following notation:
$z_{\eta} : \bbR^2 \to \bbR $ is the pre-activation of the neuron $\eta \in \{\mu_1, \mu_2, \nu_1, \nu_2\}$ and $y_\eta := \relu \circ \ z_\eta : \bbR^2 \to \bbR $ is its post-activation.
For a layer $i \in \{1, 2\}$, we denote $z_i : \bbR^2 \to \bbR^2$ its pre-activation and $y_i: \bbR^2 \to \bbR^2$ its post-activation.
We denote  $H_{\eta} := \{ x \in \bbR^2 : \ z_{\eta}(x) = 0 \}$ the $0$-level set of the pre-activation of neuron $\eta$.
For $k, l \in \{1,\dots, 6\}$, we denote $F_{k,l}$ the frontier between neighboring regions $R_k$ and $R_l$ following the numbering given \Cref{fig:ICNN_counter_ex_frontiers}.\\

\noindent \textbf{Frontiers from the first layer}
There are two cases to consider separately regarding $W_1$: the full rank case and degenerated configurations.
If $W_1$ is full rank (i.e. rank 2), then  we get from $H_{\mu_i} = \{x \in \bbR^2: (W_1)_{i:}x + (b_1)_i = 0 \}$ that $\dim(H_{\mu_i}) = 1$ for every $i \in \{1, 2\}$ and $H_{\mu_1} \cap H_{\mu_2} $ is a singleton.
If $W_1$ is not full rank, then different (degenerated) configurations can occur: one $H_{\mu_i}$ is of dimension $0$ or $2$ or $H_{\mu_1}$ and $H_{\mu_2}$ are parallel (or coincident).
In any case, all level lines of $z_{\mu_1}$ and $z_{\mu_2}$ are parallel.

\vspace{0.5cm}
\noindent \textsc{The full rank case.} \\
Next, we assume $W_1$ is full rank. \\

\vspace{0.1cm}
\noindent \textbf{Frontiers from the second layer}

In this case, the two hyperplanes generated by the first layer neurons are crossing lines, it means that the first layer divides the input space $\bbR^2$ into 4 regions, that we denote $A_k$, $k \in \{1, \ldots, 4\}$.
We now describe the possible shapes that the $0$-level set of pre-activations of second layer neurons (so called bent hyperplanes \citep{hanin2019deep}) can take.
This result is illustrated in \Cref{fig:counter_ex_shapes}.

On each of the $A_k$, $z_{\nu_i}$ is an affine function of the input $x$: $y_1$ and $V_2 x$ are both affine in $x$ on each region.
As in the reasoning for the first layer, $H_{\nu_i} \cap \intt(A_k) $ can only be: the empty set, the full region $ A_k$, a segment which joins the two boundaries of $A_k$ or a ray which intersects one boundary of $A_k$.
Besides, if $H_{\nu_i}$ intersects two neighboring regions $A_k$ and $A_l$ with rays/segment, then these rays/segments must have a common point at the frontier $A_k \cap A_l$ (to see it, consider $x \in H_{\nu_i} \cap A_k \cap A_l$ and any point $y \in H_{\nu_i} \cap \intt A_l$, then $[x, y] \subset H_{\nu_i} \cap A_l$).
Last, if $H_{\nu_i} \cap A_k$ is of dimension one (ray or segment) and $H_{\nu_i}$ intersects the frontier $A_k \cap A_l$, then $H_{\nu_i} \cap A_l$ is also of dimension 1.
In other words, if a segment or ray of $H_{\nu_i}$ hits a frontier, it must continue on the other side of the frontier.

Following the previous remarks, we provide in \Cref{fig:counter_ex_shapes} all admissible shapes for the bent hyperplane of a neuron of the second layer.
\begin{figure}[htp]
    \centering
    \includegraphics[width=0.8\textwidth]{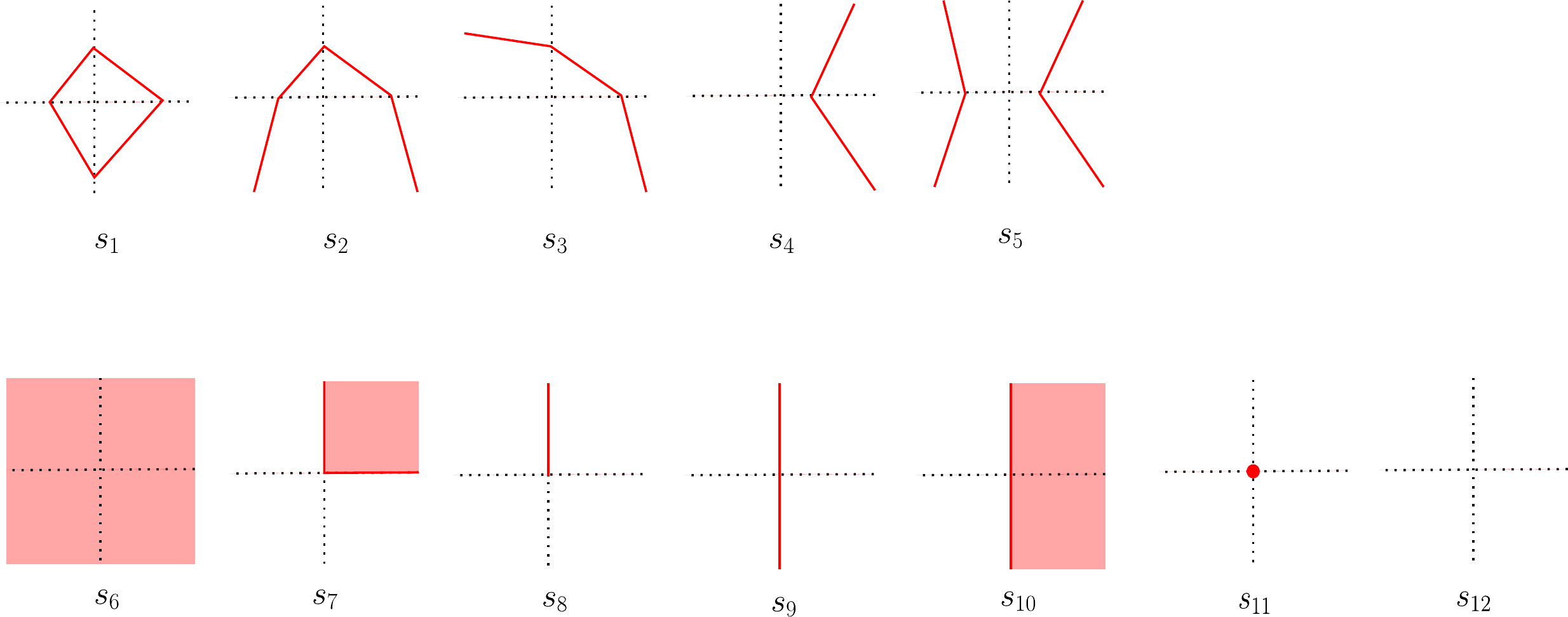}
    \caption{Admissible shapes of the bent hyperplane of the second layer. Dotted lines correspond to crossing hyperplanes given by neurons of the first layer. Top: non degenerated cases. Bottom: degenerated cases. }
    \label{fig:counter_ex_shapes}
\end{figure}

As $f_{\ex}$ has strictly more than $4$ regions, there exists at least one neuron of the second layer, wlog $\nu_1 \in N_2$, such that $H_{\nu_1} \cap \intt A_k$ has dimension $1$ for some region $A_k$ (other cases correspond to degenerated cases where $\nu_i$ does not introduce new regions in $f_\ex$).
Using the notations in \Cref{fig:counter_ex_shapes}, we end-up with $H_{\nu_1} \in \mathcal S :=\{s_1, s_2, s_3, s_4, s_5 \}$ and $H_{\nu_2} \in \mathcal S \cup \{ s_6, s_7, s_8, s_9, s_{10} \}$.

\vspace{0.5cm}
\noindent \textbf{Frontiers of the full ICNN}
The last layer is linear with respect to $y_2$ so it can not create additional frontiers (i.e. can not add points of non differentiability of $f_{\ex}$).
So, the landscape we observe in \Cref{fig:ICNN_counter_ex_frontiers} must correspond to the superposition of the bent hyperplanes of the second layer.
On the other hand, hyperplanes corresponding to the first layer may be partially hidden by the effect of the second layer.

\vspace{0.5cm}
\noindent \textbf{Identification of the neurons}
At points $A, B$ and $C$ in \Cref{fig:ICNN_counter_ex_frontiers}, pieces of hyperplanes cross, and there is some bend.
Hence, they correspond to at least one neuron of the first layer and one neuron of the second layer.
In particular, the points $A,B, C$ all belong to $H_{\mu_1} \cup H_{\mu_2}$.
As they are not aligned, they can not all belong to a single $H_{\mu_i}$, $i \in \{1,2\}$: wlog, either $[A,B] \subset H_{\mu_1}$, either $[B,C] \subset H_{\mu_1}$.

\vspace{0.5cm}
$\bullet~$\textit{Case $[B,C] \subset H_{\mu_1}$} (\Cref{fig:ICNN_rev_eng_case_1}).
In that case $B, C \in H_{\mu_1}$ and $A \in H_{\mu_2}$.
Then,  ${F_{5,6}} \cup {F_{6,4}}$ corresponds to a bent hyperplane associated with a neuron of the second layer, wlog $H_{\nu_1}$ (shapes $s_4$ or $s_5$ in \Cref{fig:counter_ex_shapes}).
The frontiers ${F_{3,5}} \cup {F_{3,4}}$ cannot also belong to $H_{\nu_1}$ (shape $s_5$): indeed, it would imply that $H_{\mu_1}$ and $H_{\mu_2}$ cross on $[B,C]$ and $A$ would not belong to $H_{\mu_2}$.
So, the only admissible pattern for $H_{\nu_1}$ is then $s_4$.
So, ${F_{3,5}} \cup {F_{3,4}} \subset H_{\nu_2}$.
We end up on a contradiction: at $A$, this bent hyperplane should split in two pieces, impossible.

\vspace{0.5cm}
$\bullet~$\textit{Case $[A,B] \subset H_{\mu_1}$} (\Cref{fig:ICNN_rev_eng_case_2}).
In that case $A, B \in H_{\mu_1}$ and $C \in H_{\mu_2}$.
We have already excluded $[B,C] \subset H_{\mu_2}$, so either ${F_{4,6}} \subset H_{\mu_2}$, either ${F_{5,6}} \subset H_{\mu_2}$.
The latter would correspond to a degenerate configuration (parallel hyperplanes).
Since we are focusing on the full-rank case for now, we do not study it at this point.
We thus end-up with the identification of \Cref{fig:ICNN_rev_eng_case_2_final} (up to a permutation of neurons in the same layer).
Note that it corresponds to the initial configuration given by the non-ICNN network implementing $f_{\ex}$.

\begin{figure}[htp]
\begin{subfigure}{0.3\textwidth}
    \centering
    \includegraphics[width=\textwidth]{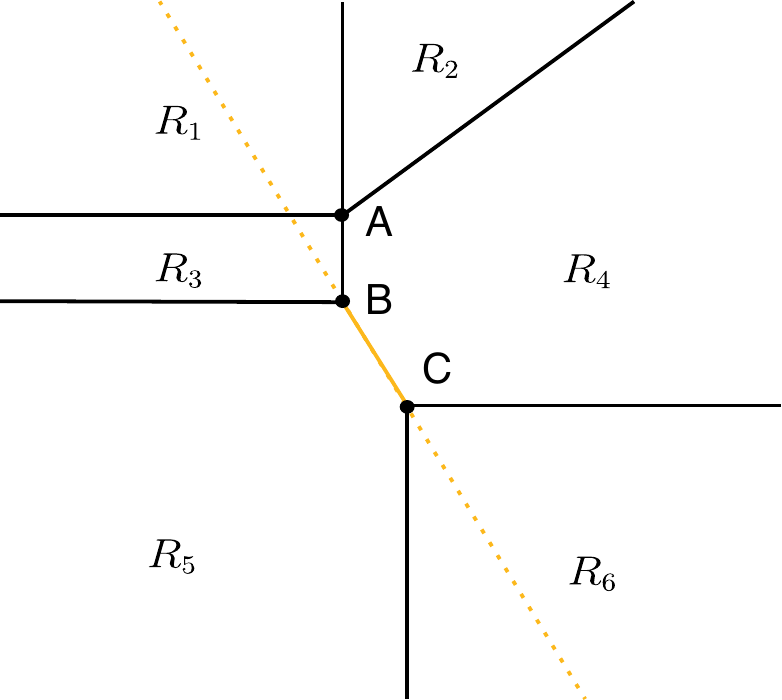}
    \caption{Case $[B,C] \subset H_{\mu_1}$}
    \label{fig:ICNN_rev_eng_case_1}
\end{subfigure}
\hfill
\begin{subfigure}{0.3\textwidth}
    \centering
    \includegraphics[width=\textwidth]{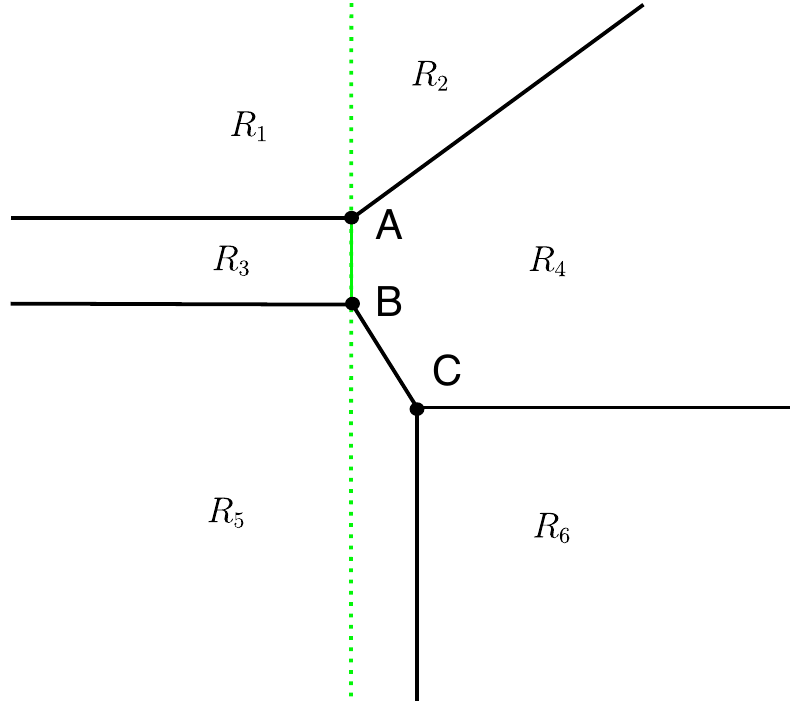}
    \caption{Case $[A,B] \subset H_{\mu_1}$}
    \label{fig:ICNN_rev_eng_case_2}
\end{subfigure}
\hfill
\begin{subfigure}{0.3\textwidth}
    \centering
    \includegraphics[width=\textwidth]{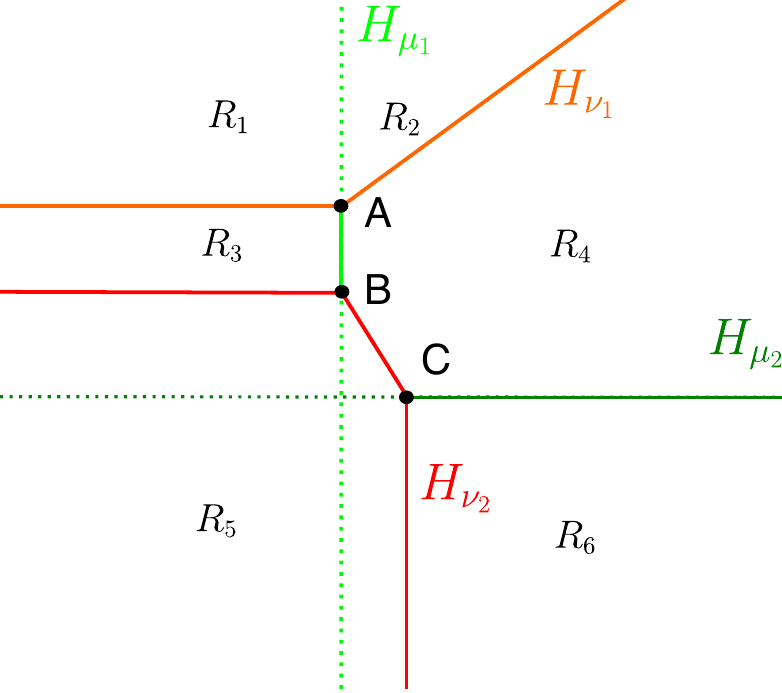}
    \caption{Final identification}
    \label{fig:ICNN_rev_eng_case_2_final}
\end{subfigure}
\caption{Identification of the neurons.}
\label{fig:ICNN_identification_first_layer}
\end{figure}

\vspace{0.5cm}
\noindent \textbf{Identification of the parameters}
Up to a scaling and to the sign, identifying the hyperplanes of the first layer gives the weights and bias of the first layer:
\begin{align}
    W_1 & =
    \begin{pmatrix}
     \pm 1 & 0 \\
        0 & \pm 1
    \end{pmatrix} , \quad \quad
    b_1 =  \begin{pmatrix}
       0 \\
         0
       \end{pmatrix}
\end{align}
We can fix the output layer $w_3 = (1,1)$ as $\relu$ networks are invariant modulo rescaling of parameters \cite{stock_embedding_2022}.
The ICNN we consider is thus parametrized as follows:
\begin{align}
    \begin{pmatrix}
        1 & 1
    \end{pmatrix} \relu \left(\begin{pmatrix}
        a & b \\
        c & d
    \end{pmatrix}
    \relu
   \begin{pmatrix}
    \epsilon_1 x_1 \\ \epsilon_2 x_2
\end{pmatrix}   + \begin{pmatrix}
         e & f \\
         g & h
    \end{pmatrix}  \begin{pmatrix}
        x_1 \\ x_2
    \end{pmatrix}  +\begin{pmatrix}
        i \\ j
    \end{pmatrix}\right) + V_3 \mathbf x +  b_3 \ .
\end{align}
where, for $k=1,2$, $\epsilon_k \in \{-1,1\}$ denotes the weight $W_1[\mu_k, \mu_k]$.
Define $\zeta_k = \frac{1+\epsilon_k}{2} \in \{0,1\}$.
Let us list all the equations the ICNN must satisfy to have the same frontiers as our target function.

\begin{itemize}
    \item $F_{1:3}$. The equation of $F_{1:3}$ is $\alpha (x_2-1)=0$ with $\alpha \neq 0$.
    The zero-level set of $z_{\nu_1}$ on $x_1 \in \bbR_-, x_2 \in \bbR_+$ rewrites
    \begin{equation*}
        - (1-\zeta_1) a x_1 + \zeta_2 b x_2 + e x_1 + f x_2 + i = 0 \, ,
    \end{equation*}
    and we identify the equality constraints
    \begin{align}
        -(1-\zeta_1) a + e & = 0 \label{eqn:ICNNEQ_x1} \\
        \zeta_2 b  + f & = \alpha \label{eqn:ICNNEQ_y1} \\
        i & = -\alpha \label{eqn:ICNNEQ_bias1}
    \end{align}
    \item $F_{2:4}$. The equation of $F_{2:4}$ is $\beta(x_2-x_1-1)=0$ with $\beta \neq 0$.
    The zero level of $z_{\nu_1}$ on $x_1 \in \bbR_+, x_2 \in \bbR_+$ is
    \begin{equation*}
        \zeta_1 a x_1 + \zeta_2 b x_2 + e x_1 + f x_2 + i = 0
    \end{equation*}
    and we identify the equality constraints
    \begin{align}
        \zeta_1 a + e & = -\beta \label{eqn:ICNNEQ_x2}\\
        \zeta_2 b  + f & = \beta \label{eqn:ICNNEQ_y2} \\
        i & = -\beta  \label{eqn:ICNNEQ_bias2}
    \end{align}
\end{itemize}

We get from \Cref{eqn:ICNNEQ_bias1,eqn:ICNNEQ_bias2} that $\alpha=  \beta$.
We study two cases: $\alpha > 0$ and $\alpha < 0$.
In the first case, wlog consider $\alpha = 1$.
Then, a first set of equations the ICNN weights must satisfy is:
\begin{align}
    -(1-\zeta_1) a + e & = 0 \label{eqn:ICNNEQ_1} \\
    \zeta_1 a + e & = -1 \label{eqn:ICNNEQ_2} \\
    \zeta_1, \zeta_2 & \in \{0,1\}  \hspace{1cm} \text{(activations of $\mu_1, \mu_2 \in N_1$)} \\
    a,b,c,d & \geq 0 \hspace{1.6cm} \text{(ICNN architecture)}
\end{align}

By injecting \eqref{eqn:ICNNEQ_2} into \eqref{eqn:ICNNEQ_1}, we get $a = -1$ which contradicts the  non-negativity constraint on $a$.

Now, we study the case where $\alpha < 0$, wlog $\alpha = -1$:
\begin{align}
    -(1-\zeta_1) a + e & = 0 \label{eqn:ICNNEQ_1_case2} \\
    \zeta_1 a + e & = 1 \label{eqn:ICNNEQ_2_case2} \\
    \zeta_1, \zeta_2 & \in \{0,1\}  \hspace{1cm} \text{(activations of $\mu_1, \mu_2 \in N_1$)} \\
    a,b,c,d & \geq 0 \hspace{1.6cm} \text{(ICNN architecture)}
\end{align}
In this case, we find $a = 1$ and there is no immediate contradiction at this stage.
By \eqref{eqn:ICNNEQ_bias1} $i = 1$ and $z_{\nu_1}((0, 0)) = i = 1 > 0$ so $\nu_1$ is active on $R_5$ (it is active at one point in the region and the activations are constant on region).

Let us now identify weights associated to the second neuron $\nu_2 \in N_2$.
\begin{itemize}
    \item $F_{3:5}$. The equation of $F_{3:5}$ is $\gamma(x_2 - 0.5) = 0$ with $\gamma \neq 0$. The $0$-level set of $z_{\nu_2}$ on $x_1 \in \bbR_-$, $x_2 \in \bbR_+$ rewrites
    \begin{equation}
        -(1-\zeta_1)cx_1 + d\zeta_2 x_2 + g x_1 + h x_2 + j = 0
    \end{equation}
    and we identify the equality constraints
    \begin{align}
        - (1-\zeta_1)c + g & = 0 \\
        \zeta_2 d + h & = \gamma \\
        j & = - 0.5 \gamma
    \end{align}
    \item $F_{4:5}$. The equation of $F_{4:5}$ is $\delta (x_2 + 2 x_1  - 0.5) = 0$ with $\delta \neq 0$. The $0$-level set of $z_{\nu_2}$ on $x_1 \in \bbR_+, x_2 \in \bbR_+$ is
    \begin{equation}
        \zeta_1 c x_1 + \zeta_2 d x_2 + g x_1 + h x_2 + j = 0 \label{eq:F45}
    \end{equation}
    and we identify the equality constraints
    \begin{align}
        \zeta_1 c + g & = 2 \delta \\
        \zeta_2 d + h & = \delta \\
        j & =  - 0.5 \delta
    \end{align}
    \item $F_{5:6}$. The equation of $F_{5:6}$ is $\eta (2 x_1  - 0.5) = 0$ with $\eta \neq 0$. The $0$-level set of $z_{\nu_2}$ on $x_1 \in \bbR_+, x_2 \in \bbR_-$ is
    \begin{equation}
        \zeta_1 c x_1 -  (1 - \zeta_2) d x_2 + g x_1 + h x_2 + j = 0
    \end{equation}
    and we identify the equality constraints
    \begin{align}
        \zeta_1 c + g & = 2 \eta \\
        -  (1 - \zeta_2) d + h & = 0 \\
        j &  = -  0.5 \eta
    \end{align}
\end{itemize}
So, it follows $\gamma = \delta = \eta $.
We study two cases $\gamma < 0$ and $\gamma > 0 $.
Assume $\gamma < 0$, wlog consider $\gamma =  - 1 $.
It gives the following set of constraints
\begin{align}
    - (1-\zeta_1)c + g &  = 0  \\
     \zeta_1 c + g & =  - 2  \\
     \zeta_1, \zeta_2 \in \{ 0, 1 \} \\
     c,d \geq 0 ,
\end{align}
which is not feasible as $c=  -2$ contradicts the ICNN constraint $c \geq 0$.
So, $\gamma$ has to be positive.

Rewriting the equation of $F_{4:5}$ (second bullet) as $2\gamma x_1 + \gamma x_2 - 0.5 \gamma = 0$, we see that for $x_1, x_2 < 0$, $z_{\nu_2} <0$ so $\nu_2$ is then inactive on $R_5$. Recall that on this region $\nu_1$ is active and
the pre-activation $z_{\nu_1}$ has 4 different affine pieces (depending on which side of $H_{\mu_1}$ and $H_{\mu_2}$ the input $x$ is).
Yet, we observe on the landscape that $f_{\ex}$ is $0$ on this region.
The skip connection term, $v_3^\top x + b_3$, cannot put at $0$ a piecewise affine function with $4$ distinct linear pieces.

We end up with a non-feasibility argument: the frontiers of $f_{\ex}$ can not be implemented by a $2$ hidden-layers, $2$ neurons per layer ICNN.

\vspace{0.5cm}
\noindent \textsc{The degenerated case.} \\
\vspace{0.1cm}
Now, assume $W_1$ is not full rank. Then, the first layer can either generate two parallel but distinct frontiers, either a single frontier.
But these frontiers must pass through the $3$ points $A,B$ and $C$.
This requirement eliminates all cases, apart from $H_{\mu_1}$ and $H_{\mu_2}$ corresponding to $F_{1:2} \cup F_{3:4}$ and $F_{5:6}$.
Up to a scaling and to the sign, identifying the hyperplanes of the first layer gives the weights and bias of the first layer:
\begin{align}
    W_1 & =
    \begin{pmatrix}
     \pm 1 & 0 \\
     \pm 1 & 0
    \end{pmatrix} , \quad \quad
    b_1 =  \begin{pmatrix}
       0 \\
         \pm \frac 1 4
       \end{pmatrix}
\end{align}
As done previously, we can fix the output layer $w_3 = (1,1)$.
The ICNN we consider is thus parametrized as follows:
\begin{align}
    \begin{pmatrix}
        1 & 1
    \end{pmatrix} \relu \left(\begin{pmatrix}
        a & b \\
        c & d
    \end{pmatrix}
    \relu
   \begin{pmatrix}
    \epsilon_1 x_1 \\ \epsilon_2 (x_1 - \frac 1 4)
\end{pmatrix}   + \begin{pmatrix}
         e & f \\
         g & h
    \end{pmatrix}  \begin{pmatrix}
        x_1 \\ x_2
    \end{pmatrix}  +\begin{pmatrix}
        i \\ j
    \end{pmatrix}\right) + V_3 \mathbf x +  b_3 \ .
\end{align}
where, for $k=1,2$, $\epsilon_k \in \{-1,1\}$ denotes the weight $W_1[\mu_k, \mu_k]$.
In this case, the two hyperplanes generated by the first layer neurons are parallel and divide the input space $\bbR^2$ into $3$ regions.
On each of these regions, $z_{\nu_1}$ and $z_{\nu_2}$ are affine function of the input.
Applying the same reasoning as before,  the $0$-level sets of the pre-activation of neurons $\nu_1, \nu_2 \in N_2$ on neighboring regions must be connected.
We thus identify $H_{\nu_1}$ to $F_{1:3} \cup F_{2:4}$ and $H_{\nu_2}$ to $F_{3:5}\cup F_{4:5}\cup F_{4:6}$.
We now study the constrains associated to $H_{\nu_2}$.
\begin{itemize}
    \item $F_{3:5}$. The equation of $F_{3:5}$ is $\gamma(x_2 - 0.5) = 0$ with $\gamma \neq 0$. The $0$-level set of $z_{\nu_2}$ on $x_1 \in \bbR_-$, $x_2 \in \bbR_+$ rewrites
    \begin{equation}
        -(1-\zeta_1)cx_1 -d(1 - \zeta_2) (x_1 - 0.25) + g x_1 + h x_2 + j = 0
    \end{equation}
    so we identify the equality constraints
    \begin{align}
        -(1-\zeta_1)c -d(1 - \zeta_2) + g & = 0 \label{eqn:ICNNEQ_deg_x1}\\
        h & = \gamma \label{eqn:ICNNEQ_deg_y1} \\
        0.25 d(1 - \zeta_2) + j & = - 0.5 \gamma \label{eqn:ICNNEQ_deg_b1}
    \end{align}
    \item $F_{4:5}$. The equation of $F_{4:5}$ is $\delta (x_2 + 2 x_1  - 0.5) = 0$ with $\delta \neq 0$. The $0$-level set of $z_{\nu_2}$ on $x_1 \in \bbR_+, x_2 \in \bbR_+$ is
    \begin{equation}
        \zeta_1 c x_1 -d(1 - \zeta_2) (x_1 - 0.25) + g x_1 + h x_2 + j = 0
    \end{equation}
    \begin{align}
        \zeta_1 c  -d(1 - \zeta_2) + g & = 2 \delta \label{eqn:ICNNEQ_deg_x2}\\
        h & = \delta \label{eqn:ICNNEQ_deg_y2} \\
        0.25 d(1 - \zeta_2) + j & =  - 0.5 \delta
    \end{align}
    \item $F_{4:6}$. The equation of $F_{4:6}$ is $\eta x_2 = 0$ with $\eta \neq 0$. The $0$-level set of $z_{\nu_2}$ on $x_1 \in \bbR_+, x_2 \in \bbR_-$ is
    \begin{equation}
        \zeta_1 c x_1 + \zeta_2 d (x_1 - 0.25) + g x_1 + h x_2 + j = 0
    \end{equation}
    \begin{align}
        \zeta_1 c + \zeta_2 d & = 0 \label{eqn:ICNNEQ_deg_x3}\\
        h & = \eta \label{eqn:ICNNEQ_deg_y3} \\
        - 0.25 \zeta_2 d  + j  &  = 0 \label{eqn:ICNNEQ_deg_b3}
    \end{align}
\end{itemize}
From which we identify:
\begin{itemize}
    \item by \Cref{eqn:ICNNEQ_deg_y1,eqn:ICNNEQ_deg_y2,eqn:ICNNEQ_deg_y3}, $\gamma = \delta = \eta$,
    \item by \Cref{eqn:ICNNEQ_deg_x1}
    \begin{align}
        -c -d + \underbrace{\zeta_1 c + \zeta_2 d}_{=0 \text{ by \eqref{eqn:ICNNEQ_deg_x3}}} + g = 0
    \end{align}
    and by \Cref{eqn:ICNNEQ_deg_x2}
    \begin{align}
        -d + \underbrace{\zeta_1 c + \zeta_2 d}_{=0 \text{ by \eqref{eqn:ICNNEQ_deg_x3}}} + g = 2 \gamma
    \end{align}
    from which if follows
    \begin{equation}
        c = 2 \gamma,
    \end{equation}
    together with $c \geq 0 $ (ICNN condition), it follows $\gamma$ (which is $\neq 0 $) must be positive.
    \item by \Cref{eqn:ICNNEQ_deg_b1,eqn:ICNNEQ_deg_b3}
    \begin{equation}
        0.25d = -0.5 \gamma - \underbrace{(j - 0.25d\zeta_2)}_{= 0 \text{ by \eqref{eqn:ICNNEQ_deg_b3}}}
    \end{equation}
    leading to $0.25d = -0.5 \gamma < 0$  which contradicts the non-negativity constraint on $d$.
\end{itemize}
We can conclude that the degenerate case is non-feasible with respect to the ICNNs constraints.

To wrap up, we cannot find an ICNN with $2$ layers and $2$ neurons per layer such that it implements $f_\ex$, which concludes the proof.

\subsection{Sanity check of the necessary convexity condition given in \Cref{lemma:pos_cond_2_layer}}
\label{app:counter_ex_isolated}

For any neuron $\eta \in (N_1 = \{ \mu_1, \mu_2 \}) \cup (N_2 = \{ \nu_1, \nu_2 \})$, note that $\fX_{\eta}$ is by definition the set of points such that $z_{\eta}(x_1, x_2) = 0$ and $z_{\eta'}(x_1, x_2) \neq 0 $ for $\eta' \neq \eta$.
Each pre-activation is a function from $\bbR^2$ to $\bbR$ given by:
\begin{align}
    z_{\mu_1} : (x_1, x_2) & \mapsto x_1 \\
    z_{\mu_2} : (x_1, x_2) & \mapsto x_2 \\
    z_{\nu_1} : (x_1, x_2) & \mapsto -\relu{(x_1)} + \relu(x_2) -1 \\
    z_{\nu_2} : (x_1, x_2) & \mapsto 2 \relu{(x_1)} + \relu{(x_2)} -0.5
\end{align}
It follows
\begin{align}
    \fX_{\mu_1} & = \{ x \in \bbR^2 : x_1 = 0 \} \setminus \{(0, 0), (0, 1), (0, 0.5) \}  = \{ 0 \} \times (-\infty, 0) \cup (0, 0.5 )\cup (0.5, 1) \cup (1, +\infty) \\
    \fX_{\mu_2} & = \{ x \in \bbR^2 : x_2 = 0 \} \setminus \{(0, 0), (0, 0.25) \}  =  (-\infty, 0) \cup (0, 0.25) \cup (0.25, +\infty)  \times \{ 0 \} \\
    \fX_{\nu_1} & = \{ x \in \bbR^2_+ :  - x_1 + x_2 - 1 = 0 \} \setminus \{(0, 1)\}   \\
    \fX_{\nu_2} & = \{ x \in \bbR^2_+ : 2 x_1 + x_2 -0.5 = 0 \} \setminus \{(0, 0.25), (0.25, 0) \}
\end{align}

\paragraph{All neurons are isolated}

By definition, a neuron $\eta \in (N_1 = \{ \mu_1, \mu_2 \}) \cup (N_2 = \{ \nu_1, \nu_2 \})$ is isolated if $\fX_{\eta} \neq \emptyset$, which is clearly satisfied.

\paragraph{Determine $\bfa_2(\fX_{\mu_1})$}

\begin{itemize}
    \item Consider $x = (x_1 = 0, x_2)$ with $x_2 \in ( 1, +\infty)$:
\begin{equation}
    z_{\nu_1}(x) > 0, \qquad z_{\nu_2}(x)   > 0,
\end{equation}
so $(1,1) \in \bfa_2(\fX_{\mu_1})$.
\item Consider $x = (x_1 = 0, x_2)$ with $x_2 \in (0.5, 1)$:
\begin{equation}
    z_{\nu_1}(x) < 0, \qquad z_{\nu_2}(x)   > 0,
\end{equation}
so $(0,1) \in \bfa_2(\fX_{\mu_1})$.
\item Consider $x = (x_1 = 0, x_2)$ with $x_2 \in (-\infty,0) \cup (0, 0.5) $:
\begin{equation}
    z_{\nu_1}(x) < 0, \qquad z_{\nu_2}(x)  < 0,
\end{equation}
so $(0,0) \in \bfa_2(\fX_{\mu_1})$.
\end{itemize}
We end up with $\bfa_2 (\fX_{\mu_1}) = \{ (0,0), (0,1), (1,1)\} $.

 \section{Algorithm}
\label{app:algo}

In the case of $\relu$ neurons, any input point $x$ in $\bbR^d$ can be assigned a {\em sign-vector}
$\sigma(x) := (\sign(z_\nu(x)))_{\nu \in H}$
-- where $\sign(u) = u / |u|$ except for $\sign(0)$ whose value is $0$.
The sign-vector of a point is simply related to its activations by $(a_\nu(x))_{\nu \in H} = \relu(\sigma(x))$.

The algorithm of \citet{berzins23relu} extracts the $1$-faces of the polyhedral complex, i.e. its edges, together with their sign-vectors.
We use this information to reconstruct the set of reachable activations $\bfa(\fX_\nu)$:
first, we use sign-vectors of points on 1-faces to derive sign-vectors of points on $d-1$-faces.
Then, taking the $\relu$ of this vector yields the activations needed to check convexity.

The sign-vector of any point belonging to a $1$-face contains $d-1$ zero entries: it lies on the 0-level set of $d- 1$ different preactivations.
On the other hand, sign-vectors of points on $d-1$-faces have exactly one zero, because they correspond to points where only one preactivation is 0.
Sign-vectors on such point can be recovered by following the so-called {\em perturbation} process described in \citet[Section 4.1]{berzins23relu}  which enables to recover all faces, in particular the $d-1$ faces.
For a neuron $\nu$:
\begin{itemize}
    \item we find all possible sign-vectors of points on $1$-faces where the coordinate $\nu$ is 0
    \item for each vector, for every subset of $d-2$ vanishing entries, we replace those 0 values by every possible combination of 1 and -1 ({\em perturbation} process).
    This gives all the sign-vectors of points on $d-1$-faces.
    \item taking the $\relu$ of those sign-vectors gives the corresponding activations, that we use in the same fashion as for the first bottleneck in \Cref{sec:numerics} to compute $\langle \bfa, \Phi^{\nu \to} \rangle$.
\end{itemize}

\addrebut{\Cref{algo1:convexity_check} presents pseudo-code of our algorithm.}
\addrebut{ \Cref{fig:curve_box_wdth} provides an ablation on the choice of the size of the hypercube domain on which convexity is checked.}
\addrebut{ \Cref{fig:times} gives computation time to check convexity of a network with architecture $\mathbf n =(n,n)$ (averaged on $10^4$ draws per architecture).}

\begin{figure}
    \centering
    \subfloat[Impact of the size of the hypercube domain.]{\includegraphics[width=0.45\linewidth]{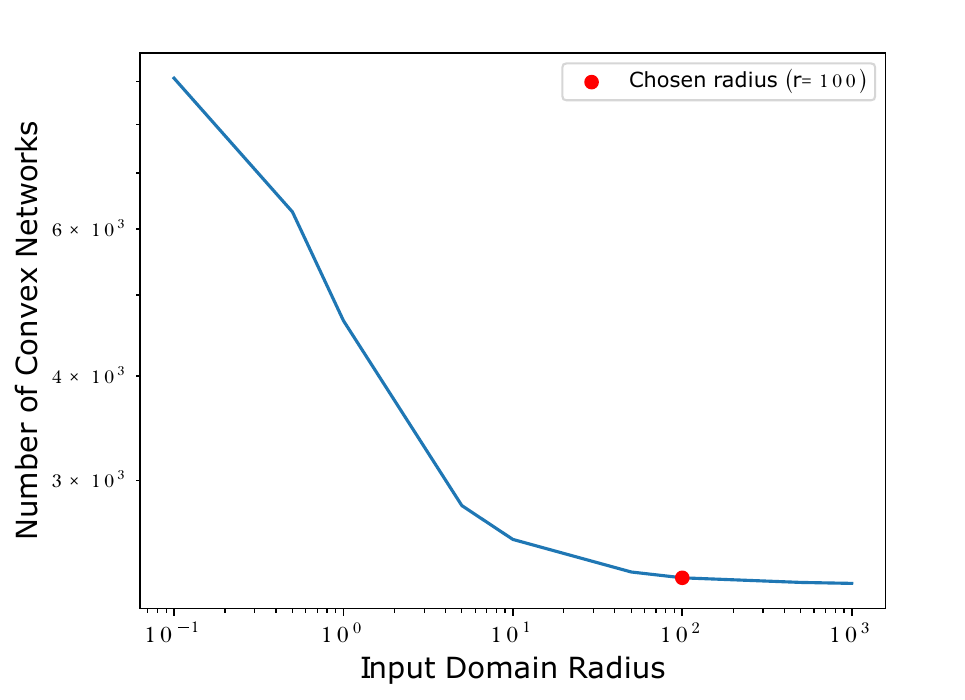}
    \label{fig:curve_box_wdth}}
    \subfloat[Computational time to check convexity of a network with architecture $\mathbf n =(n,n)$.]{\includegraphics[width=0.415\linewidth]{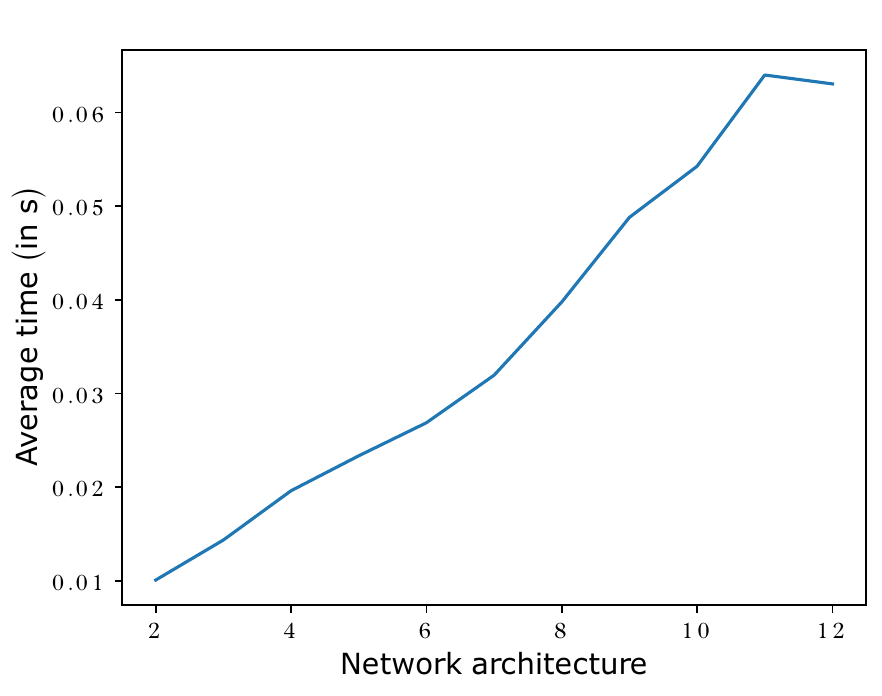}
    \label{fig:times}}
\caption{\addrebut{Ablation studies.}}
\end{figure}

\begin{algorithm}

    \caption{Convexity check}\label{algo1:convexity_check}
    \begin{algorithmic}[1]
    \Require $(G = (N,E), /\theta)$: a feed-forward, fully connected $\relu$ NN with parameters $\theta$
    \State $\mathsf{IsConvex} \gets \text{ True }$
    \State $(G', \theta') \gets$ create a copy of $G = (N,E)$ with biases set to $0$
    \State $\mathsf{signed-edges} \leftarrow$  Berzins \cite{berzins23relu} algorithm $(R_\theta)$
    \For{$\nu \in N$} \hfill \textit{// Loop over every neuron $\nu$ of the network}
        \State $\mathsf{signed-edges}_\nu \gets \mathsf{signed-edges}\left[\mathsf{signed-edges}[\nu]= 0 \right]$ \hfill \textit{// Extract edges $\mathbf e$ such that $\sigma_\nu(\mathbf e) = 0$, \emph{i.e.} edges belonging to the bent hyperplane $H_\nu$}
        \State $\mathsf{signed-hyperplanes}_\nu \gets$ Generate sign-vectors of $H_\nu$ by replacing every $0$ in $\mathsf{signed-edges}_\nu $ with both $\pm 1$ \hfill \textit{// The sign-vector of an edge $\mathbf e$ has $d-1$ zeros. Sign-vectors of hyperplanes have $1$ zero.}
        \State $\mathbf a^{\nu \to} \gets \relu(\mathsf{signed-hyperplanes}_\nu[l_\nu+1:])$ \hfill \textit{// Extract the sign vectors for neurons in layers after $l_\nu$, the layer of $\nu$ in $G$.}
        \State \textbf{Initialize} $\mathsf{res} \gets \text{zeros}(\text{size of layer } l)$ \hfill \textit{// Input for the network corresponding to $G^{\nu \to}$  }
        \State $\mathsf{res}[\nu] \gets 1$ \hfill \textit{// Set the value at index $\nu$ to 1}
        \For{layer $\in$ $(G', \theta')$ from $l_\nu + 1$ to  $l_{\text{out}}$}
        \hfill \textit{// Only consider layers after the layer of $\nu$ in $G$.}
            \State $\mathsf{res} \gets [\mathbf a^{\nu \to}]_{\text{layer}} * \text{layer}(\mathsf{res})$  \hfill \textit{// Replace standard $\relu$ activation by a entrywise product with $\mathbf a^{\nu \to}$.}
        \EndFor
        \If{$\mathsf{res} \not\geq 0$}
               \State  $\mathsf{IsConvex} \gets \text{ False }$ \hfill \textit{// As soon as a condition is violated, the function implemented by $R_\theta$ is not convex.}
        \EndIf
    \EndFor \\
    \Return $\mathsf{IsConvex}$
\end{algorithmic}
\end{algorithm}

\end{appendices}

\end{document}